\documentclass{article}

\usepackage[margin=1in]{geometry}
\usepackage{natbib}
\usepackage[english]{babel}
\usepackage{amsthm}
\theoremstyle{plain}
\newtheorem{theorem}{Theorem}
\newtheorem{lemma}{Lemma}
\newtheorem{corollary}{Corollary}
\theoremstyle{definition}
\newtheorem{definition}{Definition}
\newtheorem{assumption}{Assumption}
\newtheorem{example}{Example}

\usepackage[utf8]{inputenc}
\usepackage{amsmath,bbm,bm}
\usepackage{amssymb}
\usepackage{xspace}
\usepackage{multirow}
\usepackage{arydshln}
\usepackage{hyperref}
\usepackage[capitalize]{cleveref}

\newcommand\inner[2]{\langle #1, #2 \rangle}
\newcommand\abs[1]{\lvert #1 \rvert}
\newcommand\norm[1]{\| #1 \|}
\newcommand\tr{^\top}
\newcommand\e[0]{\mathbb{E}\xspace}
\newcommand{\indicator}[1]{\mathbbm{1}\{#1\}\xspace}
\newcommand{\softplus}[0]{\text{softplus}\xspace}

\newcommand{\expit}[0]{\text{expit}\xspace}
\DeclareMathOperator*{\argmax}{arg\,max}
\DeclareMathOperator*{\argmin}{arg\,min}

\newcommand{\simpleiv}[0]{\text{SimpleIV}\xspace}
\newcommand{\heteroiv}[0]{\text{HeteroskedasticIV}\xspace}
\newcommand{\policylearning}[0]{\text{PolicyLearning}\xspace}

\newcommand{\Fcal}[0]{\mathcal{F}\xspace}
\newcommand{\Gcal}[0]{\mathcal{G}\xspace}
\newcommand{\Pcal}[0]{\mathcal{P}\xspace}
\newcommand{\Qcal}[0]{\mathcal{Q}\xspace}
\newcommand{\Xcal}[0]{\mathcal{X}\xspace}
\newcommand{\Ucal}[0]{\mathcal{U}\xspace}

\newcommand{\EE}[0]{\mathbb{E}\xspace}
\newcommand{\VV}[0]{\mathbb{V}\xspace}
\newcommand{\RR}[0]{\mathbb{R}\xspace}

\Crefname{figure}{Table}{Tables}
\Crefname{assumption}{Assumption}{Assumptions}

\begin{document}

\title{The Variational Method of Moments}
\author{Andrew Bennett and Nathan Kallus}
\date{}

\maketitle

\begin{abstract}
The conditional moment problem is a powerful formulation for describing structural causal parameters in terms of observables, a prominent example being instrumental variable regression. We introduce a very general class of estimators called the \emph{variational method of moments} (VMM), motivated by a variational minimax reformulation of optimally-weighted generalized method of moments for finite sets of moments. VMM controls infinitely many moments characterized by flexible function classes such as neural nets and kernel methods, while provably maintaining statistical efficiency unlike existing related minimax estimators. We also develop inference algorithms and demonstrate the empirical strengths of VMM estimation and inference in experiments.
\end{abstract}

\section{Introduction}
\label{sec:intro}

For many problems in fields such as economics, sociology, or epidemiology, we seek to use observational data to estimate \emph{structural parameters}, which often describe some causal relationship. A common framework which unifies many such problems is the \emph{conditional moment problem}, which assumes that the parameter of interest $\theta_0$ is the unique element of some parameter space $\Theta$ such that
\begin{equation}
\label{eq:ident}
    \e[\rho(X;\theta_0) \mid Z] = 0\,,
\end{equation}
where $X \in \mathcal X$ denotes the observed data, $Z \in \mathcal Z$ is a random variable that is measurable with respect to $X$, 
and $\rho : \mathcal X \to \mathbb R^m$ is a vector-valued function indexed by $\Theta$.
Note \cref{eq:ident} is an identity of \emph{random variables}, not of numbers; that is, it holds almost surely with respect to the random $Z$. 
That $Z$ is measurable with respect to $X$ is without loss of generality, since given any $\tilde X,Z$ we may define $X=(\tilde X,Z)$ as the observed data; thus $\rho$ may potentially depend on all data.
Note also that at this point we let $\Theta$ be general; for example, it may be finite dimensional or it may be a class of functions. 

\begin{example}\label{ex:iv}
Perhaps the most common example of a conditional moment problem is the instrumental variable regression problem (see \emph{e.g.} \citet{angrist2008mostly}, and citations therein), where we seek to estimate the causal effect of some treatment $T$ on an outcome $Y$, where the observed relationship between $T$ and $Y$ may be confounded by some unobserved variables, but we have an instrumental variable $Z$ that affects $T$ but only affects $Y$ via its effect on $T$. Given some regression function $g$ parameterized by $\theta\in\Theta$, the value $\theta_0$ corresponding the true regression function is assumed to be the unique solution to
\begin{equation*}
    \e[Y - g(T; \theta_0) \mid Z] = 0\,.
\end{equation*}
This is an example of \cref{eq:ident} with $X = (T,Y,Z)$ and $\rho(X;\theta) = Y - g(T;\theta)$. 

More intricate variants of this, for example, include \citet{berry1995automobile}, which incorporate discrete choice and is widely used to formulate structural demand parameters in industrial organization.
\end{example}

\begin{example}\label{ex:iv-quantile}
A second example is instrumental quantile regression, which is similar to the previous example but focuses on quantiles instead of means \citep{chernozhukov2007instrumental,horowitz2007nonparametric}. Here, again assume a treatment $T$, outcome $Y$, and instrumental variable $Z$, but now we seek to estimate the causal effect of $T$ on the $p^\text{th}$ quantile of $Y$, for some $0 < p < 1$. In this case, given a quantile regression function $g$ parameterized by $\theta\in\Theta$, the value $\theta_0$ corresponding to the true quantile regression function is assumed to be the unique solution to
\begin{equation*}
    \e[\indicator{Y \leq g(T;\theta_0)} - p \mid Z] = 0\,.
\end{equation*}
This is an example of \cref{eq:ident} with $X=(T,Y,Z)$ and $\rho(X;\theta) = \indicator{Y \leq g(T;\theta)} - p$.
\end{example}

\begin{example}\label{ex:rl}
A third example of a problem is estimating the stationary state density ratio between two policies in offline reinforcement learning \citep{liu2018breaking,kallus2022efficiently,bennett2021off}.
Consider a Markov decision process given by an unknown transition kernel $p(S'\mid S,A)$ describing the distribution of next state $S'$ when action $A$ is taken in previous state $S$. Suppose $\pi_e(A\mid S),\pi_b(A\mid S)$ are two known policies assumed to induce unknown stationary distributions on $S$, $p_e(S),p_b(S)$, which we assume exist. We are interested in their ratio (generally, Radon-Nikodym derivative), $d(S)$.
Given observations $(S,A,S')$ from $p_b(S)\pi_b(A\mid S)p(S'\mid S,A)$, we have $d(S;\theta_0)\propto d(S)$ if and only if
\begin{equation*}
    \e[d(S;\theta_0) {\pi_e(A\mid S)}{\pi^{-1}_b(A\mid S)} - d(S';\theta_0) \mid S'] = 0\,.
\end{equation*}
Then, for example, if $\Theta$ satisfies $\int d(s;\theta) d\mu(s) = 1 \ \forall \theta \in \Theta$ for some fixed measure $\mu$, and there exists some $\theta_0 \in \Theta$ such that $d(S;\theta_0) \propto d(S)$, then this conditional moment restriction will identify $\theta_0$. Note that although $d(S;\theta_0) \neq d(S)$ in general, estimates of $\theta_0$ are still of interest, as they could be used to estimate $d(S)$ in downstream tasks, for example by dividing by a plug-in estimate of $\EE[d(S;\theta_0)]$ using estimates of $\theta_0$ and $\EE$ (since $d(S)$ is known to satisfy the normalization constraint $\EE[d(S)] = 1$.)
This is an example of \cref{eq:ident} with $X=(S,A,S')$, $Z=S'$, and $\rho(X;\theta) = d(S;\theta_0) {\pi_e(A\mid S)}{\pi^{-1}_b(A\mid S)} - d(S';\theta_0)$.
\end{example}

The classic approach to the conditional moment problem is to reduce it to a system of $k$ \emph{marginal} moments, $\e[F(Z)\rho(X;\theta_0)]=0$, where $F:\mathcal Z\mapsto\mathbb R^{k\times m}$ is a chosen matrix-valued function. Then, we can apply the optimally weighted generalized method of moments 
\citep[OWGMM;][]{hansen1982large}, which we present in detail in \cref{sec:owgmm} below.
Since this marginal moment formulation is implied by \cref{eq:ident} but not necessarily vice versa, this requires we find a sufficiently rich $F(Z)$ such that the marginal moment problem still identifies $\theta_0$, that is, it is still the unique solution in $\Theta$. Moreover, even if this identifies $\theta_0$ and even though OWGMM is efficient in the model implied by 
$\e[F(Z)\rho(X;\theta_0)]=0$, the result may not be efficient in the model implied by \cref{eq:ident}. 

There are a few general approaches to dealing with this.
There are classic nonparametric approaches that are sieve-based 
and simply grow $k$, the output dimension of $F(Z)$, with $n$ by including additional functions from a basis for $L_2$ such as power series \citep{chamberlain1987asymptotic}.
There are also classic nonparametric approaches that directly estimate some special identifying $F^*(Z)$ that also induces an efficient OWGMM \citep{newey1990efficient,newey1993efficient}.
For example, in \cref{ex:iv} with $g(T;\theta)=\theta\tr T$, we have $F^*(Z)=\e[T \mid Z]$, which can be nonparametrically estimated and plugged into OWGMM.
Furthermore, there are approaches that used sieve-based methods to simultaneously estimate $\e[\rho(X;\theta) \mid Z]$ for every $\theta \in \Theta$, and pick $\theta$ to minimize some weighted empirical norm of these estimated conditional expectations \citep{ai2003efficient,newey2003instrumental,chen2009efficient,chen2012estimation}.

A recent line of work instead focuses on tackling this problem using machine-learning-based approaches \citep{hartford2017deep,lewis2018adversarial,singh2019kernel,muandet2020dual,dikkala2020minimax,bennett2019deep,kallus2021causal,uehara2021finite}. These approaches are varied, with some solving the general problem in \cref{eq:ident} and others solving the more specific instrumental variable regression problem or other specific problems, with approaches based on deep learning, kernel methods, or both.
Most of these are based on an adversarial/minimax/saddle-point approach \citep{lewis2018adversarial,muandet2020dual,dikkala2020minimax,bennett2019deep,kallus2021causal,uehara2021finite}.

Currently, there is a disconnect between these two lines of approaches. On the one hand, the more classical approaches are well motivated by efficiency theory when we impose certain smoothness assumptions. This is in contrast with the recent machine-learning based approaches; while some provide consistency guarantees \citep{bennett2019deep} and even rates \citep{singh2019kernel,dikkala2020minimax,kallus2021causal,uehara2021finite}, none of these approaches are shown to be semiparametically efficient for \cref{eq:ident} or can facilitate inference on $\theta_0$. On the other hand, however, the more recent line of work leverages modern machine learning approaches, which are commonly believed to have superior practical properties. For example, they have been empirically observed to be more stable, have easier parameter tuning, or be better able to adapt to the low-dimensional latent structure of complex data.
Although our experiments do indeed seem to support this thesis, especially in more challenging settings, we emphasise that the point of this paper is \emph{not} to demonstrate that modern machine learning-based approaches are superior to classical ones. Rather, we observe that for various reasons there is significant, growing interest in machine-learning based approaches to these problems within the community, and therefore extending this line of work to be semiparametrically efficient and to perform inference is of great importance.

In this paper, we study a general class of minimax approaches, which we call the \emph{variational method of moments} (VMM). This generalizes the method of \citet{bennett2019deep}, who presented an estimator for instrumental variable regression using adversarial training of neural networks. 
Their proposal was motivated by a variational reformulation of OWGMM, aiming to combine the efficiency of more classical approaches with the flexibility of machine learning methods.
This style of estimator has since been applied to a variety of other conditional moment problems including policy learning from observational data \citep{bennett2020efficient} and estimating stationary state density ratios \citep{bennett2021off}. However, this past work did not provide a general formulation of VMM and  a detailed theoretical analysis. And, although they are motivated by efficiency considerations, it is not immediately clear that this actually leads to efficient estimators. 

We present a unified theory for a general class of VMM estimators. In particular, for some specific versions of these estimators based on either deep learning or kernel methods, we provide appropriate assumptions under which these methods are consistent, asymptotically normal, and semiparametrically efficient. In addition, we provide inference algorithms for these estimators, which can be used to construct confidence intervals for the estimated parameters. These inference algorithms are based on the same kind of variational reformulation as the estimation algorithms themselves, again with varieties based on both kernel methods and deep learning. Finally, we provide a detailed series of experiments that demonstrate that these VMM algorithms obtain very good finite-sample estimation performance and that the corresponding inference algorithms produce high quality confidence intervals.

The rest of this paper is structured as follows: 
in \cref{sec:background} we define the VMM estimator and provide motivation for it by interpreting OWGMM as a specific case thereof;
in \cref{sec:kernel-vmm} we provide our theory for \emph{kernel VMM} estimators, which are a specific instance of VMM estimators based on kernel methods; in \cref{sec:neural-vmm} we provide our theory for \emph{neural VMM} estimators, which are an alternative instance of VMM based on deep learning methods; in \cref{sec:inference} we present our inference theory, with proposed kernel- and neural net-based algorithms; in \cref{sec:experiments} we provide a detailed empirical evaluation of our proposed estimation and inference methods; and in \cref{sec:related-work} we provide a detailed discussion of past work on solving conditional moment problems and how these approaches relate to our VMM estimators. 

\paragraph{Notation.}
We use uppercase letters such as $X$ to denote random variables and lowercase ones to denote nonrandom quantities.
The set of positive integers is $\mathbb N$, and for any $n\in\mathbb N$ we use $[n]$ to refer to the set $\{1,\ldots,n\}$.
We denote by $\norm{\cdot}_{L_p}$ the usual $L_p$ functional norm, defined as $\norm{f}_{L_p} = \e[\abs{f(X)}^p]^{1/p}$, where the probability measure is implicit from context.

\section{Variational Method of Moments}
\label{sec:background}

We now define the class of \emph{variational method of moments} (VMM) estimators.
We consider data consisting of $n$ independent and identically distributed observations of $X$, namely, $X_1,\dots,X_n$ $\sim \mathcal P$, where $\mathcal P$ denotes the data distribution.
Let some sequence of function classes $\mathcal F_n$ be given, such that each $f \in \mathcal F_n$ has signature $f : \mathcal Z \to \mathbb R^m$. Let a ``prior estimate'' $\tilde\theta_n \in \Theta$ be given. In general this may be \emph{any} data-driven choice from $\Theta$, and need not necessarily be consistent for $\theta_0$; in the theory that follows we will elaborate on what conditions $\tilde\theta_n$ needs to satisfy for our respective results.
Furthermore, let $R_n : \mathcal F_n \to \mathbb [0,\infty]$ be some optional regularizer, which measures the complexity of $f \in \mathcal F_n$.
Then, we define the VMM estimate $\hat\theta_n^{\text{VMM}}=\hat\theta_n^{\text{VMM}}(\mathcal F_n,R_n,\tilde\theta_n)$ corresponding to these choices as follows:
\begin{equation}
\label{eq}
    \hat\theta_n^{\text{VMM}} = \argmin_{\theta \in \Theta} \sup_{f \in \mathcal F_n} \e_n[f(Z)\tr  \rho(X; \theta)] - \frac{1}{4} \e_n[(f(Z)\tr  \rho(X; \tilde\theta_n))^2] - R_n(f)\,,
\end{equation}
where $\e_n$ is an empirical average over the $n$ data points.

In \cref{sec:kernel-vmm} we study the instantiation of this with $\mathcal F_n$ being a reproducing kernel Hilbert space. In \cref{sec:neural-vmm} we study the instantiation with $\mathcal F_n$ being a class of neural networks. 

Before proceeding to study these new machine-learning-based instantiations of the VMM estimator with flexible choices for $\mathcal F_n$, we discuss a very simple instantiation that recovers OWGMM, which provides motivation and interpretation for each of the terms in \cref{eq}.

\subsection{The Optimally Weighted Generalized Method of Moments}
\label{sec:owgmm}

First, we present the classic OWGMM method. Given $F(Z)=(f_1(Z),\dots,f_k(Z))$, we obtain the marginal moment conditions $\e[f_i(Z)\tr  \rho(X;\theta_0)] = 0 \ \forall \ i \in [k]$.
Let a ``prior estimate'' $\tilde\theta_n$ be given 
 and define the matrix $\Gamma$ as
\begin{align*}
    \Gamma_{i,j} &= \e_n[f_i(Z)\tr  \rho(X;\tilde\theta_n)\rho(X;\tilde\theta_n)\tr  f_j(Z)]\,.
\end{align*}
Then, the OWGMM estimate $\hat\theta_n^{\text{OWGMM}}=\hat\theta_n^{\text{OWGMM}}(f_1,\dots,f_k,\tilde\theta_n)$ is defined as
\begin{equation}
\label{eq:owgmm}
    \hat\theta_n^{\text{OWGMM}} = \argmin_{\theta \in \Theta} 
    \sum_{i=1}^k\sum_{j=1}^k(\Gamma^{-1})_{ij}\e_n[f_i(Z)\tr  \rho(X;\theta)]\e_n[f_j(Z)\tr  \rho(X;\theta)]\,.
\end{equation}

Given certain regularity conditions and assuming the choice of functions $f_1,\ldots,f_k$ are sufficient such that the corresponding $k$ moment conditions uniquely identify $\theta_0$, standard GMM theory says that $\hat\theta_n$ is consistent for $\theta_0$. Furthermore, if the prior estimate $\tilde\theta_n$ is consistent for $\theta_0$, then this estimator is efficient with respect to the model defined by these $k$ moment conditions \citep{hansen1982large}.

OWGMM generalizes the method of moments, which solves $\e_n[f_i(Z)\tr  \rho(X;\theta)]=0$ for all $i\in[k]$. When there are many moments, we cannot make all of them zero due to finite-sample noise and instead we seek to make them \emph{near} zero. But it is not clear which moments are more important; for example, there may be duplicate or near-duplicate moments.
The key to OWGMM's efficiency is to \emph{optimally} combine the $k$ objectives of making each moment near zero into a single objective function.
To get a consistent prior estimate, we can for example let $\tilde\theta_n$ itself be a OWGMM with any fixed prior estimate, leading to the two-step GMM estimator. This can be repeated, leading to the multi-step GMM estimator.

Unfortunately, estimators of this kind have many limitations.
For one, in practice it is difficult or impossible to verify that any such set of functions $f_1,\ldots,f_k$ are sufficient for identification.
In addition, while such an estimator is efficient with respect to the model imposed by these $k$ moment conditions, ideally we would like to be efficient with respect to the model given by \cref{eq:ident}; that is, we would wish to be efficient with respect to the model given by \emph{all} moment conditions of the form $\e[f(Z)\tr  \rho(X;\theta_0)]=0$ for square integrable $f$.
Finally, in the case that $k$ were very large and growing with $n$, as would be required to (at least approximately) alleviate the prior two concerns, the corresponding sieve-based estimator would require impractical tuning to select which basis of $L_2$ to use and to choose $k$ as a function of $n$. As will be seen below, such an approach may be seen as equivalent to estimating the optimal instruments over a linear sieve, but unlike our variational approach that we propose below it is unclear how to appropriately regularize this sieve estimation, and take advantage of modern machine learning advances on non-parametric function approximation.

\subsection{Variational Reformulation of OWGMM}

One motivation for our VMM class of estimators, \cref{eq}, is that it recovers OWGMM with its efficient weighting.

The following result simply appeals to the optimization structures of \cref{eq,eq:owgmm} and generalizes \citet[lemma 1]{bennett2019deep}. 
We include its proof as it is short and instructive.
\begin{lemma}
\label{lem:owgmm}
$\hat\theta_n^{\text{OWGMM}}(f_1,\dots,f_k,\tilde\theta_n)=\hat\theta_n^{\text{VMM}}(\text{span}(\{f_1,\ldots,f_k\}),0,\tilde\theta_n)$.
\end{lemma}
\begin{proof}[Proof of \cref{lem:owgmm}]
Let $F(Z)=(f_1(Z),\dots,f_k(Z))$ be a map $\mathcal Z\to\mathbb R^{k\times m}$. Then,
\begin{align*}
&\hat\theta_n^{\text{OWGMM}}(f_1,\dots,f_k,\tilde\theta_n) \\
    &= \argmin_{\theta \in \Theta} \norm{\Gamma^{-1/2} \e_n[F(Z)  \rho(X;\theta)]}^2 \\
    &= \argmin_{\theta \in \Theta} \sup_{v \in \mathbb R^k} v\tr  \e_n[F(Z)  \rho(X;\theta)] - \frac{1}{4} v\tr  \Gamma v \\
    &=\argmin_{\theta \in \Theta} \sup_{v \in \mathbb R^k} \e_n[(F(Z)\tr v)\tr  \rho(X;\theta)] - \frac{1}{4} \e_n[((F(Z)\tr v)\tr\rho(X;\tilde\theta_n))^2] \,,
\end{align*}
where the second equality is a reformulation of the rotated Euclidean norm (see \cref{lem:operator-sqrt-inverse}  for the general Hilbert-space version). The conclusion follows by noting $\{F(Z)\tr v:v\in\mathbb R^k\}=\text{span}(\{f_1,\ldots,f_k\})$.
\end{proof}

Through the lens of \cref{lem:owgmm}, we can understand each term of \cref{eq} as follows. The first term pushes $\theta$ to make $\e_n[f(Z)\tr  \rho(X; \theta)]$ near zero for each $f\in\mathcal F_n$. The second term, $- \frac{1}{4} \e_n[(f(Z)\tr  \rho(X; \tilde\theta_n))^2]$, appropriately weights the relative importance of making each of these near zero. Finally, varying $\mathcal F_n$ and/or $R_n(f)$ with $n$ allows us to control the richness of moments that we consider, in analogy to sieve-based methods that grow the dimension of the space $\text{span}(\{f_1,\ldots,f_k\})$ but admitting more flexible machine-learning approaches.
This motivation is similar to \citet{bennett2019deep,bennett2020efficient,bennett2021off}, but these did not study the problem in generality or establish properties such as asymptotic normality or efficiency.

\section{Kernel VMM}
\label{sec:kernel-vmm}

First, we consider a class of VMM estimators where for every $n$ we have $\mathcal F_n = \mathcal F$, where  $\mathcal F = \bigoplus_{i=1}^m \mathcal F_i$ and each $F_i$ is a reproducing kernel Hilbert space (RKHS) of functions $\mathcal Z\to\mathbb R$ given by a symmetric positive definite kernel $K_i:\mathcal Z\times\mathcal Z\to\mathbb R$, and regularization is performed using the RKHS norm of $\mathcal F$, which we denote by $\|(f_1,\dots,f_m)\|^2=\sum_{i=1}^m\|f_i\|_{\mathcal F_i}^2$. 
We will call these estimators \emph{kernel VMM} estimators, which we concretely define according to
\begin{equation}
\label{eq:kernel-vmm}
    \hat\theta_n^{\text{K-VMM}} = \argmin_{\theta \in \Theta} J_n(\theta)\,,
\end{equation}
where
\begin{equation*}
    \notag J_n(\theta) = \sup_{f \in \mathcal F} \e_n[f(Z)\tr \rho(X;\theta)] - \frac{1}{4} \e_n[(f(Z)\tr \rho(X;\tilde\theta_n))^2] - \frac{\alpha_n}{4} \norm{f}^2 \,,
\end{equation*}
and $\alpha_n$ is some non-negative sequence of regularization coefficients. Explicitly, this fits into our general VMM definition with $\mathcal F_n = \mathcal F$ for every $n$, and $R_n(f) = \alpha_n \norm{f}^2$.

Before we provide our main theory for kernel VMM estimators, we provide a convenient reformulation of \cref{eq:kernel-vmm}.  
Let $\mathcal H$ be the dual space of $\mathcal F$ (that is, the space of all bounded linear functionals of the form $\mathcal F \mapsto \mathbb R$) and for each $\theta \in \Theta$ define the element $\bar h_n(\theta) \in \mathcal H$ according to
\begin{equation*}
    \bar h_n(\theta)(f) = \e_n[f(Z)\tr  \rho(X; \theta)]\,.
\end{equation*}
Furthermore, define the linear operator $C_n : \mathcal H \to \mathcal H$ according to
\begin{equation*}
    (C_n h)(f) = \e_n[\varphi(h)(Z)\tr  \rho(X;\tilde\theta_n) \rho(X;\tilde\theta_n)\tr  f(Z)]\,,
\end{equation*}
where $\varphi : \mathcal H \to \mathcal F$ maps any element in $\mathcal H$ to its Riesz representer in $\mathcal F$ such that $h(f)=\inner{\varphi(h)}f$.
\begin{lemma}
\label{lem:kvmm-cgmm-equiv}
The kernel VMM estimator defined in \cref{eq:kernel-vmm} is equivalent to
\begin{equation*}
    \hat\theta_n^{\text{K-VMM}} = \argmin_{\theta \in \Theta} \norm{(C_n + \alpha_n I)^{-1/2} \bar h_n(\theta)}_{\mathcal H}^2\,,
\end{equation*}
where $I$ is the identity operator $Ih=h$.
\end{lemma}

We note that comparing this result to \cref{eq:owgmm}, this is a clear infinite-dimensional generalization of the OWGMM objective, where the matrix $\Gamma$ defined there is replaced with a linear operator, and the inversion is performed using Tikhonov regularization.
Note that this re-framing of our kernel VMM estimator also shows a connection to the continuum GMM estimators considered by \citet{carrasco2000generalization}.
However, our estimator does not strictly fit within their framework. We discuss this in more detail in \cref{sec:related-work}.

\subsection{Consistency}

We first provide some sufficient assumptions in order to ensure that our kernel VMM estimator is consistent; that is, $\hat\theta_n^{\text{K-VMM}} \to \theta_0$ in probability. Before we present these assumptions, we define the conditional covariance function of the moment problem:
\begin{equation}
    \label{eq:v}
    V(Z;\theta) = \e[\rho(X;\theta) \rho(X;\theta)\tr  \mid Z]\,.
\end{equation}

For our first assumption, we require each $\mathcal F_i$ to be universally approximating with a smooth kernel. Recall for this definition that, a function is $C^\infty$-smooth if it is $n$-times continuously differentiable for \emph{every} positive integer $n$. In addition, we recall that a kernel is \emph{universal} if the corresponding RKHS is dense in the space of continuous real-valued functions on $\mathcal Z$ under the supremum norm \citep{sriperumbudur2011universality}. Note that all of the properties of the following assumption hold, for example, for the commonly used Gaussian kernel.

\begin{assumption}[Universal RKHS]
\label{asm:rkhs}
    For each $i \in [m]$, $K_i$ is $C^\infty$-smooth in both arguments and $\mathcal F_i$ is universal.
\end{assumption}

Next, we require a basic regularity condition on the observed data distribution. 
This together with \cref{asm:rkhs} ensure that $\mathcal F$ is well-behaved with respect to $\rho$,
and satisfies some nice properties in terms of boundedness and metric entropy, as formalized by \cref{lem:f-properties} in the appendix.

\begin{assumption}[Regularity]
\label{asm:regularity}

$\mathcal Z$ is a bounded subset of $\mathbb R^{d_z}$ for some positive integer $d_z$.

\end{assumption}

Next, we require that the set of possible functions $\{\rho(\cdot;\theta) : \theta \in \Theta\}$ satisfies some basic boundedness, smoothness, and complexity properties.
A simple example satisfying the below is for $\Theta$ to be a compact set in some finite-dimensional Euclidean space, and for $\rho(x;\theta)$ to be equi-Lipschitz continuous in $\theta$ for every $x$.
Other examples that easily satisfy the second part of the below include $\{\theta(\cdot;\theta) : \theta \in \Theta\}$ having finite Vapnik–Chervonenkis dimension (see \emph{e.g.} \citealp[theorem 8.19 and corollary 9.5]{kosorok2007introduction}), or be a bounded-norm subset of an RKHS (see \cref{lem:f-donsker} in the appendix for details). This assumption ensures that consistent estimation of $\theta_0$ is possible, even though inversion of the conditional moment operator could be ill-posed.

\begin{assumption}[Moment Class Complexity]
\label{asm:rho-complexity}

$\sup_{x \in \Xcal, \theta \in \theta} |\rho(x;\theta)| < \infty$, and $\rho(X;\theta)$ is Lipschitz continuous in $\theta$ under the $L_1$ norm. Also, for each $i \in [m]$ the function set $\{\rho_i(\cdot;\theta) : \theta \in \Theta\}$ is $\Pcal$-Donsker.
\end{assumption}

We also assume that the prior estimate $\tilde\theta_n$ is well-behaved, meaning that it converges sufficiently fast to some limit in probability. This limit need not be $\theta_0$ for our consistency results. This will be used to ensure the convergence of the linear operator $C_n$ defined above to some limiting operator $C$.

\begin{assumption}[Convergent Prior Estimate]
\label{asm:tilde-theta}
    The prior estimate $\tilde\theta_n$ has a limit $\tilde\theta$ in probability, and satisfies $\|\rho_i(X;\tilde\theta_n) - \rho_i(X;\tilde\theta)\|_2 = O_p(n^{-p})$ for every $i \in [m]$ and some $0 < p \leq 1/2$.
\end{assumption}

Finally, we assume a nonsingular covariance with bounded inverse moments.

\begin{assumption}[Non-Degenerate Moments]
\label{asm:non-degenerate}
    For each $\theta \in \{\tilde\theta,\theta_0\}$, we have that $V(Z;\theta)$ is invertible almost surely, and also that $\|\sigma_{\text{min}}(Z;\theta)^{-1}\|_\infty < \infty$, where $\sigma_{\text{min}}(Z;\theta)$ denotes the minimum eigenvalue of $V(Z;\theta)$.
\end{assumption}

\Cref{asm:non-degenerate} is slightly subtle and is used to ensure that the objective $J_n$ defined above converges to a well-behaved limiting objective $J$ that is uniquely minimized by $\theta_0$, which is central to our consistency proof. In the absence of this assumption, it is possible that the limiting objective may diverge.
We note that in the case of $m=1$, the second part of the assumption is equivalent to requiring that $\|V(Z;\tilde\theta)^{-1}\|_\infty, \|V(Z;\theta_0)^{-1}\|_\infty < \infty$, and in the case that the prior estimate $\tilde\theta_n$ is consistent we only need this condition to hold at $\theta_0=\tilde\theta$.
In general, it can be viewed in terms of certain moments defined in terms of the data distribution and $\rho$ being bounded.

With these assumptions, we are prepared to state our consistency result.

\begin{theorem}[Consistency]
\label{thm:consistency}
    Let \cref{asm:rkhs,asm:regularity,asm:rho-complexity,asm:non-degenerate,asm:tilde-theta} be given, and suppose the regularization coefficient satisfies $\alpha_n = o(1)$ and $\alpha_n = \omega(n^{-p})$, where $p$ is the constant referenced in \cref{asm:tilde-theta}. Then, for any $\hat\theta_n$ that satisfies $J_n(\hat\theta_n) = \inf_{\theta \in \Theta} J_n(\theta) + o_p(1)$, we have $\hat\theta_n \to \theta_0$ in probability.
\end{theorem}

Comparing this result to the corresponding consistency result given by \citet[theorem 2]{bennett2019deep}, we note that this result \emph{does not} rely on any specific identification assumptions beyond \cref{eq:ident}. Conversely, \citet{bennett2019deep} assume that the class $\mathcal F$ of neural nets that they take a supremum over is sufficient to uniquely identify $\theta_0$, which is a questionable assumption since this class is assumed to be fixed and not growing with $n$. Therefore, we argue that our VMM consistency here is given under much more reasonable assumptions. 

Next, we make some observations about how this result compares with consistency results in the literature that tackles nonlinearities using sieves. First, note that \cref{asm:rho-complexity} is weaker than the corresponding assumptions in \citet{ai2003efficient} and \citet{newey2003instrumental}, who assume that $\rho(x;\theta)$ is point-wise H\"older-continuous and $\Theta$ is compact. Instead, we require the more general assumption of continuity in $L_1$-norm, along with a Donsker condition. Conversely, \citet{chen2009efficient} and \citet{chen2012estimation} similarly allow for non-smooth $\rho$, but they consider the setting where $\Theta$ can be non-compact, introducing ill-posedness issues that they tackle in their work. Rather, we specifically consider metrics on $\Theta$ under which ill-posedness is \emph{not} an issue, given our Donsker assumption on $\{\rho(\cdot;\theta) : \theta \in \Theta\}$ and $L_1$-continuity. Furthermore, we note that assumptions similar to \cref{asm:regularity} and \cref{asm:non-degenerate} are standard in these past works, and \cref{asm:rkhs,asm:tilde-theta} are straightforward technical conditions related to implementation choices for our method.

\subsection{Asymptotic Normality}

We now present our theory for the asymptotic normality of kernel VMM estimates. Here we consider the special case where $\Theta$ is a compact subset of $\mathbb R^b$ for some positive integer $b$.
We note that in this case, as discussed above, \cref{asm:rho-complexity} follows under very simple additional conditions; \emph{e.g.} $\rho(x;\theta)$ being equi-Lipschitz continuous in $\theta$ for every $x \in \Xcal$.
Under this setting, we will characterize the asymptotic distribution of $\sqrt{n}(\hat\theta_n - \theta_0)$.

First, we require that $\rho(X;\theta)$ satisfies the following differentiability condition.

\begin{assumption}[$\rho$ Differentiable in Absolute Mean]
\label{asm:rho-deriv}

For each $i \in [m]$, there exists some vector-valued function $D_i(X;\theta) \in \RR^b$ indexed by $\theta$, and some neighborhood $\Theta_0$ of $\theta_0$, such that, for every $\theta \in \Theta_0$, we have
\begin{equation*}
    \lim_{\theta' \to \theta} \frac{ \Big\| \rho_i(X;\theta') - \rho_i(X;\theta) - (\theta'-\theta)^\top D_i(X;\theta) \Big\|_{L_1} }{ \|\theta'-\theta\| }= 0 \,.
\end{equation*}

\end{assumption}

In other words, $D_i$ is a gradient-like function such that the first-order Taylor error decays to zero at a $o(\|\theta'-\theta\|)$ rate under the $L_1$ norm. For example, in the case that $\rho_i(x;\theta)$ is continuously differentiable in $\theta$ within some neighborhood of $\theta_0$ for all $x \in \Xcal$, then \cref{asm:rho-deriv} trivially follows from Taylor's theorem. Furthermore, it is easy to see that for any $x$, $\theta$ where $\rho_i(x;\theta)$ is differentiable w.r.t. $\theta$, we must have $D_i(x;\theta) = \nabla \rho_i(x;\theta)$. However, the above is more general and allows for situation where $\rho_i(x;\theta)$ is non-differentiable at some values of $x$ and $\theta$. In particular, the following lemma allows us to establish this assumption under more general conditions.

\begin{lemma}
\label{lem:non-smooth-rho}
Suppose there exist $\phi_i : \Xcal \to \RR$ indexed by $\Theta$, and ``gradient-like" and ``Hessian-like" functions $\rho_i'$ and $\rho_i''$ such that: (1) $\rho_i(X;\theta)$ is twice differentiable in $\theta$ with gradient $\rho_i'(X;\theta)$ and Hessian $\rho_i''(X;\theta)$ whenever $\phi(X;\theta) \neq 0$; (2) $\sup_{x \in \Xcal, \theta \in \Theta} \|\rho_i'(x;\theta)\|_2 \leq c'$ and $\sup_{x \in \Xcal, \theta \in \Theta} \|\rho_i''(x;\theta)\|_{\textup{op}} \leq c''$ for some $c',c'' < \infty$; (3) $\rho_i(x;\theta)$ is $L_\rho$-Lipschitz in $\theta$ for every $x \in \Xcal$, for some $L_\rho < \infty$; and (4) $\phi(x;\theta)$ is $L_\phi(x)$-Lipschitz in $\theta$, for some $L_\phi(x)$ such that the probability density of the random variable $L_\phi(X)^{-1} \phi(X;\theta)$ is bounded within some neighborhood of zero. Then, we have that \cref{asm:rho-deriv} holds with $D_i(X;\theta) = \rho_i'(X;\theta)$.
\end{lemma}
This lemma allows us to establish \cref{asm:rho-deriv} for a range of problems where $\rho(X;\theta)$ has some points of non-smoothness. Intuitively, the boundedness condition on $\rho''$ allows us to bound the first-order Taylor error whenever $\rho$ is smooth, and the Lipschitz and bounded density assumptions on $\phi$ near $\phi=0$ prevents non-smoothness from impacting the first-order Taylor expansion, up to an additional $o(\|\theta'-\theta\|)$ factor.

Next, we let $D(X;\theta) \in \RR^{m \times b}$ denote the Jacobian-like function given by concatenating $D_i(Z;\theta)$ for all $i \in [m]$. Similarly, we define $h'_n, h' \in \mathcal H^b$ according to
\begin{equation*}
    h'_n(f) = \e_n[D(X;\theta_0)\tr  f(Z)] \qquad h'(f) = \e[D(X;\theta_0)\tr  f(Z)] \,,
\end{equation*}
where $\mathcal H$ is the dual space of $\mathcal F$, as above. We also define the analogue of the gradient of the objective $J'_n(\theta) \in \RR^b$ for each $\theta \in \Theta_0$, according to
\begin{equation*}
    J_n'(\theta) = 2 \inner{(C_n + \alpha_n I)^{-1/2} h_n(\theta)}{h'_n(\theta)} \,,
\end{equation*}
and note that in the case that $\rho(X_i;\theta)$ is differentiable at $\theta$ for $i \in [n]$ that $J_n'(\theta) = \nabla J_n(\theta)$.
In addition, we define linear operators $C : \mathcal H \to \mathcal H$ and $C_0 : \mathcal H \to \mathcal H$ according to
\begin{align*}
   (C h)(f) &= \e[\varphi(h)(Z)\tr  \rho(X;\tilde\theta) \rho(X;\tilde\theta)\tr  f(Z)] \\
   (C_0 h)(f) &= \e[\varphi(h)(Z)\tr  \rho(X;\theta_0) \rho(X;\theta_0)\tr  f(Z)]\,, 
\end{align*}
where $\tilde\theta$ is the probability limit of $\tilde\theta_n$ as specified by \cref{asm:tilde-theta}, and $\varphi$ is defined as in the definition of $C_n$ above. Given these definitions, we can now specify our additional assumptions and the asymptotic normality result.

This next additional assumption is a regularity condition on $D(X;\theta)$, which extends the properties of $\rho(X;\theta)$ specified in \cref{asm:rho-complexity} to $D(X;\theta)_j$ for each $j \in [b]$.

\begin{assumption}[Gradient Complexity]
\label{asm:continuous-deriv}
    Let $\Theta_0$ be the neighborhood of $\theta_0$ from \cref{asm:rho-deriv}.
    For each $i \in [m]$ and $j \in [b]$ we have $\sup_{x \in \Xcal, \theta \in \Theta_0} |D_i(x;\theta)_j| < \infty$, and that $D_i(X;\theta)_j$ is Lipschitz continuous in $\theta$ under $L_1$ norm. In addition, for each $i \in [m]$ and $j \in [b]$ the class $\{D_i(\cdot;\theta)_j : \theta \in \Theta\}$ is $\Pcal$-Donsker.
\end{assumption}

Next, we assume a certain non-degeneracy in the parametrization of the problem, locally near $\theta_0$.

\begin{assumption}[Non-degenerate $\Theta$]
\label{asm:non-degenerate-theta}
    For $\beta \in \mathbb R^b$, we have $\e[\sum_{j=1}^b \beta_j D(X;\theta_0)_j \mid Z] = 0$ almost surely if and only if $\beta = 0$.
\end{assumption}

\Cref{asm:non-degenerate-theta} is needed to ensure that the limiting asymptotic variance is finite and that the matrix $\Omega$ defined in the theorem statement below is invertible. It can be interpreted as the assumption that the parametrization of $\Theta$ is non-degenerate, since it requires that the functions $\e[\rho'_i(X;\theta_0) \mid Z]$ are linearly independent. Note that this assumption is somewhat lax, since if it were violated, it is likely possible we could re-parameterize the problem with a lower-dimensional $\Theta$ in order to avoid this issue.

Finally, we will need to introduce a couple of important definitions. We say that an estimator $\hat\theta_n$ for $\theta_0$ is \emph{asymptotically linear} if $\hat\theta_n = \e_n[\psi(X)] + o_p(n^{-1/2})$, for some $\psi$ satisfying $\e[\psi] = \theta_0$. n addition, we say that such an estimator is \emph{asymptotically normal} if $\sqrt{n}(\hat\theta_n - \theta_0)$ converges in distribution to a mean-zero Gaussian random variable, with some fixed covariance matrix.

With these additional assumptions and definitions, we are prepared to present our asymptotic normality result.

\begin{theorem}[Asymptotic Normality]
\label{thm:asym-norm}
    Let \cref{asm:rkhs,asm:regularity,asm:rho-complexity,asm:non-degenerate,asm:tilde-theta,asm:rho-deriv,asm:continuous-deriv,asm:non-degenerate-theta} be given, and suppose the regularization coefficient satisfies $\alpha_n = o(1)$ and $\alpha_n = \omega(n^{-p})$, where $p$ is the constant defined in \cref{asm:tilde-theta}. Then, for any $\hat\theta_n$ that satisfies $\|J'_n(\hat\theta_n)\| = o_p(n^{-1/2})$, we have that $\sqrt{n}(\hat\theta_n - \theta_0)$ is asymptotically linear and asymptotically normal, with covariance matrix $\Omega^{-1} \Delta \Omega^{-1}$, where $\Delta$ and $\Omega$ are defined according to
    \begin{align*}
        \Delta_{i,j} &= \inner{(C^{-1/2} C_0 C^{-1/2}) C^{-1/2} h'_i}{C^{-1/2} h'_j} \\
        \Omega &= \e\Big[ \EE[D(X;\theta_0) \mid Z] \tr  V(Z; \tilde\theta)^{-1} \EE[D(X;\theta_0) \mid Z] \Big]\,.
    \end{align*}
\end{theorem}

Note that this theorem requires an approximate first-order optimality condition, $\|J'_n(\hat\theta_n)\| = o_p(n^{-1/2})$, that is stronger than the approximate optimality condition in \cref{thm:consistency}. Although this condition may be difficult to interpret or verify in general, the following lemma provides some sufficient conditions.

\begin{lemma}[Sufficient Conditions for Approximate First-Order Optimality]
\label{lem:first-order-conditions}
Suppose that either (1) $\hat\theta_n \in \argmin_{\theta} J_n(\theta)$; or (2) $\rho(x;\theta)$ is twice continuously differentiable in $\theta$ for every $x \in \Xcal$, and $J_n(\hat\theta_n) = J_n(\theta_n^*) + o_p(1/n)$. Then, given the other conditions of \cref{thm:asym-norm}, we have $\|J'_n(\hat\theta_n)\| = o_p(n^{-1/2})$.
\end{lemma}

Comparing this result to comparable results in the literature leveraging more classical nonparametric approaches, we note that our differentiability condition in \cref{asm:rho-deriv} is weaker than the point-wise differentiability of $\rho(X;\theta)$ assumed by \citet{ai2003efficient}, but stronger than \citet{chen2009efficient} who only require differentiability of $\EE[\rho(X;\theta) \mid Z]$. We also note that, these two works further allow for non-parametric nuisance functions in addition to the asymptotically normal parametric component. Furthermore, we note that \cref{asm:non-degenerate-theta} is a standard condition in all of these works.

\subsection{Efficiency}

Next, we address the question of efficiency of these kernel VMM estimators. In order to present this theory, we first need to introduce the notions of \emph{regularity} and \emph{semiparametric efficiency}; we refer the reader to \citet{van2000asymptotic} for precise definitions. Roughly speaking, we say that an estimator $\hat\theta_n$ is \emph{regular} with respect to some model of distributions if it is sufficiently well behaved such that its asymptotic behavior is invariant to small perturbations (of size $O_p(n^{-1/2})$) to the data-generating distribution that remain inside the model. In addition, we say that $\hat\theta_n$ is \emph{semiparametrically efficient} with respect to a model of distributions if it is regular and achieves the minimum asymptotic variance among \emph{all} regular estimators (with respect to that model).

Given the complex form of the limiting covariance in \cref{thm:asym-norm} in terms of linear operators and inner products on $\mathcal H$, it is not immediately clear how large this covariance is and whether it is efficient under any conditions. Fortunately, the following theorem, which holds under no additional assumptions, justifies efficiency in the case that our prior estimate for $\theta_0$ is consistent.

\begin{theorem}[Efficiency]
\label{thm:efficiency}

Let the assumptions of \cref{thm:asym-norm} be given with $\tilde\theta = \theta_0$, and let $\hat\theta_n$ be any estimator that satisfies the conditions of \cref{thm:asym-norm}. Then, $\hat\theta_n$ is semiparametrically efficient with respect to the model given by \cref{eq:ident} and $\sqrt{n}(\hat\theta_n-\theta_0)$ is asymptotically normal with asymptotic covariance matrix $\Omega_0^{-1}$, where $\Omega_0$ is defined according to
\begin{equation*}
    \Omega_0 = \e\Big[ \EE[D(X;\theta_0) \mid Z] \tr  V(Z; \theta_0)^{-1} \EE[D(X;\theta_0) \mid Z] \Big]  \,.
\end{equation*}
\end{theorem}

This theorem immediately implies that such a kernel VMM estimator is not only efficient with respect to the class of all kernel VMM estimators, but that it achieves the semiparametric efficiency bound for solving \cref{eq:ident}. This is a very strong result, which ensures that these kernel VMM estimators inherit the efficiency properties that OWGMM estimators possess for standard moment problems, as was hoped.

Comparing against the efficiency results of the continuum GMM estimators of \citet{carrasco2000generalization}, which is the most similar approach to kernel VMM, ours is stronger. Specifically, they only justified that their estimator is efficient compared with other estimators in their class of continuum GMM estimators, while we have proven efficiency relative to \emph{all} possible regular estimators. That is, we achieve the same semiparametric efficiency as, \emph{e.g.}, \citet{ai2003efficient} and \citet{chen2009efficient}, who use fundamentally different sieve-based approaches.

Finally, we note that although this limiting covariance matrix has a somewhat complicated form, this form has a variational interpretation similar to the Kernel VMM estimator itself. We discuss this interpretation and how to use it to estimate the efficient asymptotic variance in \cref{sec:inference}.

\subsection{Implementing Kernel VMM Estimators}
\label{sec:multi-step-vmm}

Finally, we address some implementation considerations for kernel VMM estimators. 

Firstly, we note that the above theory does not provide any guidance on how to actually construct a prior estimate $\tilde\theta_n$ that has the required properties described in \cref{asm:tilde-theta}. In order to address this issue, we now present a concrete method for constructing such a $\tilde\theta_n$, which allows us to avoid explicitly assuming \cref{asm:tilde-theta}. Let us use the terminology that $\hat\theta_n$ is a $0$-step kernel VMM estimate if $\hat\theta_n$ is chosen as some arbitrary fixed value, which doesn't depend on the observed data. Then, for any integer $k > 0$, we say that $\hat\theta_n$ is a $k$-step kernel VMM estimate if $\hat\theta_n$ is computed by approximately solving \cref{eq:kernel-vmm} according to $J_n(\hat\theta_n) = \inf_{\theta \in \Theta} J_n(\theta) + o_p(1/n)$, with $\tilde\theta_n$ chosen as a $(k-1)$-step kernel VMM estimate. In other words, $\hat\theta_n$ is a $k$-step kernel VMM estimate if it is computed by iteratively approximately solving \cref{eq:kernel-vmm} $k$ times, with $\tilde\theta_n$ chosen as the previous iterate solution, starting from some arbitrary constant value.
This scheme is analogous to that of the $k$-step GMM estimator \citep{hansen1996finite}.
Given this definition, we have the following lemma:

\begin{lemma}
\label{lem:k-step-vmm}
    Suppose that $\tilde\theta_n$ is a $k$-step kernel VMM estimate for some $k>0$. Then, given all assumptions of \cref{thm:asym-norm} except for \cref{asm:tilde-theta}, it follows that $\tilde\theta_n$ satisfies the conditions of \cref{asm:tilde-theta} with $p=1/2$, and $\tilde\theta = \theta_0$.
\end{lemma}

Therefore, as long as we construct $\hat\theta_n$ as a $k$-step kernel VMM estimator as described above for some $k>1$, we are assured that \cref{asm:tilde-theta} will be met with $p=1/2$ and $\tilde\theta=\theta_0$. Given this and \cref{thm:efficiency}, we immediately have the following corollary for $k$-step kernel VMM estimators.

\begin{corollary}
    Suppose that $\hat\theta_n$ is calculated as a $k$-step kernel VMM estimate for some $k>1$. Then given \cref{asm:rkhs,asm:regularity,asm:rho-complexity,asm:non-degenerate,asm:rho-deriv,asm:continuous-deriv,asm:non-degenerate-theta}, and assuming that the regularization coefficient satisfies $\alpha_n = o(1)$ and $\alpha_n = \omega(n^{-1/2})$, it follows that $\hat\theta_n$ is semiparametrically efficient for $\theta_0$.
\end{corollary}

This corollary ensures that, given our regularity assumptions about $\mathcal F$ and the conditional moment problem itself, we can construct a specific $k$-step kernel VMM estimator that is semiparametrically efficient. The above also provides a valid specific choice of the regularization coefficient $\alpha_n$ that does not depend on unknown parameters.

Secondly, we address the fact that the cost function described in \cref{eq:kernel-vmm} is given by a supremum over the infinite $\mathcal F$, and provide a closed-form for the objective. By appealing to the representer theorem, and the factorization of $\mathcal F$ into the direct sum of $m$ RKHSs, we can establish the following lemma.
\begin{lemma}
\label{lem:kernel-vmm-closed-form}
Define the vector $\rho(\theta) \in \mathbb R^{n \cdot m}$ and the matrices $L \in \mathbb R^{(n \cdot m) \times (n \cdot m)}$ and $Q(\theta) \in \mathbb R^{(n \cdot m) \times (n \cdot m)}$ according to
\begin{align*}
    \rho(\theta)_{i,k} &= \rho_k(X_i; \theta),\quad
    L_{(i,k),(i',k')} = \indicator{k=k'} K_k(Z_i, Z_{i'}), \\
    Q(\theta)_{(i,k),(i',k')} &= \frac{1}{n} \sum_{j=1}^n K_k(Z_i, Z_j) \rho_k(X_j; \theta) K_{k'}(Z_{i'}, Z_j) \rho_{k'}(X_j; \theta) \,.
\end{align*}
Then, the cost function $J_n(\theta)$ being minimized by \cref{eq:kernel-vmm} is equivalent to
\begin{equation*}
    J_n(\theta) = \frac{1}{n^2} \rho(\theta)\tr  L (Q(\tilde\theta_n) + \alpha_n L)^{-1} L \rho(\theta)\,.
\end{equation*}
\end{lemma}

In other words, the kernel VMM estimator can be computed by minimizing a simple closed-form cost function, which is given by a particular convex quadratic form on the terms of the form $\rho_k(X_i; \theta)$ for $i \in [n]$ and $k \in [m]$.

In the special case of instrumental variable regression, where we are fitting the regression function within an RKHS ball, we can not only find a closed form solution for the cost function $J_n(\theta)$ to be minimized, but for the kernel VMM estimator itself. Specifically, we provide the following lemma, which follows by applying the representer theorem again.

\begin{lemma}
\label{lem:kernel-vmm-closed-form-iv}
Consider the instrumental variable regression problem, where $m=1$, $\rho(X;\theta) = Y - \theta(T)$, $\mathcal F$ is the RKHS with kernel $K_f$, and $\Theta$ is a ball of the RKHS with kernel $K_g$ with radius $r$ and centred at zero. In addition, let $Y$ denote the vector of outcomes $(Y_1,\ldots,Y_n)$, let $L_f$ and $L_g$ denote the kernel Gram matrices of $K_f$ and $K_g$ on the data $Z_1,\ldots,Z_n$ and $T_1,\ldots,T_n$, respectively, and define the $n \times n$ matrices $Q(\theta)$ and $M$ according to
\begin{align*}
    Q(\theta)_{i,i'} &= \frac{1}{n} \sum_{j=1}^n K_f(Z_i, Z_j) K_f(Z_{i'}, Z_j) (W_j - \theta(T_j))^2 \\
    M &= \frac{1}{n^2} L_f (Q(\tilde\theta_n) + \alpha_n L_f)^{-1} L_f \,.
\end{align*}
Then, we have $\hat\theta_n^{\text{K-VMM}} = \sum_{i=1}^n \beta^*_i K_g(\cdot, T_i)$, where
\begin{equation*}
    \beta^* = (L_g M L_g + \lambda_n L_g)^{-1} L_g M Y\,,
\end{equation*}
for some $\lambda_n \geq 0$ which depends implicitly on $r$, $K_f$, $K_g$, $\tilde\theta_n$, and the observed data.
\end{lemma}

The term $\lambda_n$ enters into the above equation via Lagrangian duality, since minimizing $J_n(\theta)$ over the RKHS ball with radius $r$ is mathematically equivalent to minimizing $J_n(\theta) + \lambda_n \norm{\theta}^2$ over the entire RKHS, for some implicitly defined $\lambda_n \geq 0$. In practice, however, when performing IV regression according to \cref{lem:kernel-vmm-closed-form-iv} we could freely select $\lambda_n$ as a hyperparameter instead of $r$. 
Superficially, the form of this estimator is similar to that of other recently proposed kernel-based estimators for IV regression \citep{singh2019kernel,muandet2020dual}. However, unlike those estimators, ours incorporates optimal weighting using the prior estimate $\tilde\theta_n$.

\section{Neural VMM Estimators}
\label{sec:neural-vmm}

We now consider a different class of VMM estimators, where the sequence of function classes $\mathcal F_n$ is given by a class of neural networks with growing depth and width. We will refer to estimators in this class as \emph{neural VMM} (N-VMM). Most generally, we will define the class of N-VMM estimators according to
\begin{equation}
\label{eq:neural-vmm}
    \hat\theta_n^{\text{N-VMM}} = \argmin_{\theta \in \Theta} \sup_{f \in \mathcal F_n} \Ucal_n^{\text{N-VMM}}(\theta,f) \,,
\end{equation}
where
\begin{equation*}
    \Ucal_n^{\text{N-VMM}}(\theta,f) = \e_n[f(Z)\tr  \rho(X; \theta)] - \frac{1}{4} \e_n[(f(Z)\tr  \rho(X;\tilde\theta_n))^2] - R_n(f) \,,
\end{equation*}
and $R_n(f)$ is some regularizer. In this section, we analyze $\text{N-VMM}$ for different choices of $R_n$. For simplicity, we will restrict our theoretical analysis to the case where $\mathcal F_n$ is a fully-connected neural network with ReLU activations and a common width in all layers, which allows us to use the universal approximation result of \citet[theorem 1]{yarotsky2017error}. Specifically, we fix a network architecture with $D_n$ hidden layers, each with $W_n$ neurons, with the final fully-connected layer connecting to the $m$ outputs. Then the class $\mathcal F_n$ is given by varying the weights on this network. We note that this choice is made for simplicity of exposition, but similar bounds could be given for different kinds of architectures, using other universal approximation results as in \emph{e.g.} \citet{yarotsky2017error,yarotsky2018optimal}.

\subsection{Neural VMM with Kernel Regularizer}
\label{sec:neural-kernel-vmm}

First, we consider the case where we regularize using some RKHS norm. Specifically, let $\mathcal F_K$ be a product of $m$ RKHSs satisfying \cref{asm:rkhs}, and let $\norm{f}_{n,K}=\inf_{f'\in\mathcal F_K\;:\;f'(Z_i)=f(Z_i)\;\forall i}\norm{f'}$ denote the minimum norm of any $f' \in \mathcal F_K$ that agrees with $f$ at the points $Z_1,\ldots,Z_n$. Note that
\begin{equation*}\textstyle
    \norm{f}_{n,K}^2 = \sum_{k=1}^m f_k\tr  K_k^{-1} f_k \,,
\end{equation*}
where $f_k = (f(Z_1)_k,\ldots,f(Z_n)_k)\tr $, and $K_k$ is the kernel Gram matrix on the data $Z_1,\ldots,Z_n$ using the kernel for the $k^\text{th}$ dimension of $\mathcal F_K$. Then, we will consider estimators of the form
\begin{equation}
\label{eq:neural-vmm-k}
    \hat\theta_n^{\text{NK-VMM}} = \argmin_{\theta \in \Theta} \sup_{f \in \mathcal F_n} \Ucal_n^{\text{NK-VMM}}(\theta,f) \,,
\end{equation}
where
\begin{equation*}
    \Ucal_n^{\text{NK-VMM}}(\theta,f) =  \e_n[f(Z)\tr  \rho(X; \theta)] - \frac{1}{4} \e_n[(f(Z)\tr  \rho(X;\tilde\theta_n))^2] - \frac{\alpha_n}{4} \norm{f}^2_{n,K} \,.
\end{equation*}

We note that if we were to replace $\mathcal F_n$ with $\mathcal F_K$ in \cref{eq:neural-vmm-k}, then this equation would be equivalent to \cref{eq:kernel-vmm}, since by the representer theorem regularizing by $\norm{f}_{\mathcal F_K}$ gives the same supremum over $f$ as regularizing by $\norm{f}_{n,K}$. 
Given this and the known universal approximation properties of neural networks, it may be hoped that if we grow the class $\mathcal F_n$ sufficiently fast, then the objective we are minimizing over $\theta$ in \cref{eq:neural-vmm-k} is approximately equal to that of \cref{eq:kernel-vmm} in a uniform sense over $\theta \in \Theta$. This, then, would hopefully imply that this neural VMM estimator is able to achieve the same desirable properties, in terms of consistency, asymptotic normality, and efficiency, as our kernel VMM estimators.

In order to formalize the above intuition, we first require the following assumption, which allows us to account for the rate of growth of the kernel Gram matrix inverses $K_k^{-1}$ in the results we give below.

\begin{assumption}[Inverse Kernel Growth]
\label{asm:k-growth}
There exists some deterministic positive sequence $k_n = \Omega(1)$, such that $\norm{K_i^{-1}}_2 = O_p(k_n)$ for each $i \in [m]$.
\end{assumption}

In addition, we require the following assumption on the rate of growth on the width $W_n$ and depth $D_n$ of $\mathcal F_n$, in order to ensure that we can approximate \cref{eq:kernel-vmm} sufficiently well.

\begin{assumption}[Neural Network Size]
\label{asm:network-size}
There exist constants $q\geq0,\,0<a<1/2$ and a sequence
$r_n = o(n^{-1-q} k_n^{-1})$
such that
$W_n = \omega(r_n^{-a} \log(r_n^{-1}))$ and $D_n = \omega(\log(r_n^{-1}))$.
\end{assumption}

Finally, in our results and discussion below we will define $J_n(\theta)$ to be the loss in $\theta$ minimized by $\hat\theta_n^{\text{NK-VMM}}$, and $J^*_n(\theta)$ to be the corresponding oracle loss if we were to replace $\mathcal F_n$ with $\mathcal F_K$.

\begin{lemma}
\label{lem:uniform-jn-approximation}
    Let \cref{asm:rkhs,asm:regularity,asm:rho-complexity,asm:k-growth,asm:network-size} be given. Then we have
    \begin{equation*}
        \sup_{\theta \in \Theta} \abs{J_n(\theta) - J_n^*(\theta)} = o_p(n^{-q})\,.
    \end{equation*}
\end{lemma}

This lemma follows by applying recent results on the size of a neural network required to uniformly approximate all functions of a given Sobolev norm \citep{yarotsky2017error}, and also older results that show that, under the conditions of \cref{asm:rkhs}, any RKHS ball has bounded Sobolev norm for any Sobolev space using more than $d_z / 2$ derivatives \citep{cucker2002mathematical}. 

Given this, we can immediately state the following theorem, which ensures that the theoretical results of our kernel VMM estimators carry over to our neural VMM estimators with kernel regularization.

\begin{theorem}
\label{thm:neural-vmm-properties}
    Let the assumptions of \cref{thm:consistency} and \cref{asm:k-growth,asm:network-size} be given. In addition, let $\hat\theta_n$ be any sequence that satisfies $J_n(\hat\theta_n) = \inf_{\theta \in \Theta}J_n(\theta) + o_p(n^{-q})$, where $q$ is the constant referenced in \cref{asm:network-size}. Then, in the case that these assumptions hold with $q=0$, we have $\hat\theta_n \to \theta_0$ in probability. 
    
    Furthermore, suppose in addition that the assumptions of \cref{thm:asym-norm} hold, and the above assumptions are strengthened to hold with $q = 1$. Then we have that $\sqrt{n}(\hat\theta_n - \theta_0)$ converges in distribution to a mean-zero Gaussian random variable, with covariance as given by \cref{thm:asym-norm}.
    
    Finally, assume that in addition $\tilde\theta = \theta_0$. Then the asymptotic variance of $\hat\theta_n$ is given by \cref{thm:efficiency}, and the estimator is semiparametrically efficient.
\end{theorem}

The proof of this theorem follows immediately from \cref{lem:uniform-jn-approximation}, since this Lemma and the Theorem's conditions ensure that $J_n^*(\hat\theta_n) = \inf_{\theta \in \Theta} J_n^*(\theta) + o_p(n^{-q})$. Therefore, we can directly apply \cref{thm:consistency,thm:asym-norm,thm:efficiency} to obtain these three results.

An immediate observation given this theorem is that, if we define $k$-step estimators as in \cref{sec:multi-step-vmm}, then by applying an identical argument we can construct efficient neural VMM estimators without having to explicitly make \cref{asm:tilde-theta}.

\subsection{Neural VMM with Other Regularizers}

Motivated by our theory above using kernel-based regulairzers, we now provide some discussion of general neural VMM estimators of the form given by \cref{eq:neural-vmm} for other choices of $R_n(f)$, and in particular we discuss how these estimators may be justified.

First, consider the case where the kernel Gram matrices $K_i$ for $i \in [m]$ are approximately equal to $\sigma_i I$, where $\sigma_i$ is some scalar and $I$ is the identity matrix. For example, this is the case if we use a Gaussian kernel with very small length scale parameter. In this case, we may reasonably approximate
\begin{equation}
\label{eq:f-norm-approx}\textstyle
    \norm{f}_{n,K} \approx \sum_{k=1}^m \frac{1}{\sigma_k} \sum_{i=1}^n f_k^2(Z_i)\,.
\end{equation}
That is, we could justify instead regularizing using some (possibly weighted) Frobenius norm of the matrix given by the values of the vector-valued $f$ at the $n$ data points. This form of regularization is much more attractive than that given by $\norm{f}_{n,K}$, since it doesn't involve the computation of inverse kernel Gram matrices, and it more naturally fits into estimators for \cref{eq:neural-vmm} given by some form of alternating stochastic gradient descent. We also note that this form of regularization, based on the Frobenius norm of $f$, is similar to that used by \citep{dikkala2020minimax}, although with some important differences; their proposed estimators do not include the $-(1/4) \e_n[(f(Z)\tr  \rho(X;\tilde\theta_n))^2$ term motivated by efficiency theory, and they only present theory on bounding the risk of their learned function given by $\hat\theta_n$, not on the consistency or semiparametric efficiency of the estimated $\hat\theta_n$. We discuss this comparison in more detail in \cref{sec:related-work} below.

Alternatively, we may heuristically justify leaving out the $R_n(f)$ term altogether, under the argument that neural network function classes naturally impose some smoothness constraints, and therefore optimizing over $\mathcal F_n$ is morally similar to optimizing over $\mathcal F_K$ with some norm constraint. This intuition can be made more concrete by noting that there is a rich literature showing equivalence between optimizing loss functions over neural network function classes, and optimizing the same loss over some norm-bounded RKHS class whose kernel is implicitly defined by the neural network architecture (see \emph{e.g.} \citet{shankar2020neural} and citations therein). However, we leave more specific non-heuristic claims on the performance of our neural VMM algorithms with $R_n(f)=0$ to future work.

\subsection{Implementing Neural VMM Estimators}
\label{sec:neural-vmm-implementation}

Regardless of the choice of the regularization term $R_n$, the question remains of how to actually solve \cref{eq:neural-vmm}. 
Past work \citep{bennett2019deep,bennett2020efficient} has solved this problem using the Optimistic Adam (OAdam) algorithm, which is a form of alternating stochastic gradient descent (that is, alternating between first-order gradient steps minimizing the game objective with respect to $\theta$, and maximizing the game objective with respect to $f$)
that has been designed to have good properties for solving minimax problems \citep{daskalakis2017training}.
These past works have proposed to do this by continuously updating $\tilde\theta_n$; that is, at each iteration of alternating stochastic gradient descent they set $\tilde\theta_n$ as the previous iterate solution.

Alternatively, there is a rich recent literature on other, potentially more efficient, methods for solving smooth game optimization problems along the lines of \cref{eq:kernel-vmm}. For example, see \citet{gidel2019negative,thekumparampil2019efficient,loizou2020stochastic,fiez2020implicit,lin2020near,lin2020gradient}, and references therein. Some or all of the approaches suggested in these recent works may lead to successful neural VMM implementations. However, we leave this more empirical investigation to future work, and in our experiments we focus on approaches based on OAdam with continuously updated $\tilde\theta_n$, as discussed above.

\section{Inference}
\label{sec:inference}

So far, we have developed both theory and algorithms for kernel and neural VMM estimators, providing conditions under which such estimators are consistent, asymptotically normal, and/or efficient. We now extend our efficient estimation theory to efficient inferential theory, focusing on the case of $\Theta \subseteq \mathbb R^b$.

Now, suppose we want to construct confidence intervals for $\psi(\hat\theta_n)$, for some $\psi : \mathbb R^b \mapsto \mathbb R$. This is a very general kind of quantity to consider, since, for example, if were interested in $(\hat\theta_n)_i$ for some $i \in [b]$, we could define $\psi(\theta) = \theta_i$. By the delta method, if $\hat\theta_n$ were an efficient estimate then the asymptotic variance of $\psi(\hat\theta_n)$ would be $\nabla\psi(\theta_0)\tr \Omega_0^{-1} \nabla\psi(\theta_0)$, where $\Omega_0^{-1}$ is the efficient covariance matrix defined in \cref{thm:efficiency}.   Therefore, this suggests that we could construct asymptotically calibrated Wald confidence intervals by estimating $\hat\beta_n\tr \Omega_0^{-1} \hat\beta_n$, for some data-driven $\hat\beta_n$. In particular, if $\nabla\psi(\theta_0)$ were known, which would be the case if $\psi$ were linear, then we could do this with $\hat\beta_n = \nabla\psi(\theta_0)$. Otherwise, we could do this using $\hat\beta_n = \nabla\psi(\hat\theta_n)$, where $\hat\theta_n$ is some consistent estimate of $\theta_0$ (such as a VMM estimate), which would be consistent for $\nabla\psi(\theta_0)$ given \cref{asm:continuous-deriv}.

In this section, we provide consistent algorithms for estimating $\beta\tr \Omega_0^{-1} \beta$ for arbitrary $\beta \in \mathbb R^b$, with analogous kernel and neural varieties of our algorithms. These consistent variance estimators can then immediately be used with the delta method, as discussed above, to construct asymptotically calibrated Wald confidence intervals for our efficient VMM estimators.

Our algorithms, which are presented in the next subsections, are motivated by the following key lemma.
\begin{lemma}
\label{lem:asymptotic-variance}
    Let $\Omega_0$ be defined as in \cref{thm:efficiency}, let the conditions of \cref{thm:efficiency} hold, and let $\nabla \rho(X;\theta) \in \mathbb R^{m \times b}$ denote the Jacobian of $\rho(X;\theta)$ with respect to $\theta$.  Then, for any vector $\beta \in \mathbb R^b$, we have
    \begin{align*}
        \beta\tr \Omega_0^{-1} \beta &= \sup_{\gamma \in \mathbb R^b} \gamma\tr \beta - \frac{1}{4} \gamma\tr \Omega_0 \gamma \\
        &= - \frac{1}{4} \inf_{\gamma \in \mathbb R^b} \sup_{f \in \mathcal F} \Ucal^\textup{Inf}(\gamma,f) \,,
    \end{align*}
    where
    \begin{equation*}
        \Ucal^\textup{Inf}(\gamma,f) = \e[f(Z)\tr \nabla \rho(X;\theta_0) \gamma] - \frac{1}{4} \e[(f(Z)\tr \rho(X;\theta_0))^2] -4 \gamma\tr \beta  \,.
    \end{equation*}
\end{lemma}

The first part of this lemma follows by applying a similar variational reformulation argument as in the proof of \cref{lem:owgmm}, and the second part follows by applying a similar argument again on the $\gamma\tr \Omega_0 \gamma$ term, given the definition of $\Omega_0$ from \cref{thm:efficiency}. More details are given in the appendix.

We note that the right hand side of \cref{lem:asymptotic-variance} has a very similar structure to the game objective of our VMM algorithms. Given this, the previous argument suggests that the asymptotic variance of any such $\psi(\hat\theta_n)$ could be estimated using approaches similar to our kernel and neural VMM estimation algorithms presented previously. In the remainder of this section, we build on this intuition, and present kernel- and neural-based algorithms for inference.

\subsection{Kernel Inference Algorithm}

First, we present an inference algorithm along the lines of our kernel VMM estimator. This algorithm is summarized by the following theorem.

\begin{theorem}
\label{thm:kernel-inference}
Let the conditions of \cref{thm:efficiency} be given, and let $\hat\theta_n$ be any corresponding efficient estimate of $\theta_0$. In addition, let $L$ and $Q(\theta)$ be defined as in \cref{lem:kernel-vmm-closed-form}, and define $D(\theta) \in \mathbb R^{(n \cdot m) \times b}$ and $\Omega_n \in \mathbb R^{b \times b}$ according to
\begin{align*}
    D_{(i,k),j}(\theta) &= \frac{\partial}{\partial \theta_j} \rho_k(X_i; \theta), \quad
    \Omega_n = \frac{1}{n^2} D\tr L (Q(\hat\theta_n) + \alpha_n L)^{-1} L D \,,
\end{align*}
where $\alpha_n$ is any sequence satisfying the assumptions of \cref{thm:efficiency}. Then $\Omega_n \to \Omega_0$ in probability.
\end{theorem}

We note that an immediate corollary of this theorem is that, for any continuously differentiable $\psi$ and an efficient $\hat\theta_n$ such as our VMM estimators, $\nabla\psi(\hat\theta_n)\tr \Omega_n^{-} \nabla\psi(\hat\theta_n)$ is consistent for the asymptotic variance of $\psi(\hat\theta_n)$, which is an efficient estimate of $\psi(\theta_0)$, where $\Omega_n^{-}$ denotes the pseudo-inverse of $\Omega_n$. This follows trivially by the continuous mapping theorem and Slutsky's theorem, since by assumption $\Omega_0$ is invertible.
An advantage of this algorithm is that it allows easy estimation of the entire covariance matrix $\Omega_0^{-1}$, from which the asymptotic variance of any single-dimensional function of $\hat\theta_n$ can instantly be estimated without applying any additional variational algorithms. 

\subsection{Neural Inference Algorithm}

Our neural inference algorithm is similar in nature to our neural VMM estimator, and is given by the following smooth game
\begin{equation}
\label{eq:neural-inference}
    v_n(\beta) = -\frac{1}{4} \inf_{\gamma \in \mathbb R^b} \sup_{f \in \mathcal F_n} \Ucal_n^{\textup{Inf}}(\gamma,f) \,,
\end{equation}
where
\begin{equation*}
    \Ucal_n^\textup{Inf}(\gamma,f) = \e[f(Z)\tr \nabla \rho(X;\theta_0) \gamma] - \frac{1}{4} \e[(f(Z)\tr \rho(X;\theta_0))^2] -4 \gamma\tr \beta  - R_n(f) \,,
\end{equation*}
$R_n$ is a regularizer for $f$, and $\mathcal F_n$ is a sequence of neural net classes. That is, compared with the true equation for the asymptotic variance in \cref{lem:asymptotic-variance}, we replace true expectations with empirical ones, we replace $\Fcal$ with $\Fcal_n$, and we regularize $f$.
Then, given \cref{lem:asymptotic-variance}, we expect $v_n(\beta)$ to be a reasonable estimator for $\beta\tr \Omega_0^{-1} \beta$. Furthermore, following the argument presented at the beginning of \cref{sec:inference}, we expect $v_n(\hat\beta_n)$ to be a reasonable estimator for the (efficient) asymptotic variance of $\psi(\hat\theta_n)$ if $\hat\beta_n$ is consistent for $\nabla\psi(\theta_0)$. We note that, unlike for our neural VMM estimator, we do not provide any theoretical guarantees for this algorithm, due to some additional technical complications; unlike the game objective being solved by neural VMM, the space being minimized over for $\gamma$ is unbounded, which complicates the technical argument by universal approximation we used for neural VMM. We leave this theoretical question to future work. However, we note that in our inference experiments in \cref{sec:experiments} this method seems to work well.

Unlike our kernel inference algorithm, this approach has the disadvantage that it requires solving a separate optimization problem for every given scalar parameter $\psi$. In practice, though, this may be alleviated by the practical strengths of neural methods, as discussed previously. 
In addition, as with our neural VMM algorithm, we may regularize for example by using a kernel-based norm or the Frobenius norm of $\{f(Z_1),\ldots,f(Z_n)\}$, or we may omit this regularization term entirely.

\section{Examples}

Next, let us provide some concrete examples of our theory, in order to demonstrate how the assumptions for our consistency and asymptotic normality theory may be satisfied. For each example, we do not discuss \cref{asm:rkhs,asm:tilde-theta,asm:k-growth,asm:network-size} explicitly, as these govern design choices for the algorithm that can be generically satisfied given the other assumptions.

\subsection{Nonparametric Instrumental-Variable Regression}

First, let us consider a specific nonparametric instrumental regression example, which instantiates \cref{ex:iv} from \cref{sec:intro}. Specifically, we will consider conditions under which our consistency result \cref{thm:consistency} applies. Let us consider the data generating process $Y = g(T;\theta_0) + \epsilon$, where $\EE[\epsilon \mid Z] = 0$. We assume that $Z \in \RR^{d_z}$, $\|Z\|_\infty < \infty$, $\|Y\|_\infty < \infty$, and $\VV[\epsilon \mid Z] \geq \lambda$ almost surely, for some fixed $\lambda > 0$. Let us also suppose that $g(T;\theta)$ is $L(T)$-Lipschitz continuous in $\theta$, where $\EE[L(T)^2] < \infty$, that $\sup_{\theta \in \Theta} \|g(T;\theta)\|_\infty < \infty$, and that $\Gcal = \{g(\cdot;\theta) : \theta \in \Theta\}$ is a Donsker class. As one example, these conditions would be satisfied if $\Gcal$ were given by the class of all monotonic functions on $T$ such that $\|g(X;\theta)\|_\infty \leq b$ for some fixed $b < \infty$, with the norm on $\Theta$ given by $\|\theta' - \theta\| = \|g(T;\theta') - g(T;\theta)\|_\infty$. As a second example, $\Gcal$ could be a norm-bounded RKHS satisfying the conditions of \cref{asm:rkhs}, with the norm on $\Theta$ given by the corresponding RKHS norm.

First, given the conditions on $Z$ in this example, \cref{asm:regularity} is trivial. Second, given the conditions on the regression class $\Gcal$, along with the assumption that $\|Y\|_\infty < \infty$, \cref{asm:rho-complexity} trivially follows by applying lemma 9.14 in \citet{kosorok2007introduction}. Finally, assuming that the prior estimate $\tilde\theta_n$ comes from some arbitrary consistent methodology, then \cref{asm:non-degenerate} only needs to hold for $\theta = \theta_0$. In this case, this is ensured under the above condition on the conditional variance of $\epsilon$, since $V(Z;\theta_0) = \VV[\epsilon \mid Z]$. Given this, we have consistency via \cref{thm:consistency}.

\subsection{Nonparametric Instrumental-Variable Quantile Regression}

Next, let us consider a specific nonparametric instrumental regression example, which instantiates \cref{ex:iv-quantile} from \cref{sec:intro}. Again, we will consider conditions under which consistency holds. Let us consider the data generating process $Y = g(T;\theta_0) + \epsilon$, where $\textup{Prob}(\epsilon \leq 0 \mid Z) = p$ for almost everywhere $Z$, and $\rho(X;\theta) = \indicator{Y \leq g(X;\theta)} - p$. Again, we assume that $Z \in \RR^{d_z}$, and $\|Z\|_\infty < \infty$. In this case, we will assume that $\Gcal = \{g(\cdot;\theta) : \theta \in \Theta\}$ is some regression class that is Donsker under the supremum norm $\|\theta' - \theta\| = \|g(X;\theta') - g(X;\theta)\|_\infty$, and that $Y$ has bounded density. 

Again, given the conditions on $Z$ in this example, \cref{asm:regularity} is trivial, and the Donsker part of \cref{asm:rho-complexity} follows from the fact that $\Gcal$ is Donsker by lemma 9.14 of \citet{kosorok2007introduction}. Also, we have $\EE[|\rho(X;\theta') - \rho(X;\theta)|] = \textup{Prob}(\min(g(X;\theta'), g(X;\theta)) \leq Y \leq  \max(g(X;\theta'), g(X;\theta))$. Now, since by assumption $Y$ has bounded density, it easily follows that there exists some constant $L$ such that $\textup{Prob}(\min(g(X;\theta'), g(X;\theta)) \leq Y \leq  \max(g(X;\theta'), g(X;\theta)) \leq L \|g(X;\theta') - g(X;\theta)\|_\infty$, which gives us the required Lipschitz continuity under $L_1$ norm. Also, the required boundedness is trivial since $|\rho(X;\theta)| \in \{-p, 1-p\}$, so we have \cref{asm:rho-complexity}. Finally, we have $V(Z;\theta_0) = \EE[(\indicator{\epsilon \leq 0} - p)^2 \mid Z] = p - p^2$ almost surely, and therefore $\|V(Z;\theta_0)^{-1}\|_\infty \leq (p - p^2)^{-1} < \infty$, which gives us the \cref{asm:non-degenerate}, again as long as the prior estimate $\tilde\theta_n$ is consistent. Therefore, again we have consistency via \cref{thm:consistency}.

\subsection{Parametric Instrumental-Variable Mean and Expectile Regression}

Next, we will consider a parametric expectile (including mean)  regression example \citep{newey1987asymmetric,sobotka2013estimating}, where we can establish both consistency, asymptotic normality, and efficiency. For this example, we will assume that the data generating process is again given by $Y = g(T;\theta_0) + \epsilon$, where $\epsilon$ instead satisfies $p \EE[\epsilon \indicator{\epsilon \geq 0} \mid Z] = (1-p) \e[-\epsilon \indicator{\epsilon < 0} \mid Z]$ for some $p \in (0,1)$. For $p=0.5$, we get the usual mean regression. Here, we let $\rho(X;\theta) = w(X;\theta)(Y - g(T;\theta))$, where $w(X;\theta) = p \indicator{Y \geq g(T;\theta)} + (1-p) \indicator{Y < g(T;\theta)}$, and the goal is to find the unique $\theta_0 \in \Theta$ such that $\EE[\rho(X;\theta_0) \mid Z] = 0$. Note that this problem that can be seen as a mid-point between standard instrumental variable regression and instrumented quantile regression. Similar to the above example, let us suppose that $Z \in \RR^{d_z}$, $\|Z\|_\infty < \infty$, $\|Y\|_\infty < \infty$, and $\VV[\epsilon \mid Z] \geq \lambda$ almost surely, for some fixed $\lambda > 0$. For this example, we will further assume that $T \in \RR^{d_t}$, $\|T\|_\infty < \infty$, the regression class is given by $g(t;\theta) = \theta^\top t$, where $\Theta = \{\theta \in \RR^{d_t} : \|\theta\|_2 \leq b\}$ for some $b < \infty$, and that the matrix $\EE[\EE[T \mid Z] \EE[T \mid Z]^\top]$ is full-rank.

Again, given the conditions on $Z$ in this example, \cref{asm:regularity} is trivial. Similarly, given the boundedness of $\Theta$ and $T$, and the fact that $\|w(X;\theta)\|_\infty \leq 1$, along with lemma 9.14 of \citet{kosorok2007introduction}, we easily have that \cref{asm:rho-complexity} holds. In addition, we have $V(Z;\theta_0) \geq \min(p,1-p)^2 \VV[\epsilon \mid Z]$, and so \cref{asm:non-degenerate} follows from our minimum conditional variance assumption as in the previous example. Therefore, we can establish consistency via \cref{thm:consistency}.

Next, under the additional assumption that $Y$ and $T$ both have bounded probability density, then so does $Y - g(T;\theta)$ for every $\theta \in \Theta$. Therefore, we can apply \cref{lem:non-smooth-rho} with $\phi(X;\theta) = Y - g(T;\theta)$ in order to establish \cref{asm:rho-deriv} with $D(X;\theta) = T$, which we note trivially satisfies the conditions of \cref{asm:continuous-deriv}. Finally, since $\EE[\EE[T \mid Z] \EE[T \mid Z]^\top]$ is assumed to be full rank, we have $\beta^\top \EE[\EE[T \mid Z] \EE[T \mid Z]^\top] \beta = \EE[\EE[\beta^\top D(X;\theta_0) \mid Z]^2] > 0$ for every non-zero $\beta$, which establishes \cref{asm:non-degenerate-theta}. Therefore, we also have asymptotic normality via \cref{thm:asym-norm}, and under the condition that the prior estimate $\tilde\theta_n$ was consistent we have semiparametric efficiency via \cref{thm:efficiency}.

\section{Experiments}
\label{sec:experiments}

We now present a series of experiments to demonstrate our proposed methodologies. We present two kinds of experiments. First, we test the finite-sample performance of our kernel and neural VMM algorithms on a range of synthetic conditional moment problems. In this experiment, we compare their performance with the classical sieve minimum distance (SMD) approach of \citet{ai2003efficient}, which is a sieve-based method that has previously been proposed as a semiparametrically efficient approach to solving generic conditional moment problems. In addition, we compare their performance with the recently proposed maximum moment restriction (MMR) algorithm of \citet{zhang2020maximum}, which as discussed in \cref{sec:related-work} is equivalent to the limit of our kernel VMM algorithm in the limit as $\alpha_n \to \infty$. Second, we test our proposed inference algorithms on a subset of these scenarios, evaluating the quality of the resulting confidence intervals for different variations of our estimation and inference algorithms. Code for reproducing all experiments is available at \href{https://github.com/CausalML/VMM}{https://github.com/CausalML/VMM}.

\subsection{Estimation Experiments}

\subsubsection{Estimation Scenarios}

\paragraph{\simpleiv} This is a simple parametric instrumental variable regression scenario, based on a simple data generating process where
\begin{align*}
    Z &= \sin(\pi U / 10) \\
    T &= -0.75 U + 3.5 H + 0.14 \eta - 0.6 \\
    Y &= g(T;\theta_0) + - 10 H + \epsilon \,.
\end{align*}
In this setup, $\eta$ and $H$ are exogenous iid $\mathcal N(0,1)$ variables, and $U$ is an exogenous iid $\text{Uniform}(-5,5)$ random variable, and each of $T$, $Z$, and $Y$, are scalars. We note that the random variable $H$ introduces endogeneity.  Furthermore, we have 
and $g(t;\theta) = \theta_1+\theta_2 t+\theta_3 t^2$ where $\theta\in\mathbb R^3$, with the true parameter value given by $\theta_0 = [0.5, 3.0, -0.5]$.  In this scenario, the conditional moment equation to be solved is $\e[Y - g(T;\theta_0) \mid Z] = 0$; that is, we have $X = (T,Y,Z)$, and $\rho(X;\theta) = Y - g(T;\theta)$.
Note that in this scenario the relationship between treatment and instruments is nonlinear.

\paragraph{\heteroiv} This is a more challenging instrumental variable regression scenario, which introduces a more complex nonlinear regression function class and heteroskedastic noise. It follows a similar data generating process to the prior \simpleiv scenario, except here we have
\begin{align*}
    Z &= (U_1, U_2) \\
    T &= 0.75(Z_1 + \abs{Z_2}) + 1.25 H + 0.05 \eta \\
    Y &= g(T;\theta_0) + 5 H + 0.1 \text{softplus}(Z_1 + \abs{Z_2}) \eta \,,
\end{align*}
where again $\eta$ and $H$ are iid $\mathcal N(0,1)$ distributed, and each of $U_1$ and $U_2$ are iid $\text{Uniform}(-5,5)$ distributed. We also note the ``softplus'' activation function is defined according to $\softplus(x) = \log(1 + \exp(x))$. In this case, we have $\theta \in \mathbb R^4$, and our regression class is defined according to

\begin{equation*}\textstyle
    g(t;\theta) = \theta_2 + \theta_3 (t - \theta_1) + \frac{\theta_4-\theta_3}{2} \softplus(2 (t - \theta_1))\,.
\end{equation*}
That is, our regression class is a smoothed version of a hinge function with slopes $\theta_3$ and $\theta_4$ and hinge point at $(\theta_1,\theta_2)$. The true parameter value is given by $\theta_0 = [2.0, 3.0, -0.5, 3.0]$. As with our \simpleiv scenario, the conditional moment restriction is given by $\e[Y - g(T;\theta_0) \mid Z] = 0$.

We note that although the regression residual is \emph{not} independent of the instruments $Z$ in this setting, it is mean-indepedent, since $\e[S \epsilon] = \e[\e[S \epsilon \mid Z]] = 0$. That is, we have heteroskedastic noise with respect to our instruments, which makes achieving efficiency more challenging.

\paragraph{\policylearning}

Finally, this scenario is based on learning optimal binary treatment policies from surrogate loss reductions, following \citet{bennett2020efficient}. 
Let $T \in \{-1,1\}$ denote the binary treatment variable, $Z$ denote individual covariates, $Y(t)$ denote the potential outcome for the individual that would occur if (possibly counter to fact) treatment $t$ were assigned, and $Y = Y(T)$ denote the actual outcome. Then, given logged data where treatments were decided using some randomized policy, and some well-specified parametric class of deterministic treatment policies $\Pi = \{\pi_{\theta} : \theta \in \Theta\}$, the task is to estimate the parameters of the \emph{optimal} policy within $\Pi$ That is, we wish to estimate $\theta_0 = \argmax_{\theta \in \Theta} \e[Y(\pi(Z;\theta))]$, where $\pi(z;\theta)$ denotes the treatment assigned by policy $\pi_\theta$ given $Z=z$. For this problem, we assume the following data generating process:
\begin{align*}
    Z &\sim \mathcal N(0, 1) \times \mathcal N(0, 1) \\
    T &\sim 2 \text{Bernoulli}(e(Z)) - 1 \\
    Y(t) &= \mu_t(Z) + \sigma_t(Z) \epsilon_t \quad \forall t \in \{-1,1\} \,,
\end{align*}
where $\eta_1$ and $\eta_{-1}$ are iid $\mathcal N(0,1)$ variables. The functions $e$, $\mu_{-1}$, and $\mu_1$ are all given by quadratic forms on $Z$, and the functions $\sigma_{-1}$ and $\sigma_1$ are given by quadratic forms on $Z$ with softplus activation; exact coefficients for these functions are provided in the supplement.

Now, assume that the policy class $\Pi$ is defined according to some parametric utility function $g(\cdot; \theta)$, where $\pi_{\theta} \in \Pi$ assigns an individual with covariates $z$ to treatment 1 if and only if $g(z; \theta) \geq 0$, else it assigns the individual to treatment $-1$. Then, under some assumptions outlined in \citet{bennett2020efficient} regarding correct specification with respect to the logistic regression surrogate loss, the problem of estimating $\theta_0$ is described by the conditional moment problem $\e[\abs{W} (\expit(g(Z;\theta_0)) - \indicator{W>0}) \mid Z] = 0$, where the weighting variable $W$ is defined according to
\begin{equation*}\textstyle
    W = \mu_1(Z) - \mu_{-1}(Z) + \frac{T (Y - \mu_T(Z))}{T e(Z) + (1-T)/2} \,.
\end{equation*}

In our experiments, we have $\theta\in\mathbb R^6$, and $g(z;\theta) = \theta^T \phi(z)$, where $\phi(z) = (1, z_1, z_2, z_1^2, z_2^2, z_1 z_2)$ gives a quadratic feature expansion of $z$. Given this and the fact that $\mu_1$ and $\mu_{-1}$ are quadratic forms, it easily follows that the optimal parameters are given by $\theta_0 = \theta_1 - \theta_{-1}$, where $\theta_1$ and $\theta_{-1}$ are the parameter vectors describing the quadratic forms $\mu_1$ and $\mu_{-1}$ respectively.

Note that, following \citet{bennett2020efficient}, in practice when estimating $\theta_0$ we estimate the weights $W$ using plugin estimators for $e$, $\mu_{-1}$, and $\mu_1$, which we fit using flexible neural nets. For fair comparison, in our experiments all methods use the same estimated weights $W$. 

\subsubsection{Estimation Methods}

\paragraph{KernelVMM} We implemented a 2-step kernel VMM estimator, as described in \cref{sec:multi-step-vmm}. For the first step, $\tilde{\theta}_n$ is chosen randomly. For a kernel function, we used the same kernel function as used by \citet{zhang2020maximum}, which is given by the average of three Gaussian kernels with automatically selected data-driven bandwidths; we provide details in the supplement.
In all cases, we optimize $\hat\theta_n$ by minimizing the cost function described in \cref{lem:kernel-vmm-closed-form} using L-BFGS, and we experimented with range of values of $\alpha_n$.

\paragraph{NeuralVMM} We implemented a neural VMM estimator by optimizing the minimax objective described by \cref{eq:neural-vmm} using alternating stochastic gradient descent, using the optimistic Adam (OAdam) optimizer \citep{daskalakis2017training}, as discussed in \cref{sec:neural-vmm-implementation}. 
In all cases we used a simple fixed 3-layer fully connected architecture for our $f$ network. 
As discussed in \cref{sec:neural-vmm} we use a Frobenius-norm style regularization term of the form
$R_n(f) = (\lambda_n / nm) \sum_{i=1}^n \sum_{k=1}^m f_k(Z_i)^2$,
and we experimented with a wide range of $\lambda_n$. We also experimented with explicit kernel-based regularization as in \cref{sec:neural-kernel-vmm}, however we found that it performed extremely poorly in practice due to extreme gradient values, so we did not include it in our main experiments.
We describe additional details, such as hyperparameters, network architectures, and early stopping, in the supplement.

\paragraph{MMR} The maximum moment restriction (MMR) algorithm was originally developed by \citet{zhang2020maximum} for the instrumental variable regression problem, but easily extends to other conditional moment problems.
One particular form of their algorithm (we discuss the more general form in \cref{sec:related-work}) is given by minimizing the objective function $J^{\text{MMR}}_n(\theta) = \frac{1}{n^2} \rho(\theta)\tr L \rho(\theta)$, where $L$ and $\rho(\theta)$ are defined as in \cref{lem:kernel-vmm-closed-form}. As discussed in \cref{sec:related-work} this is equivalent to KernelVMM in the limit as $\alpha_n \to \infty$.
Given this deep connection, we also present results for this version of the MMR algorithm for all scenarios, using the same kernel function as for KernelVMM. As with our KernelVMM method, we minimize this objective using L-BFGS.

\paragraph{SMD} The sieve minimum distance (SMD) method of \citet{ai2003efficient} applied to our problem given by \cref{eq:ident} (it in fact applies to a more general moment problem with infinite-dimensional nuisance components; see \cref{sec:related-work}) is based on minimizing the objective function $J_n^{\text{SMD}}(\theta) = \e_n[q_n(Z;\theta)\tr \Gamma_n(Z)^{-} q_n(Z;\theta)]$,
where $q_n(z;\theta)$ is a nonparametric sieve regression estimate of $\e[\rho(X;\theta) \mid Z=z]$, $\Gamma_n(z)$ is \emph{any} consistent estimate of $\e[\rho(X;\theta_0) \rho(X;\theta_0)\tr \mid Z=z]$, and $\Gamma_n(Z)^{-}$ denotes the pseudo-inverse of $\Gamma_n(Z)$.
The past work proposes various sieve-based approaches for $q_n(Z;\theta)$, but isn't prescriptive about the methodology for computing $\Gamma_n$. Given this, we experimented with various SMD estimators, using B-splines for $q_n(Z;\theta)$, and multiple approaches for $\Gamma_n$: (1) \emph{Identity}, in which we simply set $\Gamma_n(z) = I \ \forall \ z$; (2) \emph{Homoskedastic}, in which we set $\Gamma_n = \e_n[\rho(X;\tilde\theta_n) \rho(X;\tilde\theta_n)\tr] \ \forall z$; and (3) \emph{Heteroskedastic}, in which we we fit a diagonal $\Gamma_n(Z)$ by regressing $\rho(X;\tilde\theta_n)^2_i$ on $Z$ for each $i \in [m]$ using neural networks. We provide additional details in the supplement.

\paragraph{OWGMM}

The OWGMM estimator follows the method described in \cref{sec:owgmm}, for a flexible set of basis functions $f_1,\ldots,f_k$. As with the SMD method we chose these set of basis functions using B-splines, as this allowed for a very rich and flexible class of moment conditions. Again, we provide additional details in the supplement.

\paragraph{NCB}

Finally, we implemented a simple non-causal baseline (NCB) that estimates $\theta_0$ by ignoring $Z$ and instead trying to solve $\e[\rho(X;\theta_0) \mid X] = 0$. For example, for our instrumental variable regression scenarios this corresponds to assuming there is no endogeneity in the treatments $T$. For this baseline, we simply minimize the objective $J_n^{\text{NCB}}(\theta) = \e_n[\rho(X;\theta)^2]$,
which we implement using L-BFGS.

\subsubsection{Estimation Results}

\begin{table}\centering\footnotesize
    \begin{minipage}[b]{1.0\textwidth}
    \centering
    \begin{tabular}{clcccccc}
         \hline
         \multicolumn{2}{c}{\multirow{2}{*}{Method}} &  \multicolumn{6}{c}{$n$} \\
         & & 200 & 500 & 1,000 & 2,000 & 5,000 & 10,000 \\
         \hline
         \multirow{6}{*}{K-VMM} & $\alpha_n=0$ & $>100$ & $8.8\pm42.7$ & $>100$ & $.67\pm1.2$ & $.23\pm.29$ & $.14\pm.16$ \\
         & $\alpha_n=10^{-8}$ & $5.1\pm7.0$ & $2.8\pm3.0$ & $2.6\pm5.3$ & $3.2\pm16.5$ & $.25\pm.32$ & $.17\pm.23$ \\
         & $\alpha_n=10^{-6}$ & $5.5\pm7.0$ & $2.5\pm2.7$ & $1.7\pm3.0$ & $.78\pm1.3$ & $.24\pm.33$ & $.14\pm.16$ \\
         & $\alpha_n=10^{-4}$ & $5.5\pm7.6$ & $2.5\pm3.2$ & $1.8\pm2.9$ & $.72\pm1.3$ & $.25\pm.32$ & $.14\pm.16$ \\
         & $\alpha_n=10^{-2}$ & $6.0\pm8.3$ & $2.7\pm3.1$ & $1.7\pm2.4$ & $.72\pm1.2$ & $.26\pm.34$ & $.14\pm.17$ \\
         & $\alpha_n=1$ & $11.1\pm21.2$ & $4.1\pm6.6$ & $2.1\pm2.8$ & $.75\pm1.1$ & $.34\pm.41$ & $.16\pm.21$ \\
         \hdashline
         \multirow{4}{*}{N-VMM} & $\lambda_n=0$ & $2.5\pm2.0$ & $1.6\pm1.9$ & $.93\pm1.2$ & $.42\pm.65$ & $.16\pm.21$ & $.10\pm.14$ \\
         & $\lambda_n=10^{-2}$ & $2.2\pm1.9$ & $2.1\pm2.6$ & $.74\pm.99$ & $.42\pm.66$ & $.17\pm.23$ & $.10\pm.12$ \\
         & $\lambda_n=1$ & $2.1\pm2.0$ & $2.1\pm2.1$ & $.94\pm1.2$ & $.39\pm.65$ & $.18\pm.26$ & $.11\pm.12$ \\
         \hdashline
         \multirow{3}{*}{SMD} & Identity & $4.2\pm6.5$ & $2.5\pm3.6$ & $1.8\pm3.0$ & $.68\pm1.0$ & $.24\pm.31$ & $.15\pm.19$ \\
         & Homo & $4.2\pm6.5$ & $2.5\pm3.6$ & $1.8\pm3.0$ & $.68\pm1.0$ & $.24\pm.32$ & $.15\pm.19$ \\
         & Hetero & $4.3\pm5.7$ & $2.4\pm3.3$ & $1.7\pm2.6$ & $.66\pm1.0$ & $.24\pm.31$ & $.15\pm.18$ \\
         \hdashline
         MMR & & $17.7\pm28.0$ & $5.6\pm9.2$ & $2.8\pm3.7$ & $.83\pm1.1$ & $.37\pm.45$ & $.17\pm.23$ \\
         \hdashline
         OWGMM & & $3.1\pm5.3$ & $2.3\pm4.1$ & $1.7\pm2.0$ & $.85\pm1.0$ & $.33\pm.42$ & $.20\pm.24$ \\
         \hdashline
         NCB & & $6.2\pm1.3$ & $6.0\pm.71$ & $5.8\pm.45$ & $5.8\pm.47$ & $5.8\pm.25$ & $5.8\pm.20$ \\
         \hline
    \end{tabular} 
    \centering{\sffamily\small\begin{enumerate}\item[(a)] \simpleiv \end{enumerate}}
    \scriptsize
    \end{minipage}\hfill\hspace{0.3cm}
    \linebreak
    \begin{minipage}[b]{1.0\textwidth}
    \centering
    \begin{tabular}{clcccccc}
         \hline
         \multicolumn{2}{c}{\multirow{2}{*}{Method}} &  \multicolumn{6}{c}{$n$} \\
         & & 200 & 500 & 1,000 & 2,000 & 5,000 & 10,000 \\
         \hline
         \multirow{6}{*}{K-VMM} & $\alpha_n=0$ & $>100$ & $3.8\pm5.5$ & $>100$ & $.63\pm1.4$ & $.24\pm.29$ & $.09\pm.18$ \\
         & $\alpha_n=10^{-8}$ & $>100$ & $>100$ & $1.3\pm2.2$ & $.63\pm2.0$ & $.21\pm.23$ & $.06\pm.05$ \\
         & $\alpha_n=10^{-6}$ & $8.7\pm22.9$ & $2.0\pm2.6$ & $.78\pm.98$ & $.35\pm.50$ & $.22\pm.27$ & $.06\pm.05$ \\
         & $\alpha_n=10^{-4}$ & $9.9\pm27.6$ & $1.9\pm2.2$ & $.79\pm.96$ & $.35\pm.45$ & $.21\pm.26$ & $.05\pm.05$ \\
         & $\alpha_n=10^{-2}$ & $9.1\pm19.7$ & $2.6\pm3.6$ & $1.1\pm1.3$ & $.40\pm.49$ & $.21\pm.23$ & $.06\pm.06$ \\
         & $\alpha_n=1$ & $10.1\pm15.5$ & $5.2\pm7.0$ & $3.5\pm5.8$ & $2.5\pm4.7$ & $1.6\pm1.5$ & $1.4\pm1.5$ \\
         \hdashline
         \multirow{4}{*}{N-VMM} & $\lambda_n=0$ & $9.3\pm3.7$ & $5.3\pm2.8$ & $2.8\pm1.6$ & $1.9\pm1.3$ & $1.2\pm.84$ & $.68\pm.64$ \\
         & $\lambda_n=10^{-4}$ & $8.2\pm4.0$ & $5.4\pm2.5$ & $2.9\pm1.7$ & $1.7\pm1.3$ & $1.1\pm.80$ & $.71\pm.68$ \\
         & $\lambda_n=1$ & $7.3\pm2.7$ & $4.9\pm2.1$ & $2.7\pm1.9$ & $2.0\pm1.3$ & $1.1\pm.84$ & $.67\pm.68$ \\
         \hdashline
         \multirow{3}{*}{SMD} & Identity & $>100$ & $>100$ & $>100$ & $>100$ & $>100$ & $>100$ \\
         & Homo & $>100$ & $>100$ & $>100$ & $>100$ & $>100$ & $>100$ \\
         & Hetero & $>100$ & $>100$ & $>100$ & $>100$ & $>100$ & $>100$ \\
         \hdashline
         MMR & & $10.3\pm1.9$ & $10.2\pm1.2$ & $9.7\pm1.2$ & $9.8\pm.85$ & $9.7\pm.70$ & $9.6\pm.60$ \\
         \hdashline
         OWGMM & &$>100$ &$>100$ &$>100$ &$>100$ &$>100$ &$>100$ \\
         \hdashline
         NCB & & $9.1\pm6.7$ & $8.8\pm5.1$ & $7.6\pm3.0$ & $7.9\pm2.4$ & $7.7\pm1.2$ & $7.4\pm.89$ \\
         \hline
    \end{tabular}
    \centering{\sffamily\small\begin{enumerate}\item[(b)] \heteroiv \end{enumerate}}
    \end{minipage}\hfill\hspace{0.3cm}
    \linebreak
    \begin{minipage}[b]{1.0\textwidth}
    \centering
    \begin{tabular}{clcccccc}
         \hline
         \multicolumn{2}{c}{\multirow{2}{*}{Method}} &  \multicolumn{6}{c}{$n$} \\
         & & 200 & 500 & 1,000 & 2,000 & 5,000 & 10,000 \\
         \hline
         \multirow{6}{*}{K-VMM} & $\alpha_n=0$ & $>100$ & $>100$ & $>100$ & $1.6\pm1.1$ & $>100$ & $>100$ \\
         & $\alpha_n=10^{-8}$ & $>100$ & $>100$ & $>100$ & $>100$ & $>100$ & $>100$ \\
         & $\alpha_n=10^{-6}$ & $10.7\pm13.6$ & $1.8\pm1.7$ & $1.6\pm1.0$ & $1.6\pm.78$ & $1.9\pm.57$ & $2.1\pm.47$ \\
         & $\alpha_n=10^{-4}$ & $6.9\pm7.4$ & $2.1\pm1.3$ & $2.4\pm1.4$ & $2.4\pm.93$ & $2.6\pm.65$ & $2.8\pm.53$ \\
         & $\alpha_n=10^{-2}$ & $4.2\pm3.5$ & $3.9\pm1.9$ & $4.3\pm1.6$ & $4.1\pm1.0$ & $4.6\pm.74$ & $4.8\pm.71$ \\
         & $\alpha_n=1$ & $6.9\pm4.8$ & $8.2\pm2.8$ & $8.7\pm2.1$ & $8.4\pm1.8$ & $8.6\pm1.1$ & $8.6\pm.99$ \\
         \hdashline
         \multirow{4}{*}{N-VMM} & $\lambda_n=0$ & $>100$ & $53.4\pm88.6$ & $6.6\pm12.0$ & $1.1\pm.77$ & $.50\pm.33$ & $.92\pm.39$ \\
         & $\lambda_n=10^{-4}$ & $>100$ & $>100$ & $7.9\pm18.0$ & $1.1\pm.70$ & $.54\pm.46$ & $.92\pm.47$ \\
         & $\lambda_n=1$ & $>100$ & $5.7\pm8.7$ & $1.0\pm.74$ & $.93\pm.45$ & $1.7\pm.67$ & $2.0\pm.40$ \\
         \hdashline
         \multirow{3}{*}{SMD} & Identity & $>100$ & $>100$ & $43.3\pm97.6$ & $22.0\pm21.0$ & $>100$ & $>100$ \\
         & Homo & $24.5\pm17.8$ & $43.1\pm83.0$ & $>100$ & $26.4\pm48.3$ & $28.4\pm40.5$ & $32.5\pm63.9$ \\
         & Hetero & $>100$ & $>100$ & $>100$ & $>100$ & $>100$ & $>100$ \\
         \hdashline
         MMR & & $9.2\pm6.0$ & $10.7\pm4.1$ & $11.6\pm3.5$ & $12.3\pm3.3$ & $13.2\pm2.7$ & $13.4\pm2.3$ \\
         \hdashline
        OWGMM & &$>100$ &$10.8\pm38.0$ &$>100$ &$20.7\pm78.7$ &$25.9\pm87.7$ &$5.2\pm7.0$ \\
         \hdashline
         NCB & & $>100$ & $>100$ & $>100$ & $>100$ & $>100$ & $>100$ \\
         \hline
    \end{tabular}
    \centering{\sffamily\small\begin{enumerate}\item[(c)] \policylearning \end{enumerate}}
    \end{minipage}\hfill\hspace{0.3cm}
    \caption{Results of our estimation experiments. For each combination of scenario, method, and $n$, the MSE of $\hat\theta_n$ is estimated over 50 replications, along with standard errors. We write $>100$ whenever the MSE or standard error was greater than 100.}
    \label{tab:est-results}
\normalsize
\end{table}

For each scenario and each $n \in \{200, 500, 1000, 2000, 5000, 10000\}$, we repeated the following process 50 times: we drew a training set of $n$ random iid data points using the respective scenario's data generating process as well as an additional dataset of $n$ random dev data points for early stopping, hyperparameter tuning, etc., and then we estimated $\hat\theta_n$ using all of our methods and baselines using the sampled dataset. Then, for each combination of scenario and $n$ we computed the mean squared error (MSE) of the estimated $\hat\theta_n$ across these 50 replications. We summarize the results of this process in \cref{tab:est-results}.
In addition, we computed additional results based on the risk (\simpleiv and \heteroiv) or regret (\policylearning) of the estimated $g(\cdot; \hat\theta_n)$. However, these broadly followed the same trend as the main results here, so we leave them to the supplement.
In addition, we provide additional tables of results that break down the MSE in terms of bias and standard deviation in the supplement.

Overall, we can see that in all scenarios the best performing methods are our VMM methods, with the neural VMM method performing best in the \simpleiv and \policylearning scenarios, and the kernel VMM method performing best in the \heteroiv scenario. 
And, in all cases both the kernel and neural VMM methods significantly outperform the baselines. In particular, apart from the easy \simpleiv scenario, in the more complex \heteroiv and \policylearning scenarios, our VMM methods yield errors that are orders of magnitude smaller.

In terms of the values of the regularization hyperparameters, we note that kernel VMM can be sensitive to the choice of $\alpha_n$ when it takes extreme values. When $\alpha_n$ is too small, the algorithm appears to suffer from high variance and the occasional catastrophically-bad results, whereas when $\alpha_n$ is too large the estimation becomes very biased, with performance converging to that of MMR. However, for $\alpha_n$ in the range of $10^{-2}$ to $10^{-6}$, performance is good across all scenarios and $n$. It remains a question for future work how to automatically select this hyperparameter using observed data. However, we suspect that approaches based on the eigenvalues of $(Q(\tilde\theta_n) + \alpha_n L)^{-1}$ appearing in \cref{lem:kernel-vmm-closed-form} might be productive. 

Conversely, we note that our neural VMM algorithm is generally very insensitive to the choice of $\lambda_n$, with very little change in performance even for relatively large values of $\lambda_n$, and very strong and stable performance even when $\lambda_n = 0$. The one minor exception to this is in the challenging \policylearning scenario, where using the largest value of $\lambda_n$ results in somewhat better performance than other choices for low values of $n$, but worse performance for large $n$. This reinforces the notion that the neural network function class and optimization algorithms are naturally regularizing, and that explicit regularization is not necessarily important.

In general, for both VMM algorithms, we note that there is a wide range of regularization hyperparameter values where performance is generally very good. Furthermore, we note that for both cases the choices of $\Fcal$ used were very generic and the same across all scenarios; either an RKHS with a completely generic data-driven kernel, or a very generic shallow MLP. Together, this suggests that VMM can generally do very well with generic choices for all hyperparameters, and is not very sensitive to these choices as long as they do not take extreme values.

In the \simpleiv scenario, where $\e[\rho(X;\theta) \mid Z]$ is very simple and easy to fit uniformly over $\Theta$, the SMD and OWGMM baselines performed competitively with our VMM algorithms.
However, in the other more challenging scenarios, their behavior was generally inconsistent and poor.
We note that although the average squared error obtained by these methods was extremely high, this seems to be mostly dominated by some outliers, and the typical performance was much more reasonable.
For example, in the \heteroiv scenario when $n=10,000$, the median squared error of the Identity, Homoskedastic, and Heteroskedastic versions of SMD were $48.3$, $10.7$, and $0.22$ respectively, which is much less bad than the average squared error.
This is also evident, for example, from the separate bias and standard deviation results in the supplement.
We also note that for both SMD and OWGMM algorithms, we experimented with a wide range of choices for the underlying sieve basis sets, including the number of knots and polynomial degree for the B-splines that we used, as well as ridge-regularization values, and the results presented are for the least-bad choices.
We speculate that the superior performance of our approach is due to the kernel-based regularization of the critic class, which in practice is better able to approximate the efficient instruments with good accuracy and stability.
Indeed, it is plausible that sieve-based approaches could also achieve competitive performance using better choices of sieves, with appropriate regularization. In general, however, the use of such sieve spaces, rather than simple linear sieves with optional ridge regularization, as we experimented with, is either intractable or redundant. 
In the case of SMD, the corresponding sieve estimates $q_n(z;\theta)$ for $\EE[\rho(X;\theta) \mid Z=z]$ would no longer have closed-form solutions in $\theta$ in general, and we would somehow have to solve a bi-level optimization problem.
On the other hand, if we were to introduce such regularization to the sieve space that implicitly arises from the variational reformulation of OWGMM, we would just end up with our VMM approach as in \cref{eq}.

On the other hand, we see that the MMR baseline performed in a way that was relatively very stable, but consistently sub-optimal. The results of MMR were in general similar to kernel VMM with the largest choices for $\alpha_n$, which is expected given that it is equivalent to kernel VMM with $\alpha_n \to \infty$. In addition, as expected, the non-causal baseline is consistently biased with very poor performance.

Finally, we provide a break down of these mean squared error results in terms of bias and variance in the supplement. One interesting observation there is that some methods, in particular our neural VMM algorithm and the OWGMM baseline, do not display the expected behavior of bias vanishing at a more rapid rate than standard deviation; rather, even though both shrink, their ratio often remains approximately constant. This could be explained by a couple of factors. First, in the case of neural VMM we are not exactly optimizing the minimax optimization problem; rather, we are trying to approximate this using an alternating gradient ascent/descent approach. Therefore, such discrepancies may be explained by this deviation from theory in the practical implementation of the algorithm. Note that issue doesn't exist for kernel VMM, which performs the optimization over $\Fcal_n$ analytically. Second, in the case of OWGMM, this discrepancy seems to be explained by the instability and poor performance described above. This could be be interpreted as ``finite sample'' behavior, reflecting the fact that we are not yet in the asymptotic regime for this method. Alternatively, it may reflect intractable bias due to approximation errors of the sieve basis for the efficient instruments.

\subsection{Inference Experiments}

\subsubsection{Inference Scenarios}

\paragraph{\simpleiv} Our first considered scenario for our inference experiments is based on the same \simpleiv scenario as in our estimation experiments. For this scenario here, our target for inference is the instantaneous treatment effect at $T=0$; that is, we wish to estimate $\psi(\theta_0)$ where $\psi(\theta)=\frac{\partial}{\partial t} g(t;\theta)\bigr|_{t=0}=\theta_{2}$.

\paragraph{\heteroiv} For our second inference scenario we consider again the same \heteroiv scenario from our prior estimation experiments. Here, our target for inference is the change in slope in the true hinge function $g(\cdot;\theta_0)$. This corresponds to the function $\psi(\theta) = \theta_4 - \theta_3$.

\subsubsection{Inference Methods}

\paragraph{Kernel Inference} For our kernel inference method, we implemented the algorithm described by \cref{thm:kernel-inference}. We used the same kernel function as for our Kernel VMM estimation algorithm in our prior estimation experiments, and we present results for a variety of value of $\alpha_n$.

\paragraph{Neural Inference} For our neural inference method, we solved the game objective described by \cref{eq:neural-inference}. We used the same choice of $\mathcal F_n$ and a similar alternating SGD optimization procedure as for our NeuralVMM estimation method. We provide additional details in the supplement.
As in our estimation experiments, we used
Frobenius norm regularization, and we present results for varying values of $\lambda_n$.

\subsubsection{Inference Results}

\begin{table}\centering\small
    \begin{minipage}[b]{1.0\textwidth}
    \centering
        \begin{tabular}{cclccccc}
         \hline
         $n$ & \multicolumn{2}{c}{Method} & Cov & CovBC & PredSD(.05) & PredSD(.5) & PredSD(.95) \\
         \hline
         \multirow{10}{*}{$200$} & \multirow{6}{*}{Kernel} & $\alpha_n=0$ & 83.0 & 94.0 & .21 & .32 & .52 \\
         & & $\alpha_n=10^{-8}$ & 83.0 & 94.0 & .21 & .32 & .51 \\
         & & $\alpha_n=10^{-6}$ & 83.0 & 94.5 & .21 & .32 & .53 \\
         & & $\alpha_n=10^{-4}$ & 84.5 & 95.5 & .22 & .33 & .56 \\
         & & $\alpha_n=10^{-2}$ & 86.5 & 95.5 & .23 & .35 & .62 \\
         & & $\alpha_n=1$ & 91.0 & 96.0 & .25 & .39 & .72 \\
         \cdashline{2-8}
         & \multirow{4}{*}{Neural} & $\lambda_n=0$ & 82.0 & 94.0 & .21 & .31 & .48 \\
         & & $\lambda_n=10^{-4}$ & 81.5 & 94.0 & .21 & .31 & .49 \\
         & & $\lambda_n=1$ & 82.5 & 93.5 & .21 & .30 & .49 \\
         \hdashline
         \multirow{10}{*}{$2000$} & \multirow{6}{*}{Kernel} & $\alpha_n=0$ & 91.5 & 93.5 & .19 & .21 & .24 \\
         & & $\alpha_n=10^{-8}$ & 92.0 & 94.0 & .19 & .22 & .25 \\
         & & $\alpha_n=10^{-6}$ & 92.5 & 94.0 & .20 & .22 & .24 \\
         & & $\alpha_n=10^{-4}$ & 92.5 & 94.5 & .20 & .22 & .25 \\
         & & $\alpha_n=10^{-2}$ & 95.0 & 96.0 & .21 & .23 & .28 \\
         & & $\alpha_n=1$ & 100.0 & 100.0 & .48 & .55 & .87 \\
         \cdashline{2-8}
         & \multirow{4}{*}{Neural} & $\lambda_n=0$ & 90.0 & 92.5 & .19 & .21 & .23 \\
         & & $\lambda_n=10^{-4}$ & 90.5 & 92.5 & .19 & .21 & .22 \\
         & & $\lambda_n=1$ & 90.0 & 92.5 & .19 & .21 & .22 \\
         \hline
        \end{tabular}
    \centering{\sffamily\normalsize\begin{enumerate}\item[(a)] \simpleiv; the true standard deviation over the 200 replications was $0.34$ for $n=200$ and $0.23$ for $n=2000$. \end{enumerate}}
    \end{minipage}\hfill\hspace{0.3cm}
    \linebreak
    \begin{minipage}[b]{1.0\textwidth}
    \centering
        \begin{tabular}{cclccccc}
             \hline
             $n$ & \multicolumn{2}{c}{Method} & Cov & CovBC & PredSD(.05) & PredSD(.5) & PredSD(.95) \\
             \hline
             \multirow{10}{*}{$200$} & \multirow{6}{*}{Kernel} & $\alpha_n=0$ & 84.5 & 85.0 & .35 & .54 & 1.9 \\
             & & $\alpha_n=10^{-8}$ & 83.5 & 83.5 & .37 & .55 & 2.1 \\
             & & $\alpha_n=10^{-6}$ & 87.5 & 88.5 & .42 & .58 & 2.3 \\
             & & $\alpha_n=10^{-4}$ & 91.5 & 92.5 & .49 & .66 & 2.8 \\
             & & $\alpha_n=10^{-2}$ & 95.0 & 98.0 & .59 & .87 & 4.6 \\
             & & $\alpha_n=1$ & 100.0 & 100.0 & 1.4 & 2.5 & 13.3 \\
             \cdashline{2-8}
             & \multirow{4}{*}{Neural} & $\lambda_n=0$ & 70.5 & 65.5 & .24 & .40 & .84 \\
             & & $\lambda_n=10^{-4}$ & 71.5 & 68.0 & .25 & .43 & .84 \\
             & & $\lambda_n=1$ & 70.0 & 66.0 & .25 & .42 & .84 \\
             \hdashline
             \multirow{10}{*}{$2000$} & \multirow{6}{*}{Kernel} & $\alpha_n=0$ & 95.5 & 97.5 & .20 & .21 & .24 \\
             & & $\alpha_n=10^{-8}$ & 95.5 & 97.5 & .19 & .21 & .24 \\
             & & $\alpha_n=10^{-6}$ & 95.5 & 97.5 & .20 & .21 & .24 \\
             & & $\alpha_n=10^{-4}$ & 96.0 & 97.5 & .20 & .22 & .25 \\
             & & $\alpha_n=10^{-2}$ & 97.5 & 98.5 & .21 & .23 & .27 \\
             & & $\alpha_n=1$ & 100.0 & 100.0 & .47 & .55 & .88 \\
             \cdashline{2-8}
             & \multirow{4}{*}{Neural} & $\lambda_n=0$ & 95.0 & 95.5 & .19 & .21 & .22 \\
             & & $\lambda_n=10^{-4}$ & 94.5 & 95.5 & .19 & .21 & .22 \\
             & & $\lambda_n=1$ & 94.5 & 95.5 & .20 & .21 & .22 \\
             \hline
        \end{tabular}
    \centering{\sffamily\normalsize\begin{enumerate}\item[(b)] \heteroiv; the true standard deviation over the 200 replications was $1.6$ for $n=200$ and $0.21$ for $n=2000$. \end{enumerate}}
    \end{minipage}\hfill\hspace{0.3cm}
    \caption{Results of our inference experiments, using kernel VMM estimation with $\alpha_n = 10^{-4}$, and various VMM inference methods. For each inference method and value of $n$, we list: \textbf{Cov} the coverage of the respective 95\% confidence intervals; \textbf{CovBC} the corresponding bias-corrected coverage, by subtracting the bias of $\psi(\hat\theta_n)$ from the confidence intervals; and \textbf{PredSD($\bm{q}$)} the $q$'th percentile of the estimated standard deviation of $\psi(\hat\theta_n)$, for $q \in \{5,50,95\}$.}
    \label{tab:inf-results}
\end{table}

For each scenario and each $n \in \{200, 2000\}$, we repeated the following procedure 200 times: (1) we drew a training set of $n$ random iid data points using the respective scenario's data generating process; (2) we estimated $\hat\theta_n$ using each of our VMM methods; and (3) we estimate the efficient asymptotic variance using each of our inference methods and each of the estimated $\hat\theta_n$ as plug-ins. That is, for each random draw of data, we estimate the efficient asymptotic variance using each combination of VMM estimation method and inference method. In all cases, we compute an estimated 95\% confidence interval as
where $\hat{v}$ is the estimated asymptotic variance of $\psi(\hat\theta_n)$ via the delta method, which we computed using the corresponding inference method as detailed in \cref{sec:inference}.
In addition, for each combination of $n$, scenario, estimation method, and inference method, we computed the following summary statistics: (1) the coverage rate of our estimated confidence intervals; (2) the corresponding coverage when we adjust the confidence intervals by subtracting the bias of $\psi(\hat\theta_n)$ (which we estimated by $\frac{1}{200} \sum_{i=1}^{200} \psi(\hat\theta_n^{(i)})) - \psi(\theta_0)$, where $\hat\theta_n^{(i)}$ denotes the estimate from the $i$'th replication); and (3) the 5\%, 50\%, and 95\% percentiles of the estimated standard deviation of $\psi(\hat\theta_n)$ (given by $\sqrt{\hat{v}/n}$) across the 200 replications.

Given our previous results that kernel VMM performed very consistently with $\alpha_n$ in the range of $10^{-2}$ to $10^{-6}$, for brevity we only present results here using the kernel VMM estimation method with $\alpha_n = 10^{-4}$. However, we present additional results using other estimation methods in the appendix. We summarize the results from this procedure in \cref{tab:inf-results}.

Overall, we see that in both scenarios the results are very good when $n=2000$, with very accurate estimates of the standard deviation of $\psi(\hat\theta_n)$, and high coverage. For the \heteroiv scenario, all inference methods produce almost perfect (95\%) coverage when $n=2000$, and for the \simpleiv scenario the coverage is only slightly lower, and becomes very close to 95\% when bias of $\hat\theta_n$ is taken into account.

When $n=200$, our inference results are slightly poorer. This likely reflects several distinct issues when $n$ is small: the bias of $\hat\theta_n$ may be significant, the variance may not be well be characterized by the asymptotic variance and the tails by normal tails, and the estimates of the asymptotic variance of $\psi(\hat\theta_n)$ may be poor.
Any of these issues may lead to invalid confidence intervals and lower than expected coverage.
Indeed, we can see some or all of these issues at play in our results. In \cref{tab:inf-results}(a) we see that coverage is very good when we account for bias, and that the range of the predicted standard deviation of $\hat\theta_n$ is reasonably close to the empirically observed standard deviation of $0.34$, which suggests we are suffering from the first issue. Conversely, in \cref{tab:inf-results}(b) we see that, even accounting for bias, the coverage is lower than expected when $n=200$, and that the range of predicted standard deviations of $\psi(\hat\theta_n)$ is low compared to the empirically observed standard deviation of $1.9$, which suggests we are suffering from the second and/or third issues.

Regarding the difference in performances between our inference methods, we observe that, as expected given \cref{thm:kernel-inference}, larger values of $\alpha_n$ for our kernel method lead to wider confidence intervals. For $n=2000$, where our asymptotic theory seems to be more relevant, we see very overly-wide confidence intervals for our kernel method when $\alpha_n$ is very large, with typically good results when $\alpha_n$ takes the same range of values that worked well for estimation in our prior experiments (\emph{i.e.}, in the range of $10^{-6}$ to $10^{-2}$). This suggests that we can tune $\alpha_n$ for estimation, and use similar values for inference, and also that we can err on the side of caution and wider confidence intervals by using larger values of $\alpha_n$. Conversely, we found our neural inference method to be very insensitive to $\lambda_n$, and in general we found that it produced relatively narrow confidence intervals, with widths similar to those from our kernel method using the smallest values of $\alpha_n$.

Finally, we make a note to emphasize the fact that biased-corrected coverage values are listed merely so we can analyze, in cases where coverage is poor, to what extent this is due to bias in the estimate $\hat\theta_n$, versus due to poor estimates of the standard deviation of $\hat\theta_n$. Indeed, the bias-correction we perform is \emph{not} something that can be done in practice, and these bias-corrected coverages should not be interpreted as actual coverages that can be obtained.

\section{Related Work}
\label{sec:related-work}

\subsection{Methods for Solving Conditional Moment Problems}
\label{sec:related-work-conditional-moment}

For the general conditional moment problem, one classical approach is to solve \cref{eq:owgmm} using a growing sieve basis expansion for $\{f_1, \ldots, f_k\}$ based on, \emph{e.g.}, splines, Fourier series, or power series \citep{chamberlain1987asymptotic}. It would be expected, however, that such methods would suffer from curse of dimension issues and therefore their application would be limited to low-dimensional settings. Furthermore, it has been observed in past work \citep{bennett2019deep,bennett2020efficient} that methods of this kind can be very unstable, and perform very poorly in comparison to VMM estimators.

A very similar method to this is the sieve minimum distance approach of \citet{ai2003efficient}, which instead uses a growing sieve basis expansion to approximate the conditional function $\e[\rho(X;\theta) \mid Z=z]$ for every $\theta \in \Theta$. They propose to minimize a loss of the form $J_n(\theta) = \e_n[\hat q(Z;\theta)\tr \Gamma(Z)^{-} \hat q(Z;\theta)]$, where $\hat q(z;\theta)$ is the sieve estimate for $\e[\rho(X;\theta) \mid Z=z]$ and $\Gamma(z)$ is some consistent estimate of $\e[\rho(X;\theta_0) \rho(X;\theta_0)\tr \mid Z=z]$. One nice feature of this kind of approach is that it can readily handle infinite-dimensional nuisance components. In the case that $\theta$ can be partitioned as $\theta = (\beta, \gamma)$, where $\beta$ is a finite-dimensional parameter of interest and $\gamma$ is an infinite-dimensional functional nuisance component, \citet{ai2003efficient} propose to model $\gamma$ using a second growing sieve basis expansion, and minimize $J_n(\theta)$ over both $\beta$ and the sieve coefficients for $\gamma$. 
There is a long line of work on the theoretical efficiency of this kind of approach, even in the presence of infinite-dimensional nuisance components \citep{ai2003efficient,chen2009efficient,chen2012estimation}, which is something that our theory does not address. However, these methods have similar practical drawbacks to using a sieve basis expansion for OWGMM, which seems to particularly be the case when the conditional expectation function $q(z;\theta) = \e[\rho(X;\theta) \mid Z=z]$ is complex, as highlighted by the experimental results in \cref{sec:experiments}. A very similar approach was also proposed concurrently by \citet{newey2003instrumental}, however their approach has the same drawbacks, and furthermore they do not address efficiency.

Another related classical approach is to solve \cref{eq:owgmm} using estimates of the \emph{efficient instruments}, which are the set of $b$ functions $\{f^*_1,\ldots,f^*_b\}$ mapping $\mathcal{Z}$ to $\RR^b$, given by $f_i^*(z)_j = F^*(z)_{i,j}$, where
\begin{equation*}
    F^*(z) = V(z;\theta_0)^{-1} \e[D(X;\theta_0) \mid Z=z] \,.
\end{equation*}
Past work such as \citet{newey1990efficient,newey1993efficient} provide sufficient conditions for such estimators to be efficient.
However, since $\theta_0$ is unknown, such methods require some other method for first-stage estimation of $\theta_0$, and are likely sensitive to the quality of this method; indeed, if the estimates of $f^*_i$ are heavily biased due to poor first-stage estimation, it is unclear whether the corresponding moments will be sufficient for identification, let alone efficiency. By contrast, our method is guaranteed to be well behaved as long as our regularized critic class $\Fcal_n$ can approximate the optimal instruments, regardless of the quality of our first-stage estimate. Furthermore, estimators that have been previously proposed based on this approach \citep{newey1990efficient,newey1993efficient} employ nearest neighbor or sieve methods with similar weaknesses as discussed above.

The \emph{continuum GMM} estimators of \citet{carrasco2000generalization} are theoretically closely related to our proposed kernel VMM estimators. However, the form of their proposed estimators is very different. Suppose that we define some set of functions $\{f(\cdot; t) : t \in T\}$ of the form $\mathcal Z \mapsto \mathbb R^m$ indexed by set $T$, and we let $\mathcal H_T$ be some Hilbert space of functions in the form $T \mapsto \mathbb R$. In addition, define $h'_n(\theta) \in \mathcal H$ according to $h'_n(\theta)(t) = \e_n[f(Z;t)\tr  \rho(X; \theta)]$, and the linear operator $C'_n : \mathcal H_T \mapsto \mathcal H_T$ according to
\begin{equation*}
    (C'_n h)(t) = \inner{k_n(t,\cdot)}{h}_{\mathcal H_T}\,,
\end{equation*}
where
\begin{equation*}
    k_n(t,s) = \e_n[f(Z;s)\tr  \rho(X;\tilde\theta_n) \rho(X;\tilde\theta_n)\tr  f(Z;t)]\,,
\end{equation*}
and $\tilde\theta_n$ is some prior estimate for $\theta_0$. Then, \citet{carrasco2000generalization} study estimators of the form $\argmin_{\theta \in \Theta} \norm{((C'_n)^2 + \alpha_n I)^{-1/2} (C'_n)^{1/2} h'_n(\theta)}_{\mathcal H_T}^2$. In the case that we choose $T$ to be an RKHS class $\mathcal F$, with functions indexed by themselves, and $\mathcal H$ chosen as the dual of this RKHS, then it easily follows that the terms $C'_n$ and $h'_n$ defined here are equivalent to the terms $C_n$ and $h_n$ defined in \cref{sec:kernel-vmm}. However, the form of Tikhonov regularization applied in the inversion of $C_n^{1/2}$ is slightly different; by \cref{lem:kvmm-cgmm-equiv} we regularize using $(C_n + \alpha_n I)^{-1/2}$, whereas they regularized using $(C_n^2 + \alpha_n I)^{-1/2} C_n^{1/2}$. This difference is significant, since our form of regularization gives rise to the simple minimax VMM-style interpretation, whereas theirs does not. Furthermore, their proposed estimators use the index set $T = [0,t_{\text{max}}]$ for some $t_{\text{max}}>0$, with $\mathcal H_T$ chosen as the $L_2$ space on $T$. This choice is much less flexible than ours of using a function class as the index set and makes it more difficult to guarantee that $\theta_0$ is uniquely identified or to guarantee semiparametric efficiency, which they do not. More concretely, the main efficiency claim they provide is that their estimator is efficient compared to other estimators of the form $\sup_{\theta \in \Theta} \norm{B'_n h_n(\theta)}^2$, for any choice of bounded linear operator $B'_n$. Finally, they propose to solve their optimization problem by computing an explicit rank-$n$ eigenvalue, eigenvector decomposition of $C'_n$, and constructing a cost function to minimize based on this decomposition. In particular, if we define $g_i(\theta) \in \mathcal H_T$ according to $g_i(\theta)(t) = f(Z_i;t)\tr  \rho(X_i;\theta)$ for each $i \in [n]$ and $\theta \in \Theta$, then their objective function is given by a quadratic form on all terms of the form $\inner{g_i(\theta)}{g_j(\theta)}_{\mathcal H_T}$ for $i,j \in [n]$. This involves $n^4$ terms in total, and is therefore very computationally expensive to compute for large $n$. In comparison, the cost function in $\theta$ implied by our kernel VMM estimator could be calculated analytically as a quadratic form in $n^2$ terms based on the representer theorem. Furthermore, our variational reformulation allows for estimators based on alternating stochastic gradient descent, which may be more practical in some situations, for example when $n$ is large.

Another recently proposed and related class of estimators are given by the \emph{adversarial GMM} estimators of \citet{lewis2018adversarial}, which were recently extended to the more general class of \emph{minimax GMM} estimators by \citet{dikkala2020minimax}. In general, these estimators are defined according to $\argmin_{\theta \in \Theta} \sup_{f \in \mathcal{F}} \e_n[f(Z)\tr  \rho(X;\theta)] + R_n(f) - \Psi_n(\theta)$, where $R_n$ is some regularizer on $f$, and $\Psi_n$ is some regularizer on $\theta$. In particular, \citet{dikkala2020minimax} analyze estimators where $\mathcal F$ and $\Theta$ are both normed function spaces, and the regularizes take the form $R_n(f) = \alpha_n \norm{f}^2_{\mathcal F} + \lambda_n \sum_{i=1}^n \norm{f(Z_i)}^2$ and $\Psi_n(\theta) = \mu_n \norm{\theta}^2_{\Theta}$. On the theoretical side, they provide general results bounding the the $L_2$ distance between $\e[\rho(X;\hat\theta_n) \mid Z]$ and $\e[\rho(X;\theta_0) \mid Z]$ for this form of estimator.
Furthermore, they propose various specific estimators of this kind, for example with $\mathcal F$ chosen as a RKHS or a class of neural networks. We note that these are similar to our proposed kernel and neural VMM estimators, with the difference that they do not include the $-(1/4) \e_n[(f(Z)\tr  \rho(X;\tilde\theta_n))^2]$ term motivated by optimal weighting, and that they explicitly regularize $\theta$. In a sense, the focus of their estimators and theory is very different than ours; we focus on the question of efficiency, and provide theoretical guarantees of efficiency when $\Theta$ is finite-dimensional, whereas they focus on the case where $\Theta$ is a function space, but restrict their analysis to providing finite-sample bounds rather than addressing efficiency.
We speculate that the benefits of both kinds of approaches could be combined, and by using both the optimal weighting-based term and regularizing $\theta$ one could construct estimators that are semiparametrically efficient when $\theta_0$ is finite-dimensional, and have explicit risk guarantees in the more general setting. However, we leave this question to future work.

\subsection{Methods for Solving the Instrumental Variable Regression Problem}

Recall that for the instrumental variable regression problem we have $X = (Z,T, Y)$, where $T$ is the treatment we are regressing on, $Y$ is the outcome, and $Z$ is the instrumental variable, and  $\rho(X;\theta) = Y - g(T; \theta)$, for some regression function $g$ parameterized by $\theta$.
In this setup, $\Theta$ may either be a finite-dimensional parameter space, which corresponds to having a parametric model for $g$, or alternatively we may allow allow $\Theta$ to be some infinite-dimensional function space and simply define $g(z;\theta) = \theta(z)$, which corresponds to performing nonparametric regression.

Perhaps the most classic method for instrumental variable regression is two-stage least squares (2SLS). 
First, we perform least-squares linear regression of $\phi(T)$ on $\psi(Z)$, where $\phi$ and $\psi$ are finite-dimensional feature maps on $T$ and $Z$ respectively. That is, we learn some linear model $h(\cdot; \hat\gamma_n)$, where $h(z;\gamma) = \gamma\tr  \psi(z)$, and $\hat\gamma_n = \argmin_{\gamma} \sum_{i=1}^n \norm{\phi(T_i) - \gamma\tr  \psi(Z_i)}^2$. Then, we again perform least squares linear regression, this time of $Y$ on $h(\psi(Z);\hat\gamma_n)$. That is, we learn a linear model $g(\cdot;\hat\theta_n)$, where $g(t;\theta) = \theta\tr  \phi(t)$, and $\hat\theta_n = \argmin_{\theta} \sum_{i=1}^n (Y_i - \theta\tr  h(Z_i;\hat\gamma_n))^2$. Under the assumption that these linear models are correctly specified, then the resulting 2SLS estimator is known to be consistent for $\theta_0$ \citep[\S4.1.1]{angrist2008mostly}. However, such estimators are limited in that they require finding some finite-dimensional feature map $\phi$ such that the linear model given above is well-specified, which in practice may be infeasible. 
The sieve methods of \citet{newey2003instrumental,ai2003efficient} discussed in \cref{sec:related-work-conditional-moment} applied specifically to the instrumental variable regression problem could be viewed as similar approaches, but using growing sieve basis expansions for $\phi$ and $\psi$. However, as discussed already these methods may be problematic in practice.

Alternatively, a couple of recent works propose extending the 2SLS method in the case where both stages are performed using infinite-dimensional feature maps and ridge regularization; \emph{i.e.}, both stages are performed using kernel ridge regression. The \emph{Kernel IV} method of \citet{singh2019kernel} proposes to do this in a very direct way, by regressing $\phi(T)$ on $\psi(Z)$, and then regressing $Y$ on $h(\phi(Z))$, where both the feature maps $\phi$ and $\psi$ are infinite dimensional, and implicitly defined by some kernels $K_Z$ and $K_T$ under Mercer's theorem. In the case of learning $h$, this corresponds to solving for a linear operator between two RKHSs and in general is ill-posed, so this regression is performed using Tikhonov regularization. Then, the second-stage problem corresponds to performing RKHS regression using some implicit kernel depending on $h$, and is performed again using Tikhonov regularization.
Ultimately, however, by appealing to the representer theorem the regressions don't need to be performed separately, and there is a simple closed form solution. 
Similarly, the \emph{Dual IV} method of \citet{muandet2020dual} considers 2SLS using RKHSs for each stage and formulates this as a minimax problem of the form $\argmin_{\theta \in \Theta} \sup_{f \in \mathcal F} \e_n[(g(T;\theta) - Y) f(Y,Z)] - (1/2) \e_n[f(Y,Z)^2]$.
Ultimately, both this paper and the work of \citet{singh2019kernel} propose closed-form estimators that are superficially similar to ours in \cref{lem:kernel-vmm-closed-form-iv}, but without any terms corresponding to optimal weighting.
However, their focus is slightly different to ours; their theoretical analysis where present is in terms of consistency or regret, whereas the focus of our theoretical analysis is semiparametric efficiency.

The recent \emph{Deep IV} method of \citet{hartford2017deep} proposes to extend 2SLS using deep learning. Specifically, they propose in the first stage to fit the conditional distribution of $X$ given $Z$, for example using a mixture of Gaussians parametrized by neural networks, or by fitting a generative model using some other methodology such as generative adversarial networks or variational autoencoders. Then, in the second stage, they propose to minimize $(1/n) \sum_{i=1}^n (Y_i - \hat\e[g(X; \theta) \mid Z_i])^2$, where the conditional expectation $\hat\e[\cdot \mid z]$ is estimated using the model from the first stage, and $g$ is parameterized using some neural network architecture. This approach has the advantage of being flexible and building on recent advances in deep learning, however they do not provide any concrete theoretical characterizations, and since the first stage is bound to be imperfectly specified this can suffer from the ``forbidden regression'' issue \citep[\S4.6.1]{angrist2008mostly}.

\citet{zhang2020maximum} recently proposed the \emph{maximum moment restriction instrumental variable} algorithm. They present multiple estimators for approximately solving $\argmin_{\theta \in \Theta} \sup_{f : \norm{f}_{\mathcal F} \leq 1} \e[f(Z) (W - g(T;\theta))] - \Psi_n(\theta)$, where $\mathcal F$ is an RKHS, and $\Psi_n$ is an optional regularizer on $\theta$ in the case that it is infinite-dimensional (however, they also analyze case where $\theta$ is finite-dimensional.) Of particular note, the ``V-statistic'' version of their algorithm is equivalent to minimizing $J^{\text{MMR}}_n(\theta) = (1/n^2) \rho(\theta)\tr L \rho(\theta) + \Psi_n(\theta)$, where $\rho(\theta)$ and $L$ are defined as in \cref{lem:kernel-vmm-closed-form}. Letting $J_n^{\text{K-VMM}}(\theta;\alpha)$ denote our kernel VMM objective with regularization strength $\alpha$, and assuming $\Psi_n(\theta) = 0$, \cref{lem:kernel-vmm-closed-form} immediately implies that $\alpha J_n^{\text{K-VMM}}(\theta;\alpha) \to J_n^{\text{MMR}}(\theta)$ as $\alpha \to \infty$. In other words, there is an equivalence between MMR and kernel VMM with infinite regularization. \citet{zhang2020maximum} provide theory showing that their estimators are consistent and asymptotically normal under various assumptions. However, unlike us, they do \emph{not} establish efficiency.

\subsection{Applications of VMM Estimators}
\label{sec:related-work-ours}

Finally, we discuss some past work where VMM estimators have been applied. The original such work was by \citet{bennett2019deep}, who proposed the \emph{DeepGMM} estimator for the problem of instrumental variable regression. Specifically, the proposed estimator takes the form $\argmin_{\theta} \sup_{f \in \mathcal F} \e_n[f(Z)\tr  (Y - g(T;\theta))] - (1/4) \e_n[f(Z)^2 (Y - g(T;\tilde\theta_n))^2]$, where $\{g(\cdot; \theta) : \theta \in \Theta\}$ and $\mathcal F$ are both given by neural network function classes. That is, the DeepGMM estimator can be interpreted as a neural VMM estimator for the instrumental variable problem in the form of \cref{eq:neural-vmm} with $R_n(f) = 0$ and fixed $\mathcal F_n$ that does not grow with $n$. In their experiments DeepGMM consistently outperformed other recently proposed methods \citep{hartford2017deep,lewis2018adversarial} across a variety of simple low-dimensional scenarios, and it was the only method to continue working when using high-dimensional data where the treatments and instruments were images. In addition, DeepGMM has continued to perform competitively in more recent experimental comparisons \citep{singh2019kernel,muandet2020dual}. \citet[theorem 2]{bennett2019deep} provided conditions under which DeepGM is consistent. In addition, we could also justify that it is asymptotically normal and semiparametrically efficient by \cref{thm:neural-vmm-properties}, under some additional assumptions and by introducing kernel-based regularization.

In addition, this style of estimator was applied to the problem of policy learning from convex surrogate loss reductions by \citet{bennett2020efficient}. A common approach for optimizing binary treatment decision policies from logged cross-sectional data is to construct a surrogate cost function to minimize of the form $\e_n[\abs{\psi} l(g(X; \theta), \text{sign}(\psi))]$, where $X$ denotes observed pre-treatment information about the individual, $\psi$ is some weighting variable depending on all observed pre- and post-treatment information about the individual, the function $g(\cdot; \theta)$ encodes the policy we are optimizing which we assume is parameterized by $\theta \in \Theta$, and $l$ is some smooth convex loss function such as logistic regression loss. \citet{bennett2020efficient} showed that the model where this surrogate loss is correctly specified is given by the conditional moment problem $\e[\abs{\psi} l'(g(X; \theta), \text{sign}(\psi)) \mid X] = 0$, where $l'$ is the derivative of $l$ with respect to its first argument. Consequently, they proposed the empirical surrogate loss policy risk minimization (ESPRM) estimator, according to $\argmin_{\theta \in \Theta} \sup_{f \in \mathcal F} \e_n[f(X) \rho(X,\psi;\theta)] - (1/4) \e_n[f(X)^2 \rho(X,\psi;\tilde\theta_n)]$, where $\rho(X,\psi;\theta) = \abs{\psi} l'(g(X;\theta), \text{sign}(\psi))$, and $\mathcal F$ is a neural network function classes. That is, again this estimator can be interpreted as a neural VMM estimator as in \cref{eq:neural-vmm}, with $R_n(f) = 0$. Not only did the authors demonstrate that this algorithm led to consistently improved empirical performance over the standard approach of empirical risk minimization using the surrogate loss, but they proved that if the resulting estimator $\hat\theta_n$ is semiparametrically efficient, then this implies optimal asymptotic regret for the learnt policy compared with \emph{any} policy identified by the model given by correct specification. We note that, although the authors did not address the question of how to guarantee such efficiency for $\hat\theta_n$, we could guarantee it by \cref{thm:neural-vmm-properties} under some additional assumptions and kernel-based regularization, or under \cref{thm:efficiency} by instead using a kernel VMM estimator.

Finally, \citet{bennett2021off} applied this style of estimator to the problem of reinforcement learning using offline data logged from some fixed behavior policy, also known as the problem of off policy evaluation (OPE). They proposed an algorithm for the OPE problem under unmeasured confounding, which requires as an input an estimate of the state density ratio $d$ between the behavior policy and the target policy they are evaluating. As stated in \cref{sec:intro}, $d$ can be identified by a conditional moment problem, up to a constant factor, with the normalization constraint $\e[d(S)] = 1$. \citet{bennett2021off} proposed a VMM-style estimator for $d$, using both the conditional moment condition $\e[d(S) \beta(A,S) - d(S') \mid S'] = 0$ and the marginal moment condition $\e[d(S) - 1] = 0$, based on a slightly more general form of \cref{lem:owgmm} where the vector of conditional moment restrictions can depend on different random variables to be conditioned on. That is, the more general problem is given by the $m$ moment conditions $\e[\rho_i(X;\theta_0) \mid Z_i] = 0$ for $i \in [m]$, for some set of random variables $Z_1 \in \mathcal Z_1, \ldots, Z_m \in \mathcal Z_m$. Specifically, they propose a kernel VMM-style estimator, where both $d$ and $f$ are optimized over balls in RKHSs. 
In practice, by successively applying the representer theorem to this two-stage optimization problem, they presented a closed-form solution for the estimate $\hat d_n$ (in a similar vein to \cref{lem:kernel-vmm-closed-form-iv}).
Note that since their kernel VMM estimator is based on a slightly more intricate conditional moment formulation than we considered, with varying conditioning sets, our theoretical analysis may not apply to it. We leave the question of extending our theoretical analysis to this more general problem to future work.

\section{Conclusion}

In this paper we presented a detailed theoretical analysis for the class of variational method of moments (VMM) estimators, which are motivated by a variational reformulation of the optimally weighted generalized method of moments and which encompass several recently proposed estimators for solving conditional moment problems. We studied multiple varieties of these estimators based on kernel methods or deep learning, and provided appropriate conditions under which these estimators are consistent, asymptotically normal, and semiparametrically efficient. This is in contrast to other recently proposed approaches for solving conditional moment problems using machine learning tools, which do not provide any results regarding efficiency. In addition, we proposed inference algorithms based on the same kind of variational reformulation, again with specific algorithms based on both kernel methods and deep learning. Finally, we demonstrated in a detailed series of experiments that our VMM estimators achieve very strong estimation performance in comparison to relevant baselines and that the confidence intervals we generate are reliable.

Our paper suggests a few immediate directions for future work. First, unlike, \emph{e.g.}, the sieve minimum distance approaches of \citep{ai2003efficient,chen2009efficient,chen2012estimation}, our efficiency theory when $\theta_0$ is finite-dimensional does not accommodate possible infinite-dimensional nuisance components. Furthermore, as discussed in \cref{sec:kernel-vmm}, the latter two works allow for weaker assumptions on the smoothness and complexity of $\rho(X;\theta)$.  We suspect that our theory could be extend accordingly without fundamentally changing the VMM algorithm, but this is left to future work. 

Second, we only consider conditional moment restrictions using a single conditioning variable $Z$. In some settings, such as longitudinal studies or the RL application discussed in \cref{sec:related-work-ours}, one faces conditional moment problems with different, nested conditioning variables for each conditional moment restriction, and our current theory does not accommodate such formulations.
Again, we believe that our theory could naturally be extended to this kind of setting.

Third, we only present theory for neural VMM estimators using a kernel-based regularizer, yet we see compelling empirical results for simpler regularizers. We speculate that under appropriate conditions on the neural net classes $\mathcal F_n$, our efficiency result in \cref{thm:neural-vmm-properties} could be extended to neural VMM estimators with such regularizers.

Next, an important further direction is the automatic selection of the hyperparameter $\alpha_n$ for our kernel VMM method and corresponding inference algorithm. We speculate, for instance, that it may be possible to approximate the resulting bias and variance for different values of $\alpha_n$ and optimize a bias-variance trade-off. At the same time, work on approximating the bias of our estimator could be helpful for improving the quality of confidence intervals from our proposed inference algorithm, as we observed that in many cases coverage of our confidence intervals significantly improved when they were corrected for bias. Similarly, it is known that continuously updating GMM can have lower bias than $k$-step GMM algorithms \citep{hansen1996finite}, which suggsts we may be able to reduce bias using a continuously updating VMM where instead of using a prior estimate $\tilde\theta_n$ in the second term of the game objective we use the same $\theta$ that we are optimizing over.

Finally, we hope that this work will help motivate the construction of efficient VMM estimators for other conditional moment problems.

\section*{Acknowledgements}

This material is based upon work supported by the National Science Foundation under Grant No. 1846210.

\bibliography{ref}

\newpage
\appendix

\section{Additional Definitions}

We provide here some additional definitions of quantities that will be used in our additional lemmas and proofs. First, we define the bounded linear operators $C_n$, $C$, and $C_0$ as in \cref{sec:kernel-vmm}. Specifically, given \cref{lem:rkhs-dual}, we have the following more useful equivalent definitions:
\begin{align}
    \label{eq:cn}
    \inner{C_n u}{v} &= \e_n[\varphi(u)\tr  \rho(X; \tilde\theta_n) \rho(X; \tilde\theta_n)\tr  \varphi(v)] \\
    \label{eq:c}
    \inner{C u}{v} &= \e[\varphi(u)\tr  \rho(X; \tilde\theta) \rho(X; \tilde\theta)\tr  \varphi(v)] \\
    \label{eq:c0}
    \inner{C_0 u}{v} &= \e[\varphi(u)\tr  \rho(X; \theta_0) \rho(X; \theta_0)\tr  \varphi(v)] \,,
\end{align}

where $\varphi$ maps any element of $\mathcal H$ to its corresponding dual element in $\mathcal F$. Note that by \cref{lem:c-compact} these operators are compact, and therefore well-defined on all of $\mathcal H$. In addition, we define the operators $B_n$ and $B$ according to
\begin{align}
    \label{eq:bn}
    B_n &= (C_n + \alpha_n I)^{-1/2} \\
    \label{eq:b}
    B &= C^{-1/2}\,,
\end{align}
wherever these operators are well-defined.

Next, as in \cref{sec:kernel-vmm}, we define the special elements of $\mathcal H$, given by
\begin{align}
    \label{eq:hn}
    \bar h_n(\theta)(f) &= \e_n[f(Z)\tr  \rho(X; \theta)] \\
    \label{eq:h}
    \bar h(\theta)(f) &= \e[f(Z)\tr  \rho(X; \theta)] \,.
\end{align}

We define the empirical and population objective functions according to
\begin{align}
    \label{eq:jn}
    J_n(\theta) = \sup_{f \in \mathcal F} \e_n[f(Z)\tr  \rho(X;\theta)] - \frac{1}{4} \e_n[(f(Z)\tr  \rho(X;\tilde\theta_n))^2] - \frac{\alpha_n}{4} \norm{f}^2 \\
    \label{eq:j}
    J(\theta) = \sup_{f \in \mathcal F} \e[f(Z)\tr  \rho(X;\theta)] - \frac{1}{4} \e_n[(f(Z)\tr  \rho(X;\tilde\theta))^2]\,.
\end{align}

Note that by \cref{lem:kvmm-cgmm-equiv} the above definition for $J_n$ is equivalent to $J_n(\theta) = \norm{B_n \bar h_n(\theta)}^2$. In addition, as argued in the proof of \cref{lem:bh-value}, the above definition for $J$ is equivalent to $J(\theta) = \norm{B \bar h(\theta)}^2$.

Next, we provide the following definition of well-behaved ``$\rho$-like'' classes 
\begin{definition}[Moment Regular Classes]
\label{def:regular}

We say that a class of functions $\Qcal = \{q(\cdot;\theta) : \theta \in \Theta\}$ indexed by $\Theta$ that map $\Xcal \to \RR^m$ is \emph{moment regular} if it satisfies the following properties:
\begin{enumerate}
    \item $\sup_{x \in \Xcal, \theta \in \Theta} |q(x;\theta)| \leq \infty $
    \item $q(X;\theta)$ is Lipschitz continuous in $\theta$ under the $L_1$ norm
    \item $\Qcal$ is $\Pcal$-Donsker
\end{enumerate}

\end{definition}

Finally, for any set $S$ in a metric space we denote by $N_{\epsilon}(S)$ the $\epsilon$-covering number of $S$, which is defined as the cardinality of the smallest finite set $T$ such that $\sup_{s \in S} \min_{t \in T} \norm{s - t} \leq \epsilon$.

\section{Additional Technical Lemmas}
\label{apx:additional-lemmas}

\begin{lemma}
\label{lem:psd-operator-inequality}

Suppose $A$ and $B$ are both diagonalizable operators on Hilbert space $\mathcal H$, whose eigenvalues are non-negative and bounded away from zero and infinity. Then $\norm{(A + B)^{-1}} \leq \norm{A^{-1}} + \norm{B^{-1}}$.

\end{lemma}

\begin{proof}[Proof of \Cref{lem:psd-operator-inequality}]

First, note that given $A$ and $B$ are diagonalizable with positive spectrum bounded away from zero and infinity, clearly $A^{-1}$, $B^{-1}$, $A+B$, and $(A+B)^{-1}$ are also. Now, for any given operator $C$, let $\sigma_{\text{min}}(C)$ and $\sigma_{\text{max}}(C)$ denote the minimum and maximum eigenvalues of $C$ respectively. Then we have
\begin{align*}
    \norm{(A+B)^{-1}} &= \sigma_{\text{max}}((A+B)^{-1}) \\
    &= \sigma_{\text{min}}(A+B)^{-1} \\
    &\leq (\min(\sigma_{\text{min}}(A), \sigma_{\text{min}}(B)))^{-1} \\
    &= \max(\sigma_{\text{min}}(A)^{-1}, \sigma_{\text{min}}(B)^{-1}) \\
    &= \max(\sigma_{\text{max}}(A^{-1}), \sigma_{\text{max}}(B^{-1})) \\
    &\leq \norm{A^{-1}} + \norm{B^{-1}}\,.
\end{align*}

\end{proof}

\begin{lemma}
\label{lem:lipschitz-convergence}
Suppose $Q_n$ is a real valued sequence of stochastic process on some index set $T$, such that $Q_n(t) \to 0$ in probability for every $t \in T$. In addition suppose that $T$ is totally bounded with respect to some metric, and $Q_n$ is $\alpha$-H\"older with respect to this metric almost surely, for some $\alpha \in (0, 1]$, with constant that doesn't depend on the data or $n$. Then $\sup_{t \in T} \abs{Q_n(t)} \to 0$ in probability. Moreover $Q_n \to 0$ in distribution as a stochastic process under the $l^\infty(T)$ metric.
\end{lemma}

\begin{proof}[Proof of \Cref{lem:lipschitz-convergence}]
Fix some arbitrary $\epsilon > 0$, and let $N_\epsilon$ be some $\epsilon$-covering of $T$, which is finite since $T$ is totally bounded by assumption. Also, for any $t \in T$ let $n_\epsilon(t)$ be an $\epsilon$-close element of $N_\epsilon$. Then with probability 1 we have
\begin{align*}
    \sup_{t \in T} \abs{Q_n(t)} &\leq \sup_{t \in T} \abs{Q_n(n_\epsilon(t))} + L \epsilon^\alpha \\
    &= \max_{t \in N_\epsilon} \abs{Q_n(t)} + L \epsilon^\alpha.
\end{align*}
Now $N_\epsilon$ is finite so therefore $\max_{t \in N_\epsilon} \abs{Q_n(t)} \to 0$ in probability given our assumptions and applying a union bound, so for every $\epsilon' > 0$ we have
\begin{align*}
    P(\sup_{t \in T} \abs{Q_n(t)} > \epsilon') &\leq P(\max_{t \in N_\epsilon} \abs{Q_n(t)} + L \epsilon^{\alpha} > \epsilon') \\
    &\leq P(\max_{t \in N_\epsilon} \abs{Q_n(t)} > \epsilon' - L \epsilon^\alpha).
\end{align*}
Now the right hand side of the above bound converges to zero in probability as long as $\epsilon' - L \epsilon^\alpha > 0$. However this bound holds for arbitrary $\epsilon > 0$ and $\epsilon' > 0$, so for any given $\epsilon' > 0$ we could pick \emph{e.g.} $\epsilon = (\epsilon'/ (2 L))^{1/\alpha}$, which gives $P(\sup_{t \in T} \abs{Q_n(t)} > \epsilon') \leq P(\max_{t \in N_{\epsilon}} \abs{Q_n(t)} > \epsilon'/2)$. Therefore, $P(\sup_{t \in T} \abs{Q_n(t)} > \epsilon') \to 0$ for every $\epsilon'$, which gives us the first result.

Finally, the second result then follows trivially from definition of the $l^\infty(T)$ metric and convergence in probability, and the fact that convergence in probability implies convergence in distribution.
\end{proof}

\begin{lemma}
\label{lem:mvt}
Let the Hilbert Space $\mathcal H$ be given. Suppose $h_\theta \in \mathcal H$ indexed by $\theta \in \Theta \subseteq \mathbb R^b$ is differentiable on the line segment from $\theta_s$ to $\theta_e$, where $\theta_e$ and $\theta_e$ are in the interior of $\Theta$, meaning that $(\partial h_\theta / \partial \theta_j)(\theta') \in \mathcal H$ for all $j \in \{1,\ldots,b\}$ and $\theta' \in \text{segment}(\theta_s,\theta_e)$. Then there exist some sequence of probabilities $\alpha_i$ (satisfying $\alpha_i \geq 0 \ \forall i$, $\sum_{i=1}^\infty \alpha_i = 1$) and parameters $\theta_i$ (satisfying $\theta_i \in \text{segment}(\theta_s, \theta_e) \ \forall i)$ such that
\begin{equation*}
h_{\theta_e} - h_{\theta_s} = \left( \sum_{i=1}^{\infty} \alpha_i \frac{\partial h_{\theta}}{\partial \theta}(\theta_i) \right)\tr  (\theta_e - \theta_s).
\end{equation*}
\end{lemma}

\begin{proof}[Proof of \Cref{lem:mvt}]
First note that for any fixed $g \in \mathcal H$ the function $f(\theta) = \inner{h_\theta}{g}$ is a real-valued function of $\theta$. Thus by the standard mean value theorem, we have
\begin{align}
\label{eq:standard-mvt-application}
    \inner{h_{\theta_e}}{g} - \inner{h_{\theta_s}}{g} &= \left( \frac{\partial}{\partial \theta} \inner{h_\theta}{g}(\bar \theta) \right)\tr  (\theta_e - \theta_s) \nonumber \\
    &= \inner{\frac{\partial h_\theta}{\partial \theta}(\bar \theta)\tr  (\theta_e - \theta_s)}{g},
\end{align}
for some $\bar \theta \in \text{segment}(\theta_s, \theta_e)$, where $\bar \theta$ may depend on the choice of $g$. Next, for any given $\theta' \in \Theta$, define the function $g(\theta') \in \mathcal H$ by
\begin{equation*}
    g(\theta') = h_{\theta_e} - h_{\theta_s} - \frac{\partial h_\theta}{\partial \theta}(\theta')\tr  (\theta_e - \theta_s).
\end{equation*}

Now let $\mathcal G \subseteq \mathcal H$ be defined as the convex hull of all such functions $g(\theta)$ for $\theta \in \text{segment}(\theta_s,\theta_e)$. For any given $g \in \mathcal G$, define $N(g) = \{g' \in \mathcal G :  \inner{g}{g'} = 0\}$. First of all, we can observe from \cref{eq:standard-mvt-application} that $N(g)$ is non-empty for every $g \in \mathcal G$, since there is some $\bar \theta \in \text{segment}(\theta_s,\theta_e)$ satisfying $\inner{g}{g(\bar \theta)} = 0$. Next, $N(g)$ is clearly a convex set for every $g \in \mathcal G$, since if $\inner{g}{g_1} = \inner{g}{g_2} = 0$ it is trivial from linearity that $\inner{g}{t g_1 + (1-t) g_2} = 0$. Finally the set function $N$ forms a closed graph, since if we have sequences $g_i$ and $g_i'$ satisfying $\inner{g_i}{g_i'} \ \forall i$ with limits $g_\infty$ and $g_\infty'$, it is trivial to verify that $\inner{g_\infty}{g_\infty'} = 0$.

Given the above $N$ satisfies the conditions of Kakutani's fixed point theorem and therefore has a fixed point $g^*$. This fixed point satifies $\inner{g^*}{g^*} = 0$, and therefore $g^* = 0$. Therefore by construction of $\mathcal G$ we have sequences $\alpha_i$ and $\theta_i$ satisfying
\begin{equation*}
    h_{\theta_e} - h_{\theta_s} - \sum_{i=1}^\infty \alpha_i \frac{\partial h_\theta}{\partial \theta}(\theta_i)\tr  (\theta_e - \theta_s) = 0,
\end{equation*}
and the result trivially follows.
\end{proof}

\begin{lemma}
\label{lem:rkhs-dual}
    Suppose $\mathcal F$ is a Hilbert Space, and $\mathcal H$ is the dual space of bounded linear functionals of the form $\mathcal F \to \mathbb R$. For any $f \in \mathcal F$, let $\varphi(f)$ denote the corresponding element of $\mathcal H$ under the isomorphism from the Riesz representation theorem. Then for any $f \in \mathcal F$ and $h \in \mathcal H$, we have $\inner{\varphi(f)}{h}_{\mathcal H} = h(f)$.
\end{lemma}

\begin{proof}[Proof of \cref{lem:rkhs-dual}]
    For any $f,g \in \mathcal F$, by the Riesz representation theorem and the definition of $\varphi$ we have
    \begin{equation*}
        \inner{f}{g}_{\mathcal F} = \varphi(g)(f) = \inner{\varphi(f)}{\varphi(g)}_{\mathcal H}
    \end{equation*}
    Now, the above equation holds for arbitrary $f$ and $\varphi(g)$, and since $\varphi$ is a bijection, $\varphi(g)$ can take arbitrary values in $\mathcal H$. Therefore, the required result immediately follows.
\end{proof}

\begin{lemma}
\label{lem:hilbert-norm}
Given any Hilbert space $\mathcal H$, and element $h \in \mathcal H$, we have 
\begin{equation*}
    \norm{h}^2 = \sup_{h' \in \mathcal H} \inner{h}{h'} - \frac{1}{4} \norm{h'}^2\,.
\end{equation*}
\end{lemma}

\begin{proof}[Proof of \Cref{lem:hilbert-norm}]

First note that this identity is tirivial in the case that $h=0$, so in what follows we will assume that $\norm{h} > 0$. By Cauchy Schwartz we have for any $h \in \mathcal H$, $\norm{h}^2 = \sup_{\norm{h'}^2 \leq \norm{h}^2} \inner{h}{h'}$. Noting that this is a constrained convex optimization problem, we can compute its Lagrangian as
\begin{equation*}
    \mathcal L(h', \lambda) = \inner{h}{h'} + \lambda(\norm{h'}^2 - \norm{h}^2)\,.
\end{equation*}

Taking the Gateaux derivative of this quantity with respect to $h'$ in the direction of $\epsilon$ we obtain
\begin{equation*}
    D(\mathcal L(h', \lambda))(h',\epsilon) = \inner{h}{ \epsilon} + 2\lambda \inner{h'}{\epsilon}\,.
\end{equation*}

Thus we have a critical point when $h + 2 \lambda h' = 0$, and we can verify easily that this is a maximum point whenever $\lambda < 0$. This gives us a dual formulation for $\norm{h}^2$ is given by
\begin{align*}
    \norm{h}^2 &= \inf_{\lambda < 0} -\frac{1}{2\lambda} \norm{h}^2 + \lambda (\frac{1}{4\lambda} \norm{h}^2 - \norm{h}^2) \\
    &= \inf_{\lambda < 0} -\frac{1}{4\lambda} \norm{h}^2 - \lambda \norm{h}^2\,.
\end{align*}

Taking derivative with respect to $\lambda$ we can see that this is minimized by setting $\lambda = -\frac{1}{2}$. Given this and strong duality, which follows easily from Slater's condition since the optimization problem has the feasible interior solution $h'=0$, we know it must be the case that $\norm{h}^2 = \sup_{h'} \inner{h}{h'} - (1/2) \norm{h'}^2 + (1/2) \norm{h}^2$. Finally, rearranging terms and doing a change of variables $h' \leftarrow (1/2)h'$ gives us the required identity.

\end{proof}

\begin{lemma}
\label{lem:operator-sqrt-inverse}
Let $\mathcal H$ and $C$ be given, where $\mathcal H$ is a separable Hilbert space, and $C$ is a compact, self-adjoint, PSD linear operator on $\mathcal H$. In addition, for any $\alpha \geq 0$, let $\mathcal H^+_\alpha$ be the maximal subspace of $\mathcal H$ on which $(C + \alpha I)^{-1/2}$ is well-defined. Then for any $h \in \mathcal H^+_\alpha$ we have 
\begin{equation*}
    \norm{(C+\alpha I)^{-1/2} h}^2 = \sup_{h' \in H} \inner{h}{h'} - \frac{1}{4} \inner{C h'}{h'} - \frac{\alpha}{4} \norm{h'}^2 \,.
\end{equation*}
Furthermore, $\mathcal H^+_\alpha$ is exactly characterized as the set of all $h \in \mathcal H$ where the above supremum is finite.
\end{lemma}

\begin{proof}[Proof of \Cref{lem:operator-sqrt-inverse}]
First, consider any arbitrary $h \in \mathcal H^+_\alpha$. Applying \cref{lem:hilbert-norm} we have
\begin{align*}
    \norm{(C+\alpha I)^{-1/2} h}^2 &= \inner{(C+\alpha I)^{-1/2} h}{(C+\alpha I)^{-1/2} h} \\
    &= \sup_{h' \in \mathcal H} \inner{(C+\alpha I)^{-1/2} h}{h'} - \frac{1}{4} \inner{h'}{h'} \\
    &= \sup_{h' \in \mathcal H} \inner{h}{(C+\alpha I)^{-1/2} h'} - \frac{1}{4} \inner{h'}{h'} \\
    &= \sup_{h' \in \text{Range}((C + \alpha I)^{-1/2})} \inner{h}{h'} - \frac{1}{4} \inner{(C+\alpha I)^{1/2} h'}{(C+\alpha I)^{1/2} h'} \\
    &= \sup_{h' \in \mathcal H} \inner{h}{h'} - \frac{1}{4} \inner{C h'}{h'} - \frac{\alpha}{4} \norm{h'}^2 \,,
\end{align*}
where in this derivation we make use of the fact that $C$ is self-adjoint (and therefore $(C+\alpha I)^{1/2}$ and $(C+\alpha I)^{-1/2}$ are also), and that $C$ is compact, so therefore $(C+\alpha I)^{1/2}$ is well-defined on all of $\mathcal H$ and $\text{Range}((C+\alpha I)^{-1/2}) = \mathcal H$.

For the second part of the theorem, we note that by definition $C$ is compact, and therefore is diagonalizable. Let $\{h_1,h_2,\ldots\}$ denote the set of the corresponding orthonormal eigenvectors, and let $\{\sigma_1,\sigma_2,\ldots,\}$ denote the corresponding eigenvalues. Now, clearly $C h = \sum_{i=1}^\infty \sigma_i a_i h_i$ and $(C+aI)^{1/2} h = \sum_{i=1}^\infty (\sigma_i + \alpha)^{1/2} a_i h_i$, so therefore $h \in \mathcal H^+$ if and only if $\sum_{i=1}^\infty a_i^2 / (\sigma_i + \alpha) < \infty$, where in this series we let $0/0 = 0$.

Now, suppose that $h \notin \mathcal H^+_\alpha$, meaning that $\sum_{i=1}^\infty a_i^2 / (\sigma_i + \alpha) = \infty$. Define $h'_n = \sum_{i=1}^n (2 a_i / (\sigma_i+\alpha)) h_i$. Then clearly $h'_n \in \mathcal H$ for every finite $n$. This construction gives us
\begin{equation*}
    \inner{h}{h'_n} - \frac{1}{4} \inner{(C+\alpha I) h'_n}{h'_n} = \left( \sum_{i=1}^n a_i^2 / (\sigma_i + \alpha) \right)\,, 
\end{equation*}
and therefore we have $\lim_{n \to \infty} \inner{h}{h'_n} - (1/4) \inner{C h'_n}{h'_n} - (\alpha/4) \norm{h'_n}^2 = \infty$. This immediately justifies that $\sup_{h' \in \mathcal H} \inner{h}{h'} - (1/4) \inner{C h'}{h'} - (\alpha/4) \norm{h'}^2$.

Conversely, suppose that $h \in \mathcal H^+_\alpha$. Then by definition $\norm{(C+\alpha)^{-1/2} h}^2 < \infty$, which immediately given the first part of the lemma gives us that $\sup_{h' \in \mathcal H} \inner{h}{h'} - (1/4) \inner{C h'}{h'} - (\alpha/4) \norm{h'}^2 < \infty$.

Therefore we have argued both directions, so $\mathcal H^+_\alpha$ exactly corresponds to the space where this supremum is finite, as required.

\end{proof}

\begin{lemma}
\label{lem:f-properties}

Suppose that $\mathcal F$ is an RKHS satisfying \cref{asm:rkhs}, and let \cref{asm:regularity} be given. In addition, let $F_1 = \{f \in \mathcal F : \norm{f} \leq 1\}$. Then $\mathcal F$ has the following properties:
\begin{enumerate}
    \item $\log N_\epsilon(F_1, L_\infty) \leq k_N \epsilon^{-b} + \log m$ for some $0 < b < 1$, and some constant $k_N$.
    \item There exists some positive constant $k_F$ such that $\norm{f}_{\infty} \leq k_F \norm{f}_{\mathcal F}$ for every $f \in \mathcal F$.
\end{enumerate}

\end{lemma}

\begin{proof}[Proof of \Cref{lem:f-properties}]

By \cref{asm:rkhs}, $\mathcal F$ is the direct sum of $m$ scalar-valued RKHSs; that is, we have
\begin{equation*}
    \mathcal F = \mathcal F_1 \bigoplus \mathcal F_2 \bigoplus \ldots \bigoplus \mathcal F_m\,,
\end{equation*}
where for each $i \in [m]$, $\mathcal F_i$ is scalar-valued RKHS with kernel $K_i$, and for each $f = (f_1,\ldots,f_m) \in \mathcal F$ and $f' = (f'_1,\ldots,f'_m) \in \mathcal F$ we have
\begin{equation*}
    \inner{f}{f'}_{\mathcal F} = \sum_{i=1}^m \inner{f_i}{f'_i}_{\mathcal F_i}\,.
\end{equation*}

Now, let $i \in [m]$ be fixed. By \cref{asm:regularity}, $\mathcal Z$ is a compact-valued subset of $\mathbb R^{d_z}$ for some positive integer $d_z$, and $K_i$ is $C^\infty$-smooth. In addition, let $F_{1,i} = \{f \in \mathcal F_i : \norm{f}_{\mathcal F_i} \leq 1\}$. Then by theorem D of \citet{cucker2002mathematical}, for every positive integer $h>d$, we have $\log N_{\epsilon}(F_{1,i}, L_{\infty}) \leq (k_{h,i} / \epsilon)^{2d / h}$, where $k_{h,i}$ is some constant that depends on $h$ and $i$ but is independent of $\epsilon$. Choosing $h = 2d+1$ gives us $\log N_{\epsilon}(F_{1,i}, L_{\infty}) \leq k_{2d+1,i}^b \epsilon^{-b}$, where $b = 2d / (2d + 1)$, which clearly satisfies $0 < b < 1$.

Next, we note that since $\norm{f - f'}_{\infty} = \max_{i \in [m]} \norm{f_i - f'_i}_{\infty}$ and $F_1 \subseteq F_{1,1} \bigoplus \ldots \bigoplus F_{1,m}$, any set that is simultaneously an $\epsilon$-cover of each of $F_{1,i}$ for $i \in [m]$ under $L_{\infty}$-norm must also be an $\epsilon$-cover of $F_1$ under $L_{\infty}$-norm. Therefore, we have
\begin{align*}
    N_\epsilon(F_1,L_\infty) &\leq \sum_{i=1}^m N_\epsilon(F_{1,i},L_\infty) \\
    &\leq m \exp(k_F \epsilon^{-b})\,,
\end{align*}
where $k_F = \max_{i \in [m]} k_{2d+1,i}^b$. This then gives us
\begin{equation*}
    \log N_\epsilon(F_1,L_\infty) \leq k_F \epsilon^{-b} + \log(m)\,,
\end{equation*}
which establishes our first property.

For the second property, let $K_{i,x}$ denote the element of $\mathcal F_i$ representing evaluation at $x$. We note that since $\mathcal X$ is compact and for each $i \in [m]$ we have that $K_i$ is$C^{\infty}$-smooth and therefore continuous, so by the extreme value theorem $\sup_{x \in \mathcal X} \sqrt{K_i(x,x)}$ is finite. Therefore, for any $f_i \in \mathcal F_i$, we have
\begin{align*}
    \norm{f_i}_\infty &= \sup_{x \in \mathcal X} \abs{f_i(x)} \\
    &= \sup_{x \in \mathcal X} \abs{\inner{f_i}{K_{i,x}}} \\
    &\leq \sup_{x \in \mathcal X} \norm{f_i}_{\mathcal F_i} \norm{K_{i,x}}_{\mathcal F_i} \\
    &= \sup_{x \in \mathcal X} \sqrt{K_i(x,x)} \norm{f_i}_{\mathcal F_i} \,,
\end{align*}
where the inequality step follows from Cauchy-Schwartz.

Next, let $k_F = \max_{i \in [m]} \sup_{x \in \mathcal X} \sqrt{K_i(x,x)}$. Then for any $f = (f_1,\ldots,f_m) \in \mathcal F$, we have
\begin{align*}
    \norm{f}_{\infty} &= \max_{i \in [m]} \norm{f_i}_{\infty} \\
    &\leq \sum_{i=1}^m \norm{f_i}_{\infty} \\
    &\leq \sum_{i=1}^m \sup_{x \in \mathcal X} \sqrt{K_i(x,x)} \norm{f_i}_{\mathcal F_i} \\
    &\leq k_F \sum_{i=1}^m \norm{f_i}_{\mathcal F_i} \\
    &= k_F \norm{f}_{\mathcal F}\,,
\end{align*}
which establishes our second property.

\end{proof}

\begin{lemma}
\label{lem:f-donsker}

Suppose that $\Fcal$ satisfies the assumptions of \cref{lem:f-properties}, and let $F_1 = \{f \in \Fcal : \|f\| \leq 1\}$. Then it is $\Pcal$-Donsker. 

\end{lemma}

\begin{proof}[Proof of \Cref{lem:f-donsker}]

Let $\Qcal$ denote the space of all possible probability distributions on $Z$. Then, given the result of \cref{lem:f-properties}, for any finite $S > 0$, we have the uniform entropy integral
\begin{align*}
    \int_0^S \sqrt{ \sup_{P \in \mathcal Q} \log N_{\epsilon}(\mathcal F_1, L_2(P)) } d \epsilon  &\leq  \int_0^S \sqrt{ \log N_{\epsilon}(\mathcal F_1, L_\infty) } d \epsilon \\
    &\leq \int_0^S k_F^{1/2} \epsilon^{-b/2} d \epsilon \,,
\end{align*}
which is finite since $b/2 < 1$. Also, given \cref{lem:f-properties} we know that $\norm{f} \leq k_F$ for every $f \in F_1$, so we have a square integrable envelope for $F_1$, which is given by the constant function $f(x) = k_F$.

Given this finite uniform entropy integral and the envelope for $F_1$, the conditions of theorem 8.19 of \citet{kosorok2007introduction} are satisfied, so we conclude that $F_1$ is $\Pcal$-Donsker.

\end{proof}

\begin{lemma}
\label{lem:uniform-clt}

Suppose that $\mathcal G$ is a class of functions of the form $g : \mathcal X \to \mathbb R$, and that $\mathcal G$ is $P$-Donsker in the sense of \citet{kosorok2007introduction}. Then we have
\begin{equation*}
    \sup_{g \in \mathcal G} \e_n[g(X)] - \e[g(X)] = O_p(n^{-1/2}).
\end{equation*}

\end{lemma}

\begin{proof}[Proof of \Cref{lem:uniform-clt}]

Let $Q_n$ denote the sequence of stochastic processes indexed by $\mathcal G$, defined by
\begin{equation*}
    Q_n(g) = \sqrt{n}(\e_n[g(X)] - \e[g(X)])\\.
\end{equation*}

Then by the definition of $P$-Donsker in \citet{kosorok2007introduction}, we have that $Q_n$ converges in distribution under the $L_\infty(\mathcal G)$-norm to a tight Gaussian process. Next, since supremum is a continuous function w.r.t. this norm, by the continuous mapping theorem we have that $\sup_{g \in \mathcal G} Q_n(g)$ converges to some limiting distribution. Therefore $\sqrt{n} \sup_{g \in \mathcal G} \e_n[g(X)] - \e[g(X)]$ is stochastically bounded, and so $\sup_{g \in \mathcal G} \e_n[g(X)] - \e[g(X)] = O_p(n^{-1/2})$.

\end{proof}

\begin{lemma}
\label{lem:c-compact}

Let the assumptions of \cref{thm:consistency} be given. Then the operators $C$ and $C_n$ defined in \cref{eq:c,eq:cn} are compact.

\end{lemma}

\begin{proof}[Proof of \Cref{lem:c-compact}]

First we consider the case of $C$, which is defined according to
\begin{equation*}
    (C h)(g) = \e[\varphi(f)(Z)\tr  \rho(X;\tilde\theta) \rho(X;\tilde\theta)\tr  g(Z)]\,,
\end{equation*}
where $\varphi$ maps any element of $\mathcal H$ to its dual element in $\mathcal F$.

Let $H_1 = \{h \in \mathcal H : \norm{h} \leq 1\}$, $F_1 = \{f \in \mathcal F : \norm{f} \leq 1\}$, and $R_1 = \{C h : h \in H_1\}$. Recall that by the Riesz representation theorem, $F_1$ and $H_1$ are isometrically isomorphic. In addition, by \cref{asm:rkhs}, there must exist some finite $\epsilon'$-cover of $F_1$ under the $L_\infty$ norm for every $\epsilon'>0$, which we will denote by $N(\epsilon')$. In order to establish that $C$ is a compact operator, we need to argue that $R_1$ is a compact set; that is, there exists some finite $\epsilon$-cover of $R_1$ for every $\epsilon>0$, under the $\mathcal H$ norm.

Now, let $h \in H_1$ be an arbitrary given element. In addition, let $f'$ be an element of $N(\epsilon')$ such that $\norm{\varphi(h) - f'}_\infty \leq \epsilon'$, and define $h' \in \mathcal H$ according to
\begin{equation*}
    h'(g) = \e[f'(Z)\tr  \rho(X;\tilde\theta) \rho(X;\tilde\theta)\tr  g(Z)]\,.
\end{equation*}

Then, we have
\begin{align*}
    \norm{h - h'}_{\mathcal H} &= \sup_{h'' \in H_1} \inner{h - h'}{h''} \\
    &= \sup_{f'' \in F_1} (h - h')(f'') \\
    &= \sup_{f'' \in F_1} \e[(f - f')(Z)\tr  \rho(X;\tilde\theta) \rho(X;\tilde\theta)\tr  f''(Z)] \\
    &\leq \sup_{f'' \in F_1} \e[(f - f')(Z)\tr  \rho(X;\tilde\theta) \rho(X;\tilde\theta)\tr  f''(Z)] \\
    &\leq \sup_{f'' \in F_1} m \epsilon' \norm{\rho(X;\tilde\theta) \rho(X;\tilde\theta)\tr  f''(Z)}_\infty \\
    &\leq k m \epsilon'\,,
\end{align*} for some constant $k$, where in the second equality we apply \cref{lem:rkhs-dual}, and in the final inequality we apply \cref{asm:regularity} and \cref{lem:f-properties}. Now, the set of all such $h'$ that could be used in the bound above is finite (with cardinality at most $\abs{N(\epsilon')}$), and thus setting $\epsilon' = \frac{1}{mk} \epsilon$ we have a finite $\epsilon$-cover of $H_1$. Therefore, since the above reasoning holds for arbitrary $\epsilon > 0$, we have that $C$ is compact.

Finally, for $C_n$ we note that the proof is identical to that of $C$, simply replacing $\e$ everywhere by $\e_n$ and replacing $\tilde\theta$ everywhere by $\tilde\theta_n$, since none of the steps are changed by these replacements.
    
\end{proof}

\begin{lemma}
\label{lem:cn-convergence}

Let $C$ and $C_n$ be defined as in \cref{eq:c,eq:cn}. Then, given the assumptions of \cref{thm:consistency}, we have $\norm{C_n - C} = O_p(n^{-p})$, where $p$ is the constant referred to in \cref{asm:tilde-theta}.

\end{lemma}

\begin{proof}[Proof of \Cref{lem:cn-convergence}]

First, let $\tilde C_n$ be defined according to
\begin{equation*}
    (\tilde C_n h)(g) = \e[\varphi(h)(Z)\tr  \rho(X;\tilde\theta_n) \rho(X;\tilde\theta_n)\tr  g(Z)]\,,
\end{equation*}
for any $h \in \mathcal H$ and $g \in \mathcal F$, where $\varphi$ maps any element of $\mathcal H$ to its dual element in $\mathcal F$. In addition, let $H_1 = \{h \in \mathcal H : \norm{h} \leq 1\}$, and $F_1 = \{f \in \mathcal F : \norm{f} \leq 1\}$. In addition, let $Q = \sup_{x \in \mathcal X, \theta \in \Theta} \abs{\rho(x;\theta)}$, which by \cref{asm:regularity} is finite. Given this, we have
\begin{equation*}
    \norm{C_n - C} \leq \norm{C_n - \tilde C_n} + \norm{\tilde C_n - C}\,,
\end{equation*}

We will proceed by bounding these two terms separately. For the first, we have
\begin{align*}
    \norm{C_n - \tilde C_n} &= \sup_{h \in H_1} \norm{(C_n - \tilde C_n) h} \\
    &= \sup_{h,h' \in H_1} \inner{(C_n - \tilde C_n) h}{h'} \\
    &= \sup_{f,f' \in F_1} \e_n[f(Z)\tr  \rho(X;\tilde\theta_n) \rho(X;\tilde\theta_n)\tr  f'(Z)] - \e[f(Z)\tr  \rho(X;\tilde\theta_n) \rho(X;\tilde\theta_n)\tr  f'(Z)] \\
    &\leq \sup_{g \in \mathcal G^2} \e_n[g(X)] - \e[g(X)]\,,
\end{align*}
where
\begin{align*}
    \mathcal G &= \{g : g(x) = f(z)\tr  \rho(x; \theta), f \in F_1, \theta \in \Theta \} \\
    \mathcal G^2 &= \{g : g(x) = g'(x) g''(x), g' \in \mathcal G, g'' \in \mathcal G\}\,.
\end{align*}

Now, applying corollary 9.32 of \citet{kosorok2007introduction} to \cref{asm:rho-complexity} and \cref{lem:f-donsker}, it trivially follows that $\mathcal G$ and $\mathcal G^2$ are $\Pcal$-Donsker. Therefore, by \cref{lem:uniform-clt} we have $\norm{\tilde C_n - C} = O_p(n^{-1/2})$.

Next we consider the term $\norm{\tilde C_n - C}$. By a similar reasoning to above, we have
\begin{align*}
    \norm{\tilde C_n - C} &= \sup_{h,h' \in H_1} \inner{(\tilde C_n - C) h}{h'} \\
    &= \sup_{f,f' \in F_1} \e[f(Z)\tr ( \rho(X;\tilde\theta_n) \rho(X;\tilde\theta_n)\tr  - \rho(X;\tilde\theta) \rho(X;\tilde\theta)\tr  )f'(Z)] \\
    &\leq  \sum_{i,j \in [m]} \sup_{f,f' \in F_1} \e\Big[f_i(Z) f_j(Z) \rho_i(X; \tilde\theta_n) \Big( \rho_j(X;\tilde\theta_n) - \rho_j(X;\tilde\theta) \Big) \Big] \\
    &\quad + \sum_{i,j \in [m]} \sup_{f,f' \in F_1} \e\Big[f_i(Z) f_j(Z) \rho_j(X; \tilde\theta) \Big( \rho_i(X;\tilde\theta_n) - \rho_i(X;\tilde\theta) \Big) \Big] \\
    &\leq 2 m^2 k_F^2  Q \e\Big[ \Big| \rho_i(X;\tilde\theta_n) - \rho_i(X;\tilde\theta) \Big| \Big] \,,
\end{align*}
where $k_F$ is the constant defined in \cref{lem:f-properties}. Now, by \cref{asm:tilde-theta} we have that $\rho(X;\tilde\theta_n)_i = \rho(X;\tilde\theta)_i + O_p(n^{-p})$ in $L_1$-norm for each $i \in [m]$, so therefore the right hand side of the above bound is $O_p(n^{-p})$.

Finally, putting the two bounds together, and noting that by \cref{asm:tilde-theta} we have $p \leq 1/2$, we can conclude that $\norm{C_n - C} = O_p(n^{-p})$.

\end{proof}

\begin{lemma}
\label{lem:c-inverse}

For any square integrable function $q$, let $h_q \in \mathcal H$ be defined according to $h_q(f) = \e[f(Z)^t q(Z)]$. Then we have $C^{-1} h_q = h_{\phi}$, where
\begin{equation*}
    \phi(z) = V(z; \tilde\theta)^{-1} q(z)\,.
\end{equation*}

\end{lemma}

\begin{proof}[Proof of \Cref{lem:c-inverse}]

By the definition of $C$, for any function $q$ we have
\begin{align*}
    (C h_q)(f) &= \e[q(Z)\tr  \rho(X;\tilde\theta)\rho(X;\tilde\theta)\tr  f(Z)] \\
    &= \e[q(Z)\tr  V(Z;\tilde\theta) f(Z)] \,.
\end{align*}

In addition, we have
\begin{align*}
    h_q(f) &= \e[q(Z)\tr  f(Z)] \\
    &= \e[(V(Z;\tilde\theta)^{-1}q(Z))\tr  V(Z;\tilde\theta) f(Z)] \\
    &= (C h_\phi)(f)\,,
\end{align*}
where $\phi$ is defined as in the lemma statement. The result immediately follows from this.

\end{proof}

\begin{lemma}
\label{lem:ht-properties}

Let $B$ and $B_n$ be defined as in \cref{eq:b,eq:bn}, and let some moment regular class (as in \cref{def:regular}) $\Qcal = \{q(\cdot;t) : t \in T\}$ be given. Also, let $h_n(t)$ and $h(t)$ be some collections of functions in $\mathcal H$ indexed by $T$, defined according to
\begin{align*}
    h_n(t)(f) &= \e_n[f(Z)\tr q(x;t)] \\
    h(t)(f) &= \e[f(Z)\tr  q(x;t)]\,,
\end{align*}
Then, under the assumptions of \cref{thm:consistency}, these sequences have the following properties:
\begin{enumerate}
    \item $\sup_{t \in T} \norm{h(t)} < \infty$.
    \item $\sup_{t \in T} \norm{h_n(t)} < \infty$.
    \item $\sup_{t \in T} \norm{B h(t)} < \infty$.
    \item $\sup_{t \in T} \norm{B^2 h(t)} < \infty$.
    \item $\sup_{t \in T} \norm{h_n(t) - h(t)} = O_p(n^{-1/2})$.
    \item $\norm{B(h(t') - h(t))} \leq L' \norm{t' - t}^{1/2} \ \forall t',t \in T$, for some constant $L'$ that doesn't depend on $t$ and $t'$.
    \item $\norm{B^2(h(t') - h(t))} \leq L'' \norm{t' - t} \ \forall t',t \in T$, for some constant $L''$ that doesn't depend on $t$ and $t'$.
\end{enumerate}

\end{lemma}

\begin{proof}[Proof of \Cref{lem:ht-properties}]

First, define $F_1 = \{f : f \in \mathcal F, \norm{f} \leq 1\}$. Given \cref{lem:f-properties}, we have that $F_1$ is uniformly bounded; that is, $\sup_{f \in F_1, z \in \mathcal Z} \abs{f(z)} < \infty$. Furhermore, given our assumptions

For the first of our required properties, we have
\begin{align*}
    \sup_{t \in T} \norm{h(t)} &= \sup_{t \in T, \norm{h'} \leq 1} \inner{h(t)}{h'} \\
    &= \sup_{t \in T, f \in F_1} \e[f(Z)\tr  q(X;t)]\,,
\end{align*}
which follows from \cref{lem:rkhs-dual}. This supremum is clearly finite given the uniform boundedness of $q$ and $F_1$, which gives us our first property.

For the second property, we have
\begin{align*}
    \sup_{t \in T} \norm{h_n(t)} &= \sup_{t \in T, \norm{h'} \leq 1} \inner{h_n(t)}{h'} \\
    &= \sup_{t \in T, f \in F_1} \e_n[f(Z)\tr  q(X;t)]\,,
\end{align*}
which again follows from \cref{lem:rkhs-dual}. As above, this supremum is clearly finite given the uniform boundedness of $q$ and $F_1$, which gives us our second property.

For the third of these properties, by \cref{lem:bh-value} we have
\begin{equation*}
    \norm{B h(t)}^2 = \e[\e[q(X;t) \mid Z]\tr  V(Z; \tilde\theta)^{-1} \e[q(X;t) \mid Z]]\,.
\end{equation*}
Now, by the uniform boundedness of $q$, it follows that the $L_2$ norm of $\e[q_i(X;t) \mid Z]$ is bounded by some constant independent of $t$. Then, given \cref{asm:non-degenerate}, we have $\sup_{t \in T} \norm{B h(t)}^2 < \infty$, which gives us our third property.

For the fourth of these properties, we note that $B^2 = C^{-1}$, and thus applying \cref{lem:c-inverse} and \cref{lem:rkhs-dual} we have
\begin{align*}
    \sup_{t \in T} \norm{C^{-1} h(t)} &= \sup_{t \in T, f \in F_1} \e[f(Z)\tr  V(Z; \tilde\theta)^{-1} \e[q(X; t) \mid Z]]\,.
\end{align*}
Again, by the uniform boundedness of $q$ and $F_1$, it follows that the $L_2$ norms of $\e[q_i(X;t) \mid Z]$ and $f_i$ are uniformly bounded for each $i$, so given \cref{asm:non-degenerate} we instantly have $\sup_{t \in T} \norm{B^2 h(t)} < \infty$, which establishes our fourth property.

For the fifth property, we have
\begin{align*}
    \sup_{t \in T} \norm{h_n(t) - h(t)} &= \sup_{t \in T, \norm{h'} \leq 1} \inner{h_n(t) - h(t)}{h'} \\
    &= \sup_{t \in T, f \in F_1} (\e_n[f(Z)\tr  q(X; t)] - \e[f(Z)\tr  q(X; t)]) \\
    &= \sup_{g \in \mathcal G} (\e_n[g(X)] - \e[g(X)])\,,
\end{align*}
where $\mathcal G = \{g : g(x) = f(z)\tr  q(x; t), f \in F_1, t \in T\}$. Now, by \cref{def:regular} and \cref{lem:f-donsker}, along with the Donsker preservation property of corollary 9.32 of \citet{kosorok2007introduction}, we have that $\mathcal G$ is $\Pcal$-Donsker. Therefore, by \cref{lem:uniform-clt} we have $\sup_{g \in \mathcal G} (\e_n[g(X)] - \e[g(X)]) = O_p(n^{-1/2})$, which gives us our fifth property.

For the sixth property, we first appeal again to \cref{lem:bh-value}, to establish that
\begin{equation*}
    \norm{B (h(t') - h(t))}^2 = \e[\e[q(X;t') - q(X;t) \mid Z]\tr  V(Z; \tilde\theta)^{-1} \e[q(X;t') - q(X;t) \mid Z]]\,.
\end{equation*}
Now, given the uniform boundedness of $F_1$ and \cref{asm:non-degenerate}, it follows that $\|\EE[q(X;t') - q(X;t) \mid Z]^\top V(Z;\tilde\theta)^{-1}\|_\infty \leq Q'$ for some $Q'$ independent of $t$ and $t'$.
Also, our Lipschitz assumption on $q$ implies that the $L_1$ norm of $\e[q_i(X;t') - q_i(X;t) \mid Z]$ is bounded by $k\norm{t' - t}$, for some constant $k$ that doesn't depend on $t'$ or $t$, for each $i$. Therefore, we have
\begin{equation*}
    \norm{B (h(t') - h(t))}^2 \leq m k Q' \|t' - t\| \,, 
\end{equation*}
Therefore, our sixth property follows with $L' = \sqrt{m k Q'}$.

Finally, for our seventh property, we note that by \cref{lem:c-inverse} we have
\begin{equation*}
    \norm{B^2(h(t') - h(t))} = \sup_{f \in F_1} \e[f(Z)\tr  V(Z; \tilde\theta)^{-1} \e[q(X; t') - q(X; t) \mid Z]]\,.
\end{equation*}
Now, given the uniform boundedness of $F_1$ and \cref{asm:non-degenerate}, it follows that $\|f(Z)^\top V(Z;\tilde\theta)^{-1}\|_\infty \leq Q''$ for some $Q''$ independent of $f \in F_1$. Then, following a similar reason as with the sixth property, we have our desired result with $L'' = m k Q''$.

\end{proof}

\begin{lemma}
\label{lem:bn-convergence}

Let $B$ and $B_n$ be defined as in \cref{eq:b,eq:bn}, let $T$ be some compact index set, and let $h_n(t)$ be some collections of functions in $\mathcal H$ indexed by $t \in T$, which satisfy the assumptions of \cref{lem:ht-properties}. Then under the assumptions of \cref{thm:consistency}, we have $\sup_{t \in T} \norm{B_n h_n(t) - B h(t)} \to 0$ in probability.

\end{lemma}

\begin{proof}[Proof of \Cref{lem:bn-convergence}]

Recall that $B_n = (C_n + \alpha_n I)^{-1/2}$, and $B = C^{-1/2}$. In addition, define
\begin{equation*}
    \tilde B_n = (C + \alpha_n I)^{-1/2}\,.
\end{equation*}

Given this definition, we have
\begin{equation*}
    \norm{B_n h_n(t) - B h(t)} \leq \norm{B_n h_n(t) - B_n h(t)} + \norm{B_n h(t) - \tilde B_n h(t)} + \norm{\tilde B_n h(t) - B h(t)}\,.
\end{equation*}

We proceed by bounding the three terms in the RHS of the above separately. For the first term, we have
\begin{align*}
    \sup_{t \in T} \norm{B_n h_n(t) - B_n h(t)} &\leq \sup_{t \in T} \norm{B_n} \norm{h_n(t) - h(t)} \\
    &\leq \norm{(\alpha_n I)^{-1/2}} \sup_{t \in T} \norm{h_n(t) - h(t)}  \\
    &\leq \norm{(\alpha_n I)^{-1/2}} O_p(n^{-1/2}) \\
    &= O_p(\alpha_n^{-1/2} n^{-1/2})\,.
\end{align*}
where in the above we make use of the fact that $\norm{(A + B)^{-1}} \leq \norm{A^{-1}}$ for PSD $A$ and $B$ sharing the same diagonalizing basis of orthonormal eigenvalues, and that $\sup_{t \in T} \norm{h_n(t) - h(t)} = O_p(n^{-1/2})$ by \cref{lem:ht-properties}.

Next, for the second term, we have
\begin{align*}
    &\sup_{t \in T} \norm{B_n h(t) - \tilde B_n h(t)} \\
    &= \sup_{t \in T} \norm{B_n (\tilde B_n^{-1} - B_n^{-1}) \tilde B_n h(t)} \\
    &\leq \sup_{t \in T} \norm{B_n} \norm{\tilde B_n^{-1} - B_n^{-1}} \norm{\tilde B_n h(t)} \\
    &\leq \sup_{t \in T} \norm{B h(t)} \norm{(\alpha_n I)^{-1/2}} \norm{(\tilde B_n^{-1} - B_n^{-1}) (\tilde B_n^{-1} + B_n^{-1}) (\tilde B_n^{-1} + B_n^{-1})^{-1}} \\
    &\leq  M \alpha_n^{-1/2} \norm{\tilde B_n^{-2} - B_n^{-2}} \norm{(\tilde B_n^{-1} + B_n^{-1})^{-1}} \\
    &\leq  M \alpha_n^{-1/2} \norm{C - C_n} (\norm{\tilde B_n} + \norm{B_n}) \\
    &\leq  M \alpha_n^{-1/2} O_p(n^{-p}) (\alpha_n^{-1/2} + \alpha_n^{-1/2}) \\
    &=  O_p(\alpha_n^{-1} n^{-p}) \,,
\end{align*}
where $M = \sup_{t \in T} \norm{B h(t)}$, which by \cref{lem:ht-properties} is finite, and $p$ is the constant referenced in \cref{asm:tilde-theta}. In addition, in this derivation we make use of the same fact about diagonalizable PSD operators as before, and we also appeal to \cref{lem:psd-operator-inequality,lem:cn-convergence}.

For the final term, note that by \cref{lem:c-compact} $C$ is compact, and so by the spectral theorem is diagonalizable. Let $v_1,v_2,\ldots$ denote the sequence of orthonormal eigenvectors forming a basis for $\mathcal H$, and let $\sigma_1,\sigma_2,\ldots$ denote the corresponding sequence of eigenvalues. Then, for any $h \in \text{domain}(B)$, we have
\begin{align*}
    \norm{\tilde B_n h - B h}^2 &= \sum_{i=1}^\infty \left( (\sigma_i + \alpha_n)^{-1/2} - \sigma_i^{-1/2} \right)^2 \inner{h}{v_i}^2 \\
    &\leq \sum_{i=1}^\infty (2 (\sigma_i + \alpha_n)^{-1} + 2 \sigma_i^{-1}) \inner{h}{v_i}^2 \\
    &\leq 4 \norm{B h}^2\,.
\end{align*}

Therefore, the infinite series defining $\norm{\tilde B_n h(t) - B h(t)}^2$ has a uniform convergent envelope that doesn't depend on $n$. In addition, since we have assumed the assumptions of \cref{thm:consistency}, and therefore $\alpha_n \to 0$, we have that each term in the series defining $\norm{\tilde B_n h - B h}^2$ converges to zero. Given these two observations, we can exchange limit and summation by Tannery's theorem, and we have $\norm{\tilde B_n h(t) - B h(t)}^2 \to 0$, and thus $\norm{\tilde B_n h(t) - B h(t)} \to 0$, for each $t \in T$.

Next, we argue that the above convergence holds uniformly over $T$. For any $t,t' \in T$, we can bound
\begin{align*}
    &\abs{\norm{\tilde B_n h(t) - B h(t)} - \norm{\tilde B_n h(t') - B h(t')}} \\
    &\leq \norm{(\tilde B_n h(t) - B h(t)) - (\tilde B_n h(t') - B h(t'))} \\
    &= \norm{\tilde B_n (h(t) - h(t')) - B (h(t) - h(t'))} \\
    &\leq 2 \norm{B (h(t) - h(t'))} \\
    &\leq 2 L' \norm{t - t'}^{1/2} \,,
\end{align*}
where $L'$ is the constant from \cref{lem:ht-properties}, and we apply the bound $\norm{\tilde B_n h - B h}^2 \leq 4 \norm{B h}^2$ from above. Therefore we have that $\norm{\tilde B_n h(t) - B h(t)}$ is $1/2$-H\"older in $t$, and by assumption $T$ is compact-valued, so we can apply \cref{lem:lipschitz-convergence} to conclude that $\sup_{t \in T} \norm{\tilde B_n h(t) - B h(t)} \to 0$.

Putting all of the above together, and noting that by the assumptions of \cref{thm:consistency} we have $\alpha_n = \omega(n^{-p})$ and $p \leq 1/2$, all three of the terms in the original bound must converge to zero in probability uniformly over $t \in T$, so we have our desired result.

\end{proof}

\begin{lemma}
\label{lem:bn-squared-convergence}

Let $B$ and $B_n$ be defined as in \cref{eq:b,eq:bn}, and let $h_n(t)$ be some collections of functions in $\mathcal H$ indexed by $t \in T$, which satisfy the assumptions of \cref{lem:ht-properties}. Then under the assumptions of \cref{thm:consistency}, we have $\sup_{t \in T} \norm{B_n^2 h_n(t) - B^2 h(t)} \to 0$ in probability.

\end{lemma}

\begin{proof}[Proof of \Cref{lem:bn-squared-convergence}]

Recall that $B_n^2 = (C_n + \alpha_n I)^{-1}$, and $B = C^{-1}$. In addition, define
\begin{equation*}
    \tilde B_n^2 = (C + \alpha_n I)^{-1}\,.
\end{equation*}

Given this definition, we have
\begin{align*}
    \sup_{t \in T} \norm{B_n^2 h_n(t) - B^2 h(t)} &\leq \sup_{t \in T} \norm{B_n^2 h_n(t) - B_n^2 h(t)} + \sup_{t \in T} \norm{B_n^2 h(t) - \tilde B_n^2 h(t)} \\
    &\qquad + \sup_{t \in T} \norm{\tilde B_n^2 h(t) - B^2 h(t)}\,.
\end{align*}

We proceed by bounding the three terms in the RHS of the above separately. For the first term, we have
\begin{align*}
    \sup_{t \in T} \norm{B_n^2 h_n(t) - B_n^2 h(t)} &\leq \norm{B_n^2} \sup_{t \in T} \norm{h_n(t) - h(t)} \\
    &\leq \norm{(\alpha_n I)^{-1}} \sup_{t \in T} \norm{h_n(t) - h(t)}  \\
    &\leq \norm{(\alpha_n I)^{-1}} O_p(n^{-1/2}) \\
    &= O_p(\alpha_n^{-1} n^{-1/2})\,,
\end{align*}
where in the above we apply \cref{lem:ht-properties}, and make use of the fact that $\norm{(A + B)^{-1}} \leq \norm{A^{-1}}$ for PSD $A$ and $B$ sharing the same diagonalizing basis of orthonormal eigenvalues.

Next, for the second term, we have
\begin{align*}
    \sup_{t \in T} \norm{B^2_n h(t) - \tilde B_n^2 h(t)} &= \sup_{t \in T} \norm{B_n^2 (\tilde B_n^{-2} - B_n^{-2}) \tilde B_n^2 h(t)} \\
    &\leq \norm{B_n^2} \norm{C_n - C} \sup_{t \in T} \norm{\tilde B_n^2 h(t) } \\
    &\leq \norm{(\alpha_n I)^{-1}} O_p(n^{-p}) \sup_{t \in T} \norm{B^2 h(t)} \\
    &=  O_p(\alpha_n^{-1} n^{-p})\,,
\end{align*}
where in this derivation we again make use of \cref{lem:ht-properties}, and the same fact about diagonalizable PSD operators as before. In addition, we apply \cref{lem:cn-convergence}.

For the final term, note that by \cref{lem:c-compact}, $C$ is compact, and so by the spectral theorem is diagonalizable. Let $v_1,v_2,\ldots$ denote the sequence of orthonormal eigenvectors forming a basis for $\mathcal H$, and let $\sigma_1,\sigma_2,\ldots$ denote the corresponding sequence of eigenvalues. Then, since $h(t) \in \text{domain}(B^2)$ for each $t$, we have
\begin{align*}
    \norm{\tilde B_n^2 h(t) - B^2 h(t)}^2 &= \sum_{i=1}^\infty \left( (\sigma_i + \alpha_n)^{-1} - \sigma_i^{-1} \right)^2 \inner{h(t)}{v_i}^2 \\
    &\leq \sum_{i=1}^\infty (2 (\sigma_i + \alpha_n)^{-2} + 2 \sigma_i^{-2}) \inner{h(t)}{v_i}^2 \\
    &\leq 4 \norm{B^2 h(t)}^2\,.
\end{align*}

Therefore, the infinite series defining $\norm{\tilde B_n^2 h(t) - B^2 h(t)}^2$ has a uniform convergent envelope that doesn't depend on $n$. In addition, since we have assumed the assumptions of \cref{thm:consistency}, and therefore $\alpha_n \to 0$, we have that each term in the series defining $\norm{\tilde B_n^2 h - B^2 h}^2$ converges to zero. Given these two observations, we can exchange limit and summation by Tannery's theorem, and we have $\norm{\tilde B_n^2 h(t) - B^2 h(t)}^2 \to 0$, and thus $\norm{\tilde B_n^2 h(t) - B^2 h(t)} \to 0$, for each $t \in T$.

Next, we argue that the above convergence holds uniformly over $T$. For any $t,t' \in T$, we can bound
\begin{align*}
    &\abs{\norm{\tilde B_n^2 h(t) - B^2 h(t)} - \norm{\tilde B_n^2 h(t') - B^2 h(t')}} \\
    &\leq \norm{(\tilde B_n^2 h(t) - B^2 h(t)) - (\tilde B_n^2 h(t') - B^2 h(t'))} \\
    &= \norm{\tilde B_n^2 (h(t) - h(t')) - B^2 (h(t) - h(t'))} \\
    &\leq 2 \norm{B^2 (h(t) - h(t'))} \\
    &\leq 2 L'' \norm{t - t'}\,,
\end{align*}
where $L''$ is the constant from \cref{lem:ht-properties}, and we apply the bound $\norm{\tilde B_n^2 h(t) - B^2 h(t)}^2 \leq 4 \norm{B^2 h(t)}^2$ from above. Therefore we have that $\norm{\tilde B_n^2 h(t) - B^2 h(t)}$ is Lipshitz in $t$, and by assumption $T$ is compact-valued, so we can apply \cref{lem:lipschitz-convergence} to conclude that $\sup_{t \in T} \norm{\tilde B_n^2 h(t) - B^2 h(t)} \to 0$.

Putting all of the above together, and noting that by the assumptions of \cref{thm:consistency} we have $\alpha_n = \omega(n^{-p})$, all three of the terms in the original bound must converge to zero in probability, so we have our desired result.

\end{proof}

\begin{lemma}
\label{lem:bh-value}
    Let $h_q \in \mathcal H$ be defined according to $h_q(f) = \e[f(Z)\tr  q(Z)]$, for some square integrable $q$, and let $B$ be defined as in \cref{eq:b}. Then we have
    \begin{equation*}
        \norm{B h_q}^2  = \e[q(Z)\tr  V(Z; \tilde\theta)^{-1} q(Z)]\,.
    \end{equation*}
\end{lemma}

\begin{proof}[Proof of \Cref{lem:bh-value}]

First note that by \cref{lem:c-compact} $C$ is compact, so therefore applying \cref{lem:operator-sqrt-inverse,lem:rkhs-dual} we have
\begin{equation*}
    \norm{B h_q}^2 = \sup_{f \in \mathcal F} U(\theta, f)\,,
\end{equation*}
where
\begin{equation*}
    U(\theta, f) = \e[f(Z)\tr  q(Z)] - \frac{1}{4} \e[(f(Z)\tr  \rho(X;\tilde\theta))^2]\,.
\end{equation*}

We will proceed by solving for the function $f^*$ that maximizes the above supremum. Consider the Gateaux derivative of $U(\theta,f)$ at $f^*$ in the direction of $\epsilon$, where $f^*$ and $\epsilon$ are both square integrable functions of $Z$. We have
\begin{align*}
    d(U(\theta,f); \epsilon)_{|f=f^*} &= \e[\epsilon(Z)\tr  q(Z)] - \frac{1}{2} \e[\epsilon(Z)\tr  \rho(X; \tilde\theta) \rho(X; \tilde\theta)\tr  f^*(Z)] \\
    &= \e \left[ \epsilon(Z)\tr  \left(q(Z) - \frac{1}{2} V(Z; \tilde\theta) f^*(Z) \right) \right] \,.
\end{align*}

By \cref{asm:non-degenerate} $V(Z; \tilde\theta)^{-1}$ is a bounded linear operator from $L_2$ to $L_2$, and by the assumptions of this lemma $q(Z)$ is in $L_2$, so therefore $V(Z; \tilde\theta)^{-1} q(Z)$ is in $L_2$ also. Thus, if we choose
\begin{equation*}
    f^*(Z) = 2 V(Z; \tilde\theta)^{-1} q(Z)\,,
\end{equation*}
the Gateaux above derivative is equal to zero for every square-integrable $\epsilon(Z)$. Now, $U(\theta,f)$ is concave in $f$, and it is easy to verify by second derivatives that this choice corresponds to a minimum, so therefore we conclude that $f^*$ as defined above as the global maximizer of $U(\theta, f)$ out of all square-integrable functions. 

Finally, we note that by \cref{asm:rkhs} $\mathcal F$ is universal, so there must exist some sequence $f_n$ in $\mathcal F$ such that $\norm{f_n - f^*}_2 \to 0$ as $n \to \infty$, regardless of whether or not $f^* \in \mathcal F$. Therefore, we conclude that
\begin{equation*}
    \norm{B h_q}^2 = \e[q(Z)\tr  V(Z; \tilde\theta)^{-1} q(Z)]\,.
\end{equation*}

\end{proof}

\begin{lemma}
\label{lem:j-continuous}

Given the assumptions of \cref{thm:consistency}, $J(\theta)$ as defined in \cref{eq:j} is Lipschitz continuous in $\theta$.

\end{lemma}

\begin{proof}[Proof of \Cref{lem:j-continuous}]

We first note that by \cref{lem:bh-value} 
\begin{equation*}
    J(\theta) = \e[\e[\rho(X; \theta) \mid Z]\tr  V(Z; \tilde\theta)^{-1} \e[\rho(X; \theta) \mid Z]]\,.
\end{equation*}

Now, define $q_{\theta}(Z) = \e[\rho(X; \theta) \mid Z]$. Given this, fore any $\theta, \theta' \in \Theta$ we have
\begin{align*}
    \abs{J(\theta) - J(\theta')} &= \abs{\e[q_\theta(Z)\tr  V(Z; \tilde\theta)^{-1} q_\theta(Z)] - \e[q_{\theta'}(Z)\tr  V(Z; \tilde\theta)^{-1} q_{\theta'}(Z)]} \\
    &\leq \abs{\e[q_\theta(Z)\tr  V(Z; \tilde\theta)^{-1} (q_\theta - q_{\theta'})(Z)]} \\
    &\qquad\qquad + \abs{\e[(q_{\theta} - q_{\theta'})(Z)\tr  V(Z; \tilde\theta)^{-1} q_{\theta'}(Z)]}\,.
\end{align*}

Now, it easily follows from \cref{asm:rho-complexity} and \cref{asm:non-degenerate} that we have $\sup_{\theta \in \Theta} \norm{q_{\theta}^\top V(Z;\tilde\theta)^{-1}}_{\infty} = Q$ for some constant $Q < \infty$. Therefore, we have
\begin{equation*}
    |J(\theta) - J(\theta')| \leq 2 m Q L \|\theta' - \theta\| \,,
\end{equation*}
where $L$ is the Lipschitz constant from \cref{asm:rho-complexity} such that $\EE[|q_{\theta'}(Z)_i - q_{\theta}(Z)_i|] \leq \EE[|\rho_i(X;\theta') - \rho_i(X;\theta)|] \leq L \|\theta' - \theta\|$ for each $i$.
Therefore, we have our required result with Lipschitz constant $2 m Q L$.

\end{proof}

\begin{lemma}
\label{lem:objective-convergence}

Let $J$ and $J_n$ be defined as in \cref{eq:j,eq:jn}. Then given the assumptions of \cref{thm:consistency}, we have $\sup_{\theta \in \Theta}\abs{J_n(\theta) - J(\theta)} \to 0$ in probability.

\end{lemma}

\begin{proof}[Proof of \Cref{lem:objective-convergence}]

First, for any fixed $\theta$, we can obtain the bound
\begin{align*}
    \abs{J_n(\theta) - J(\theta)} &= \abs{\norm{B_n \bar h_n(\theta)}^2 - \norm{B \bar h(\theta)}^2} \\
    &\leq \abs{\norm{B_n \bar h_n(\theta)} + \norm{B \bar h(\theta)}} \abs{\norm{B_n \bar h_n(\theta)} - \norm{B \bar h(\theta)}} \\
    &\leq (2 \norm{B \bar h(\theta)} + \epsilon_n(\theta)) \epsilon_n(\theta)\,,
\end{align*}
where
\begin{equation*}
    \epsilon_n(\theta) = \norm{B_n \bar h_n(\theta) - B \bar h(\theta)}\,.
\end{equation*}

Now, given \cref{asm:regularity}, the sequences given by $\bar h_n(\theta)$ with limit $\bar h(\theta)$ satisfy the assumptions of \cref{lem:ht-properties}, so it follows that $\sup_{\theta \in \Theta} \norm{B \bar h(\theta)} < \infty$. Next, by \cref{lem:bn-convergence} we have $\sup_{\theta \in \Theta} \epsilon_n(\theta) \to 0$ in probability. Putting these two observations together with the above bound, we get $\sup_{\theta \in \Theta} \abs{J_n(\theta) - J(\theta)} \to 0$ in probability, as required.

\end{proof}

\begin{lemma}
\label{lem:unique-solution}

Let $J$ be defined as in \cref{eq:j}. Then $\theta_0$ is the unique element of $\Theta$ such that $J(\theta) = 0$.

\end{lemma}

\begin{proof}[Proof of \Cref{lem:unique-solution}]

First, by \cref{lem:bh-value}, we have that
\begin{equation*}
    J(\theta) = \e[\e[\rho(X; \theta) \mid Z]\tr  V(Z; \tilde\theta)^{-1} \e[\rho(X; \theta) \mid Z]]\,,
\end{equation*}
where $V(Z;\tilde\theta) = \e[\rho(X; \tilde\theta) \rho(X; \tilde\theta)\tr  \mid Z]$. Now, by \cref{asm:non-degenerate} we know that $V(z; \tilde\theta)^{-1}$ is positive definite, with strictly positive eigenvalues, for almost everywhere $z$. Thus, for almost everywhere $z$, we have that
\begin{equation*}
    \e[\rho(X; \theta) \mid Z=z]\tr  V(z; \tilde\theta)^{-1} \e[\rho(X; \theta) \mid Z=z] \geq 0,
\end{equation*}
with equality if and only if $\e[\rho(X; \theta) \mid Z=z] = 0$. Thus $J(\theta) = 0$ if and only if $\e[\rho(X; \theta) \mid Z] = 0$ almost surely. Finally, by definition $\theta_0$ is the unique parameter value such that $\e[\rho(X; \theta_0) \mid Z] = 0$ almost surely, which gives us our result.

\end{proof}

\begin{lemma}
\label{lem:gaussian-convergence}

Let the assumptions of \cref{thm:asym-norm} be given, and define the random vector $W^{(n)}$, whose $i$'th entry is given by
\begin{equation*}
    W^{(n)}_i = \inner{B_n \frac{\partial}{\partial \theta_i} \bar h_n(\hat\theta_n)}{\sqrt{n} B_n \bar h_n(\theta_0)}\,.
\end{equation*}
Then $W^{(n)} + \theta_0$ is an asymptotically linear and asymptotically normal estimator for $\theta_0$, with covariance matrix $\Delta$, defined as in the statement of \cref{thm:asym-norm}.
\end{lemma}

\begin{proof}[Proof of \Cref{lem:gaussian-convergence}]

First, we note that since $B_n$ is self-adjoint, we have
\begin{equation*}
    W^{(n)}_i = \inner{B_n^2 \frac{\partial}{\partial \theta_i} \bar h_n(\hat\theta_n)}{\sqrt{n} \bar h_n(\theta_0)}\,.
\end{equation*}
Now, by \cref{lem:ht-properties} and \cref{asm:continuous-deriv}, it easily follows that
\begin{equation*}
    \norm{\frac{\partial}{\partial\theta_i} \bar h_n(\hat\theta_n) - \frac{\partial}{\partial\theta_i} \bar h(\theta_0)} = O_p(n^{-1/2})\,.
\end{equation*}
In addition, noting that $B^2 = C^{-1}$, we have
\begin{align*}
    \norm{B^2 \frac{\partial}{\partial \theta_i} \bar h(\theta_0)} &= \sup_{\norm{h}_{\mathcal H} \leq 1} \inner{C^{-1} \frac{\partial}{\partial \theta_i} \bar h(\theta_0)}{h} \\
    &= \sup_{\norm{f}_{\mathcal F} \leq 1} (C^{-1} \frac{\partial}{\partial \theta_i} \bar h(\theta_0))(f) \\
    &= \sup_{\norm{f}_{\mathcal F} \leq 1} \e[f(Z)\tr  V(Z;\tilde\theta)^{-1} \e[\rho'_i(X;\theta_0) \mid Z]]\,,
\end{align*}
where in the above we apply \cref{lem:rkhs-dual,lem:c-inverse}. Now, it easily follows from \cref{lem:f-properties} that $\{f : \norm{f} \leq 1\}$ are uniformly bounded in $L_2$ norm, and in addition it follows from \cref{asm:continuous-deriv} that $\e[\rho'_i(X) \mid Z]$ has finite $L_2$ norm. Then, given \cref{asm:non-degenerate} and the above bound, it trivially follows that $\norm{B^2 (\partial / \partial\theta_i) \bar h(\theta_0)} < \infty$.
Given the above, and \cref{lem:cn-convergence}, it follows from \cref{lem:bn-squared-convergence} that $B_n^2 (\partial / \partial \theta_i) \bar h_n(\hat\theta_n) \to B^2 (\partial / \partial\theta_i) \bar h(\theta_0)$ in probability under $\mathcal H$.

Next, note that $\bar h_n(\theta_0)(f) = \e_n[f(Z)\tr  \e[\rho(X;\theta_0) \mid Z]]$. Now, by \cref{lem:f-donsker} and corollary 9.32 of \citet{kosorok2007introduction} it easily follows that the function class $\mathcal G = \{g : g(z) = f(z)\tr  \e[\rho(X;\theta_0) \mid Z=z], \norm{f} \leq 1\}$ is $\Pcal$-Donsker in the sense of \citet{kosorok2007introduction}, and thus it follows from \cref{lem:uniform-clt} that $\sqrt{n} \bar h_n(\theta_0)$ is stochastically bounded. Therefore, given the previous convergence result, we have
\begin{align*}
    W_i^{(n)} &= \inner{B^2 \frac{\partial}{\partial \theta_i} \bar h(\theta_0)}{\sqrt{n} \bar h_n(\theta_0)} + o_p(1) \\
    &= \frac{1}{\sqrt{n}} \sum_{j=1}^n \inner{q_i}{h(X_j;\theta_0)} + o_p(1)\,,
\end{align*}
where $q_i = B^2 (\partial / \partial \theta_i) \bar h(\theta_0)$. Furthermore, we note that by the definition of $\theta_0$ we have $\e[\inner{q_i}{h(X;\theta_0)}] = 0$, and
\begin{align*}
    \e[\inner{q_i}{h(X;\theta_0)}\inner{q_j}{h(X;\theta_0)}] &= \e[f_i(Z)\tr  \rho(X;\theta_0) f_j(Z)\tr  \rho(X;\theta_0)] \\
    &= \inner{C_0 q_i}{q_j} \\
    &= \inner{(B C_0 B) B \frac{\partial}{\partial \theta_i} \bar h(\theta_0)}{B \frac{\partial}{\partial \theta_j} \bar h(\theta_0)} \\
    &= \Delta_{i,j} \,,
\end{align*}
where $f_i,f_j \in \mathcal F$ are the dual elements corresponding to $q_i,q_j \in \mathcal H$, and $\Delta$ is defined as in the statement of \cref{thm:asym-norm} (since $B = C^{-1/2}$). Therefore, we conclude by the central limit theorem that $W^{(n)} + \theta_0$ is an asymptotically linear and asymptotically normal estimator for $\theta_0$, with covariance matrix $\Delta$, as required.

\end{proof}

\begin{lemma}
\label{lem:omega-value}
    Define the matrices $\Omega$ and $\Omega'$ according to
    \begin{align*}
        \Omega_{i,j} &= \left< B \frac{\partial \bar h(\theta_0)}{\partial \theta_i}, B \frac{\partial \bar h(\theta_0)}{\partial \theta_j} \right> \\
        \Omega'_{i,j} &= \e[\e[\rho'_i(X;\theta_0) \mid Z]\tr  V(Z;\tilde\theta)^{-1} \e[\rho'_j(X;\theta_0) \mid Z]]\,.
    \end{align*}
    Then $\Omega = \Omega'$.
\end{lemma}

\begin{proof}[Proof of \Cref{lem:omega-value}]

First, we note that for any arbitrary vector $\alpha$ we have 
\begin{align*}
    \alpha\tr  \Omega \alpha &= \norm{B \sum_{i=1}^n \alpha_i (\partial / \partial\theta_i) \bar h(\theta_0)}^2 \\
    &= \e[ \e[\sum_{i=1}^n \alpha_i\rho'_i(Z;\theta_0)]\tr  V(Z;\tilde\theta)^{-1} \e[\sum_{i=1}^n \alpha_i\rho'_i(Z;\theta_0)]] \\
    &= \alpha\tr  \Omega' \alpha \,,
\end{align*}
where the second equality follows from \cref{lem:bh-value}. Now given this, it is clear that both $\Omega$ and $\Omega'$ are symmetric and PSD, since clearly $\norm{B \sum_{i=1}^n \alpha_i (\partial / \partial\theta_i) \bar h(\theta_0)}^2 \geq 0$. However given these matrices are symmetric and PSD, having $\alpha\tr  \Omega \alpha = \alpha\tr  \Omega' \alpha$ for every $\alpha$ ensures that $\Omega = \Omega'$, which gives us our required result.

\end{proof}

\begin{lemma}
\label{lem:fn-norm-bound}

Let $f_n^*(\theta) \in \mathcal F$ be the function maximizing $U_n(\theta, f)$ for any given $\theta$. Then under the assumptions of \cref{thm:consistency}, there exists some constant $\eta$ such that
\begin{equation*}
    \sup_{\theta \in \Theta} \norm{f_n^*(\theta)}_{\mathcal F} \leq \eta + o_p(1)\,.
\end{equation*}

\end{lemma}

\begin{proof}

Recall that
\begin{align*}
    U_n(\theta, f) &= \e_n[f(Z)^t \rho(X;\theta)] - \frac{1}{4} \e_n[f(Z)\tr  \rho(X;\tilde\theta_n)] - \frac{\alpha_n}{4} \norm{f}^2 \\
    &= \inner{\bar h_n(\theta)}{h} - \frac{1}{4} \inner{(C + \alpha_n I) h}{h}\,,
\end{align*}
where $h$ is the element of $\mathcal H$ corresponding to $f \in \mathcal F$ by the duality isomorphism. Now, consider the Gateaux derivative of the above at $h_n^*(\theta)$ in the direction of $\epsilon$. This Gateaux derivative is clearly given by
\begin{equation*}
    \inner{\bar h_n(\theta)}{\epsilon} - \frac{1}{2} \inner{(C + \alpha_n I) h_n^*(\theta)}{\epsilon}\,.
\end{equation*}
For $h_n^*(\theta)$ to be the optimizing element, we require this derivative to be zero for every $\epsilon \in \mathcal H$, which immediately gives us
\begin{align*}
    h_n^*(\theta) &= 2(C_n + \alpha_n)^{-1} \bar h_n(\theta) \\
    &= 2 B_n^2 \bar h_n(\theta) \,.
\end{align*}

Next, define $h^*(\theta) = 2 B^2 \bar h(\theta)$. We note that by \cref{asm:regularity}, we have that $\bar h_n(\theta)$ and $\bar h(\theta)$ satisfy the conditions of \cref{lem:ht-properties}. Therefore by \cref{lem:ht-properties} we have $\sup_{\theta \in \Theta} \norm{B^2 \bar h(\theta)} < \infty$. Furthermore, by assumption the conditions of \cref{lem:cn-convergence}, so therefore by \cref{lem:bn-squared-convergence} we have $\sup_{\theta \in \Theta} \norm{B_n^2 \bar h_n(\theta) - B^2 \bar h(\theta)} = o_p(1)$. Putting this together, we have
\begin{align*}
    \sup_{\theta \in \Theta} \norm{f_n^*(\theta)} &= \sup_{\theta \in \Theta} \norm{h_n^*(\theta)} \\
    &\leq \sup_{\theta \in \Theta} \norm{h^*(\theta)} + \sup_{\theta \in \Theta} \norm{h_n^*(\theta) - h^*(\theta)} \\
    &= 2 \sup_{\theta \in \Theta} \norm{B^2 \bar h(\theta)} + 2 \sup_{\theta \in \Theta} \norm{B_n^2 \bar h_n(\theta) - B^2 \bar h(\theta)} \\
    &= \eta + o_p(1)\,,
\end{align*}
where we define $\eta = 2 \sup_{\theta \in \Theta} \norm{B^2 \bar h(\theta)}$, which as argued above is finite.

\end{proof}

\section{Omitted Proofs}
\label{apx:proofs}

\begin{proof}[Proof of \Cref{lem:kvmm-cgmm-equiv}]

First, note that clearly $C_n$ has rank at most $n$, as is therefore compact. Therefore, applying \cref{lem:operator-sqrt-inverse}, we get
\begin{align*}
    \norm{(&C_n + \alpha_n)^{-1/2} \bar h_n(\theta)}^2 \\
    &= \sup_{h' \in \mathcal H} \inner{\bar h_n(\theta)}{h'} - \frac{1}{4} \inner{(C_n + \alpha_n I) h'}{h'} \\
    &= \sup_{h \in \mathcal H} \e_n[\varphi(h)(Z)\tr  \rho(X;\theta)] - \frac{1}{4} \e_n[(\varphi(h)(Z)\tr  \rho(X;\tilde\theta_n))^2]  - \frac{1}{4} \alpha_n \norm{h'}^2 \,,
\end{align*}
where in the above we apply \cref{lem:rkhs-dual}, and $\varphi$ maps any element of $\mathcal H$ to its dual element in $\mathcal F$. The required result then immediately follows given the duality between $\mathcal H$ and $\mathcal F$.

\end{proof}

\begin{proof}[Proof of \Cref{thm:consistency}]

Let $J_n$ and $J$ be defined as in \cref{eq:jn,eq:j}. Now by assumption (LINK) we have $J_n(\hat \theta_n) \leq \inf_{\theta} J_n(\theta) + o_p(1)$, and by \cref{lem:objective-convergence} we have $\sup_{\theta \in \Theta} \abs{J_n(\theta) - J(\theta)} \to 0$ in probability. Therefore we have
\begin{align*}
J(\hat \theta_n) &\leq J_n(\hat \theta_n) + o_p(1) \\
&\leq \inf_{\theta} J_n(\theta) + o_p(1) \\
&\leq J_n(\theta_0) + o_p(1) \\
&\leq J(\theta_0) + o_p(1)\,.
\end{align*}

Next, suppose that $\hat\theta_n$ does not converge in probability to $\theta_0$. This implies that for some $\epsilon > 0$, we have that $P(\norm{\hat\theta_n - \theta_0} < \epsilon)$ does not converge to zero. Let $J^* = \inf_{\theta : \norm{\theta - \theta_0} \geq \epsilon} J(\theta)$. Then it must be the case that $\limsup_{n \to \infty} J(\hat\theta_n) \geq J^*$. Furthermore, by \cref{lem:j-continuous} $J(\theta)$ is continuous in $\theta$, and by \cref{lem:unique-solution} we know that $\theta_0$ is the unique element of $\Theta$ such that $J(\theta_0) = 0$, so we have that $J^* > 0$. However, we have already established that $J(\hat\theta_n) \leq J(\theta_0) + o_p(1)$, so therefore $\limsup_{n \to \infty} J(\hat\theta_n) = 0$. Thus we have a contradiction, so we conclude that $\hat\theta_n \to \theta_0$ in probability.

\end{proof}

\begin{proof}[Proof of \Cref{lem:non-smooth-rho}]

Let $\epsilon_i(X;\theta,\theta') = \rho_i(X;\theta') - \rho_i(X;\theta) - (\theta'-\theta)^\top \rho'_i(X;\theta)$, and let $B(\theta,r)$ denote the ball of radius $r$ about $\theta$. Also, let $E(\theta,\theta')$ denote the event where $\phi(X;\tilde\theta) \neq 0$ for all $\tilde\theta \in B(\theta, \|\theta'-\theta\|)$. Now, under event $E(\theta,\theta')$, by Taylor's theorem we have that $|\epsilon_i(X;\theta,\theta')| = |(\theta'-\theta)^\top \rho''(X;\tilde\theta) (\theta'-\theta)$ for some $\tilde\theta$ on the segment between $\theta$ and $\theta'$. Therefore, applying the boundedness assumption on $\rho''$, we have
\begin{equation*}
    |\epsilon_i(X;\theta,\theta')| \leq c'' \|\theta' - \theta\|^2 \,,
\end{equation*}
under event $E(\theta,\theta')$. Furthermore, regardless of whether this event occurs or not, we can generally bound
\begin{align*}
    |\epsilon_i(X;\theta,\theta')| &\leq |\rho_i(X;\theta') - \rho_i(X;\theta)| + \|\theta'-\theta\| \|\rho'_i(X;\theta)\| \\
    &\leq (L_\rho + c') \|\theta' - \theta\| \,.
\end{align*}
Therefore, we have
\begin{align*}
    \EE[|\epsilon_i(X;\theta,\theta')|] &\leq \textup{Prob}(E(\theta,\theta')) c'' \|\theta' - \theta\|^2  + \textup{Prob}(\neg E(\theta,\theta')) (L_\rho + c' ) \|\theta' - \theta\| \\ 
    &\leq c'' \|\theta' - \theta\|^2  + \textup{Prob}(\neg E(\theta,\theta')) (L_\rho + c' ) \|\theta' - \theta\| \,.
\end{align*}
Therefore, if we can show that $\textup{Prob}(\neg E(\theta,\theta')) \to 0$ as $\theta' \to \theta$, then by the above equation we have $\EE[|\epsilon_i(X;\theta,\theta')|] = o(\|\theta'-\theta\|)$ as $\theta' \to \theta$, which is what we are required to prove. Therefore, it only remains to show this.

Now, given the assumed Lipschitz-continuity of $\phi$, if $|\phi(X;\theta)| \geq L_\phi(X) \|\theta' - \theta\|$, then it must be the case that event $E(\theta,\theta')$ occurs. Therefore, we have
\begin{align*}
    \textup{Prob}(\neg E(\theta,\theta')) &\leq \textup{Prob}\Big(|\phi(X;\theta)| \leq L_\phi(X) \|\theta' - \theta\|\Big) \\
    &= \EE \Big[ \indicator{ L_\phi(X)^{-1} |\phi(X;\theta)| \leq  \|\theta' - \theta\| } \Big] \\
    &= \int_{-\|\theta'-\theta\|}^{\|\theta'-\theta\|} q_{\phi,\theta}(x) d \mu(x) \\
    &\leq 2 \|\theta' - \theta\| \sup_{|x| \leq \|\theta' - \theta\|} q_{\phi,\theta}(x) \,,
\end{align*}
where $q_{\phi,\theta}(x)$ is the probability density of $L_\phi(X)^{-1} \phi(X;\theta)$, and $\mu(x)$ is the standard Borel measure on $\RR$. Now, by assumption, $\sup_{|x| \leq \epsilon} q_{\phi,\theta}(x) < \infty$ for some $\epsilon > 0$. Therefore, whenever $\|\theta' - \theta\| \leq \epsilon$ we have $\textup{Prob}(\neg E(\theta,\theta')) \leq 2 \sup_{|x| \leq \epsilon} q_{\phi,\theta}(x) \|\theta' - \theta\|$, and so $\textup{Prob}(\neg E(\theta,\theta')) \to 0$ as $\theta' \to \theta$. Combining this with the above argument, we can conclude.

\end{proof}

\begin{proof}[Proof of \Cref{thm:asym-norm}]

Let us define
\begin{align*}
    \psi_n(\theta) &= \inner{B_n \bar h_n(\theta)}{B_n h'_n(\theta)} \\
    \psi(\theta) &= \inner{B \bar h(\theta)}{B h'(\theta)} \,.
\end{align*}

Note that by the theorem assumptions that $\|\psi_n(\hat\theta_n)\| = o_p(n^{-1/2})$, and also that $\|\psi(\theta_0)\| = 0$, since $\bar h(\theta_0) = 0$.

Next, we have
\begin{align*}
    &\sqrt{n} \big( \psi_n(\hat\theta_n) - \psi(\hat\theta_n) \big) - \sqrt{n} \big(\psi_n(\theta_0) - \psi(\theta_0) \big) \\
    &= \inner{\sqrt{n}(\bar h_n(\hat\theta_n) - \bar h(\hat\theta_n))}{B_n^2(h'_n(\hat\theta_n) - h'_n(\theta_0))} \\
    &\qquad + \inner{\sqrt{n}(\bar h_n(\hat\theta_n) - \bar h(\hat\theta_n)) - \sqrt{n}(\bar h_n(\theta_0) - \bar h(\theta_0))}{B_n^2 h'_n(\theta_0)} \\
    &\qquad + \inner{\sqrt{n}(\bar h(\hat\theta_n) - \bar h(\theta_0))}{B_n^2 h'_n(\hat\theta_n) - B^2 h'(\hat\theta_n)} \,.
\end{align*}
Now, by \cref{lem:ht-properties}, we have $\|\sqrt{n}(\bar h_n(\hat\theta_n) - \bar h(\hat\theta_n))\| = O_p(1)$,  $\|B_n^2 h'_n(\theta_0)\| = O_p(1)$, and $\|\sqrt{n}(\bar h(\hat\theta_n) - \bar h(\theta_0))\| = O_p(\sqrt{n} \|\hat\theta_n-\theta_0\|)$. Furthermore, by \cref{lem:ht-properties} and \cref{lem:bn-squared-convergence} we have $\|B_n^2(h'_n(\hat\theta_n) - h'_n(\theta_0))\| = O_p(\|\hat\theta_n-\theta_0\|) + o_p(1)$, and $\|B_n^2 h'_n(\hat\theta_n) - B^2 h'(\hat\theta_n)\| = o_p(1)$. Also, applying \cref{asm:rho-deriv}, we have
\begin{align*}
    &\sqrt{n} \| \bar h_n(\hat\theta_n) - \bar h(\hat\theta_n) - \bar h_n(\theta_0) - \bar h(\theta_0) \| \\
    &= \sqrt{n} \sup_{\|f\|_K \leq 1} (\EE_n - \EE)[f(Z)^\top(\rho(X;\hat\theta_n) - \rho(X;\theta_0))] \\
    &\leq \sqrt{n} \sup_{\|f\|_K \leq 1} (\EE_n - \EE)[f(Z)^\top\big(\rho(X;\hat\theta_n) - \rho(X;\theta_0) - D(X;\theta_0) (\hat\theta_n-\theta_0) \big)] \\
    &\qquad + \sqrt{n} \sup_{\|f\|_K \leq 1} (\EE_n - \EE)[f(Z)^\top D(X;\theta_0) (\hat\theta_n-\theta_0)] \,.
\end{align*}
Since $\sup_{\|f\|_K \leq 1} \|f\|_\infty < \infty$ and \cref{asm:rho-deriv}, the first term above is clearly $o_p(\sqrt{n} \|\hat\theta_n - \theta_0\|$. Furthermore, by \cref{asm:continuous-deriv} and \cref{lem:f-donsker}, and the Donsker preservation property of corollary 9.32 of \citet{kosorok2007introduction}, we have that $\{f(Z)^\top D(X;\theta_0)_j\}$ is $\Pcal$-Donsker for each $j \in [b]$, and therefore it easily follows that the second term above is $O_p(\|\hat\theta_n - \theta_0\|) = o_p(\sqrt{n} \|\hat\theta_n - \theta_0\|)$. Therefore, putting the above together, we have
\begin{align*}
    &\|\sqrt{n} \big( \psi_n(\hat\theta_n) - \psi(\hat\theta_n) \big) - \sqrt{n} \big(\psi_n(\theta_0) - \psi(\theta_0) \big) \| \\
    &= O_p(1)\left(O_p(\|\hat\theta_n-\theta_0\|) + o_p(1)\right) \\
    &\qquad\qquad + o_p(\sqrt{n}\|\hat\theta_n-\theta_0\|) O_p(1) + O_p(\sqrt{n}\|\hat\theta_n-\theta_0\|) o_p(1) \\
    &= o_p(1 + \sqrt{n} \|\hat\theta_n-\theta_0\|) \,.
\end{align*}

Next, let $\Omega$ be defined according to
\begin{equation*}
    \Omega_{i,j} = \inner{B h'(\theta)_i}{B h'(\theta)_j} \,,
\end{equation*}
and note that by \cref{lem:omega-value} this definition is identical to the definition of $\Omega$ in the theorem statement. Then, for any $\delta \in \RR^b$ we have
\begin{align*}
    &\psi(\theta_0 + \delta) - \psi(\theta_0) - \Omega \delta \\
    &= \inner{B \bar h(\theta_0 + \delta) - B \bar h(\theta_0)}{B h'(\theta_0 + \delta) - B h'(\theta_0)} \\
    &\qquad + \inner{\bar h(\theta_0 + \delta) - \bar h(\theta_0) - \delta^\top h'(\theta_0)}{B^2 h'(\theta_0)} \,.
\end{align*}
Applying \cref{lem:f-properties} and \cref{asm:rho-deriv}, similar to above, we have $\|\bar h(\theta_0 + \delta) - \bar h(\theta_0) - \delta^\top h'(\theta_0)\| = o(\|\delta\|)$, and also by \cref{lem:ht-properties} we have $\|B \bar h(\theta_0 + \delta) - B \bar h(\theta_0)\|, \|B h'(\theta_0 + \delta) - B h'(\theta_0)\| = O(\|\delta\|)$, and $\|B^2h'(\theta_0)\| = O(1)$. Therefore, $\psi(\theta_0 + \delta) - \psi(\theta_0) - \Omega \delta = o(\|\delta\|)$, and thus $\psi$ is Fréchet differentiable at $\theta_0$, with derivative $\Omega$. Furthermore, it trivially follows from \cref{asm:non-degenerate-theta} that $\Omega$ is strictly positive definite, and therefore is invertible.

Therefore, we have satisfied the conditions of theorem 2.11 of \citet{kosorok2007introduction}, and therefore we have that $\sqrt{n}(\hat\theta_n - \theta_0)$ has the same limit in law as $-\Omega^{-1} \sqrt{n} \psi_n(\theta_0)$. Therefore, our result immediately follows from \cref{lem:gaussian-convergence}.

\end{proof}

\begin{proof}[Proof of \Cref{lem:first-order-conditions}]

First, note that
\begin{align*}
    J_n(\theta') - J_n(\theta) &= \inner{B_n (\bar h_n(\theta') - \bar h_n(\theta))}{B_n (\bar h_n(\theta) + \bar h_n(\theta'))}  \\
    &= \inner{B_n (\theta'-\theta)^\top h'_n(\theta))}{B_n (\bar h_n(\theta) + \bar h_n(\theta'))}  \\
    &\qquad + \inner{(\bar h_n(\theta') - \bar h_n(\theta) - (\theta'-\theta)^\top h'_n(\theta))}{B^2_n (\bar h_n(\theta) + \bar h_n(\theta'))} \,.
\end{align*}
Now, we have
\begin{align*}
    &\|\bar h_n(\theta') - \bar h_n(\theta) - (\theta'-\theta) h'_n(\theta)\| \\
    &= \sup_{\|f\|_K \leq 1} \EE_n[f(Z)\top (\rho(X;\theta') - \rho(X;\theta) - D(X;\theta) (\theta'-\theta))] \\
    &\leq k_F \sum_{i=1}^m \EE_n \left[ \Big| \rho_i(X;\theta') - \rho_i(X;\theta) - (\theta'-\theta)^\top D_i(X;\theta) \Big| \right] \,.
\end{align*}
Therefore, it easily follows from \cref{asm:rho-deriv} that $\|\theta'-\theta\|^{-1} \|\bar h_n(\theta') - \bar h_n(\theta) - (\theta'-\theta) h'_n(\theta)\| \to 0$ in probability as $\theta' \to \theta$. And so, also applying \cref{lem:ht-properties} to the above, we have
\begin{equation*}
    J_n(\theta') = J_n(\theta) + (\theta'-\theta)^\top J'_n(\theta) + o_p(\|\theta'-\theta\|) \,.
\end{equation*}

Now, let $\theta_n^*$ be any minimizer of $J_n$, and suppose that $J_n'(\theta^*_n) \neq 0$. Define $\theta(\epsilon) = \theta_n^* - \epsilon J_n'(\theta_n^*)$. Then, $J_n(\theta(\epsilon)) = J_n(\theta_n^*) - \epsilon \|J_n'(\theta_n^*)\|_2 + o_p(\epsilon)$ as $\epsilon \to 0$. This would imply that $J_n(\theta(\epsilon)) < J_n(\theta^*_n)$ for some $\epsilon>0$, which contradicts that $\theta_n^*$ is a minimizer of $J_n$. Therefore, we have that $J'_n(\hat\theta_n) = 0$, from which case (a) of the lemma statement follows immediately.

For case (b), we note that under the assumption that $\rho$ is twice continuously differentiable in $\theta$, it easily follows that $J_n$ is twice differentiable in $\theta$ (since compact $\Theta$ implies that the continuous derivatives of $\rho$ are bounded, and therefore $\|\nabla \bar h_n(\theta)_i\| = \sup_{\|f\|_K \leq 1} \EE_n[f(Z)^\top \nabla \rho(X;\theta)_i]$ and $\|\nabla^2 \bar h_n(\theta)_{i,j}\| = \sup_{\|f\|_K \leq 1} \EE_n[f(Z)^\top \nabla^2 \rho(X;\theta)_{i,j}]$ must have finite.) Furthermore, we have
\begin{equation*}
    \nabla^2 J_n(\theta) = \inner{B_n \nabla^2 h_n(\theta)}{h_n(\theta)} + \inner{B_n \nabla h_n(\theta)}{\nabla h_n(\theta)^\top}.
\end{equation*}
Then, applying arguments as in \cref{lem:ht-properties}, \cref{lem:bn-convergence}, and \cref{lem:bn-squared-convergence}, it easily follows that $\nabla^2 J_n(\theta_0) \to \nabla^2 J(\theta_0) = \inner{B \nabla \bar h(\theta_0)}{B \nabla \bar h(\theta_0)^\top}$ in probability, and therefore further applying \cref{thm:consistency} we have that $\nabla^2 J_n(\theta_n^*) \to \nabla^2 J(\theta_0)$ in probability.
Thus, for sufficiently large $n$ it must be the case that $J_n$ is positive-definite. 

Next, we have $J_n(\hat\theta_n) = J_n(\theta_n^*) + (\hat\theta_n - \theta_n^*)^\top \nabla^2 J_n(\hat\theta_n^*) (\hat\theta_n - \theta_n^*) + o_p(\|\hat\theta_n - \theta_n^*\|^2)$, since as argued above $\nabla J_n(\theta_n^*) = 0$, and therefore $J_n(\hat\theta_n) - J_n(\theta_n^*) = o_p(1/n)$ implies that $\|\hat\theta_n - \theta_n^*\| = o_p(n^{-1/2})$. Finally, this gives us $J'_n(\hat\theta_n) = J'_n(\theta_n^*) + \nabla^2 J_n(\theta_n^*) (\hat\theta_n - \theta_n^*) + o_p(\|\hat\theta_n - \theta_n^*\|) = o_p(n^{-1/2})$, as required.

\end{proof}

\begin{proof}[Proof of \Cref{thm:efficiency}]

For the first part of the theorem, we note that in the case that $\tilde\theta = \theta_0$, we clearly have
\begin{equation*}
    \Delta_{i,j} = \inner{B \frac{\partial}{\partial \theta_i} \bar h(\theta_0)}{B \frac{\partial}{\partial \theta_j} \bar h(\theta_0)}\,,
\end{equation*}
since in this case we have $C = C_0$. Next, by \cref{lem:omega-value} it is clear that $\Delta = \Omega_0$, and since $\tilde\theta = \theta_0$ we have $\Omega = \Omega_0$. Therefore we have that the asymptotic variance of $\sqrt{n}(\hat\theta_n - \theta_0)$ is given by $\Omega_0^{-1}$, as required.

For the second part of the theorem, we first note that this limiting covariance matches the efficiency bound of \citet{chamberlain1987asymptotic}. Furthermore, by \cref{thm:asym-norm} we know that the estimator is asymptotically linear, so by theorem 18.7 of \citet{kosorok2007introduction} we conclucde that it is also regular, and efficient relative to all possible regular asymptotically linear estimators.

Finally, in order to argue that this estimator is efficient within the class of VMM estimators with different values of $\tilde\theta$, we need to argue that $\Omega^{-1} \Delta \Omega^{-1} - \Omega_0^{-1}$ is PSD for any $\tilde\theta$. This is equivalent to showing that $\alpha\tr  \Delta \alpha - \alpha\tr  \Omega \Omega_0^{-1} \Omega \alpha \geq 0$ for every vector $\alpha$. Given the definitions of $\Delta$, $\Omega$, and $\Omega_0$, and defining
\begin{equation*}
    q_{\alpha} = \sum_{i=1}^m \alpha_i B \frac{\partial}{\partial \theta_j} \bar h(\theta_0)\,,
\end{equation*}
we have
\begin{align*}
    \alpha\tr  \Delta \alpha - \alpha\tr  \Omega \Omega_0^{-1} \Omega \alpha &= \inner{(B C_0 B) q_{\alpha}}{q_{\alpha}} - \inner{q_\alpha}{B \nabla \bar h(\theta_0)}\tr  \Omega_0^{-1} \inner{B \nabla {q_\alpha} \bar h(\theta_0)}{q_\alpha} \\
    &= \inner{(C_0 B q_{\alpha}}{B q_{\alpha}} - \inner{B q_\alpha}{\nabla \bar h(\theta_0)}\tr  \Omega_0^{-1} \inner{\nabla {q_\alpha} \bar h(\theta_0)}{B q_\alpha} \\
    &= \inner{(W B q_{\alpha}}{B q_{\alpha}} \,,
\end{align*}
where the linear operator $W$ satisfies
\begin{equation*}
    W h = C_0 h - \nabla \bar h(\theta_0)\tr  \Omega_0^{-1} \inner{\nabla \bar h(\theta_0)}{h}\,.
\end{equation*}
Now, recall that $\Omega_0 = \inner{C_0^{-1/2} \nabla h(\theta_0)}{C_0^{-1/2} \nabla h(\theta_0)\tr }$. Given this, for any $h$, we can observe that
\begin{align*}
    &(W C_0^{-1} W)(h) \\
    &= W (h - C_0^{-1} \nabla \bar h(\theta_0)\tr  \Omega_0^{-1} \inner{\nabla \bar h(\theta_0)}{h}) \\
    &= W h - (C_0 - W h) + \nabla \bar h(\theta_0)\tr  \Omega_0^{-1} \inner{\nabla \bar h(\theta_0)}{C_0^{-1} \nabla \bar h(\theta_0)\tr } \Omega_0^{-1} \inner{\nabla \bar h(\theta_0)}{h} \\
    &= 2 W h - (C_0 - \nabla \bar h(\theta_0)\tr  \Omega_0^{-1} \inner{\nabla \bar h(\theta_0)}{h}) \\
    &= W h \,.
\end{align*}

Therefore, we have $W = W C_0^{-1} W$. Furthermore, it is trivial to wee that $W$ is self-adjoint, since $C_0$ is self-adjoint and $\Omega_0^{-1}$ is symmetric. Plugging these results into the above, we get
\begin{align*}
    \alpha\tr  \Delta \alpha - \alpha\tr  \Omega \Omega_0^{-1} \Omega \alpha &= \inner{W C_0^{-1} W B q_\alpha}{B q_\alpha} \\
    &= \norm{C_0^{-1/2} W B q_\alpha}^2 \\
    &\geq 0\,,
\end{align*}
which concludes the proof.

\end{proof}

\begin{proof}[Proof of \Cref{lem:k-step-vmm}]

We will prove this by induction. First, in the case that $k = 1$, $\tilde\theta_n$ is given by a VMM estimate whose prior estimate is a constant. However, a constant prior estimate satisfies \cref{asm:tilde-theta} trivially, and by assumption the remaining conditions of \cref{thm:asym-norm} hold, so therefore $\tilde\theta_n = \theta_0 + O_p(n^{-1/2})$. Given this and \cref{asm:continuous-deriv}, by the mean value theorem we have
\begin{equation*}
    \rho_i(x; \tilde\theta_n) = \rho_i(x; \theta_0) + \nabla \rho_i(x; \theta'_n)\tr  (\tilde\theta_n - \theta_0)\,,
\end{equation*}
for each $i \in [m]$ and $x \in \mathcal X$, where $\theta'_n$ lies on the segment between $\theta_0$ and $\tilde\theta_n$ for each $i$ and $x$. Now, by \cref{asm:continuous-deriv} $\nabla\rho_i(x;\theta)$ is uniformly bounded over $x$ and $\theta$, so therefore since $\tilde\theta_n = O_p(n^{-1/2})$ it follows that
\begin{equation*}
    \sup_{x \in \mathcal X, i \in [m]} \abs{\rho_i(x; \tilde\theta_n) - \rho_i(x; \theta_0)} = O_p(n^{-1/2})\,,
\end{equation*}
so therefore we have established the base case.

For the inductive case, assume that the desired property holds for some $k > 1$. Then, by the inductive hypothesis, $\tilde\theta_n$ is given by a VMM estimate whose prior estimate whose prior estimate satisfies \cref{asm:tilde-theta} with $p=1/2$. Furthermore, by assumption the remaining conditions of \cref{thm:asym-norm} hold, so therefore we have $\tilde\theta_n = \theta_0 + O_p(n^{-1/2})$. Therefore by the same mean value theorem-based argument as above, again we have $\sup_{x \in \mathcal X, i \in [m]} \abs{\rho_i(x; \tilde\theta_n) - \rho_i(x; \theta_0)} = O_p(n^{-1/2})$, which establishes the inductive case.

\end{proof}

\begin{proof}[Proof of \Cref{lem:kernel-vmm-closed-form}]

Recall that 
\begin{equation*}
    J_n(\theta) = \sup_{f \in \mathcal F} \e_n[f(Z)\tr  \rho(X;\theta)] - \frac{1}{4} \e_n[(f(Z)\tr  \rho(X;\tilde\theta_n))^2] - \frac{\alpha_n}{4} \norm{f}^2\,.
\end{equation*}
Now, by the representer theorem, there must exist an optimal solution to the supremum over $f \in \mathcal F$ that takes the form $f_k(z) = \sum_{i=1}^n \beta_{i,k} K_k(Z_i, z)$ for each $k \in [m]$ and $z \in \mathcal Z$, for some vector $\beta \in \mathbb R^{n \times m}$. Plugging in this to the above, we get
\begin{align*}
    J_n(\theta) &= \sup_{\beta \in \mathbb R^{n \times m}} \sum_{i,k} \beta_{i,k} \frac{1}{n} \sum_{j=1}^n K_k(Z_i, Z_j) \rho_k(X_j;\theta) \\
    &\qquad - \frac{1}{4} \sum_{i,k,i',k'} \beta_{i,k} \beta_{i',k'} \frac{1}{n} \sum_{j=1}^n K_k(Z_i, Z_j) \rho_k(X_j;\tilde\theta_n) K_{k'}(Z_{i'}, Z_j) \rho_{k'}(X_j;\tilde\theta_n) \\
    &\qquad - \frac{\alpha_n}{4} \sum_{i,i',k} \beta_{i,k} \beta_{i',k} K_k(Z_i, Z_{i'}) \\
    &= \sup_{\beta \in \mathbb R^{n \times m}} \frac{1}{n} \beta\tr  L \rho(\theta) - \frac{1}{4} \beta\tr  (Q(\tilde\theta_n) + \alpha_n L) \beta \,.
\end{align*}

Next, taking derivatives of the above with respect to $\beta$, and noting that $Q(\tilde\theta_n) + \alpha_n L$ is clearly positive semi-definite, the above above supremum is clearly attained when
\begin{align*}
    &\frac{1}{n} L \rho(\theta) - \frac{1}{2} (Q(\tilde\theta_n) + \alpha_n L) \beta = 0 \\
    &\qquad\iff \beta = \frac{2}{n} (Q(\tilde\theta_n) + \alpha_n L)^{-1} L \rho(\theta)\,.
\end{align*}

Plugging this optimal $\beta$ into the above supremum, and noting that $L$ is symmetric, gives us
\begin{equation*}
    J_n(\theta) = \frac{1}{n^2} \rho(\theta)\tr  L (Q(\tilde\theta_n) + \alpha_n L)^{-1} L \rho(\theta)\,,
\end{equation*}
as required.

\end{proof}

\begin{proof}[Proof of \Cref{lem:kernel-vmm-closed-form-iv}]

First, applying \cref{lem:kernel-vmm-closed-form} to this particular problem, we get
\begin{align*}
    J_n(\theta) &= (W - \theta(T))\tr  M (W - \theta(T)) \\
    &= \theta(T)\tr  M \theta(T) - 2 W\tr  M \theta(T) + c \,,
\end{align*}
where $\theta(T) = (\theta(T_1), \ldots, \theta(T_n))$, and $c$ is a constant that doesn't depend on $\theta$. Now, by Lagrangian duality, there must exist some value $\lambda_n \geq 0$ which implicitly depends on $r$, $K_f$, $K_g$, $\tilde\theta_n$, and the observed data, such that minimizing $J_n(\theta)$ over $\Theta$ is equivalent to minimizing $J_n(\theta) + \lambda_n \norm{\theta}^2$ over the entire RKHS given by $K_g$. Now, by the representer theorem, there must exist some optimal solution to this unconstrained problem that takes the form $\theta(t) = \sum_{i=1}^n \beta_i K_g(t, T_i)$ for every treatment value $t$. Plugging this in to the regularized and unconstrained version of our problem, and noting that $\theta(T) = L_g \beta$, we have that the optimal vector $\beta^*$ satisfies
\begin{align*}
    \beta^* &= \argmin_{\beta \in \mathbb R^n} (L_g \beta)\tr  M (L_g \beta) - 2 W\tr  M L_g \beta + c + \lambda_n \beta\tr  L_g \beta \\
    &= \argmin_{\beta \in \mathbb R^n} \beta\tr  (L_g M L_g + \lambda_n L_g) \beta - 2 (L_g M W)\tr  \beta \,.
\end{align*}

Taking derivatives with respect to $\beta$ in the above, and noting that $L_g M L_g + \lambda_n L_g$ is clearly positive semi-definite, it follows that this supremum is attained when
\begin{align*}
    &2 (L_g M L_g + \lambda_n L_g) \beta^* - 2 L_g M W = 0 \\
    &\qquad \iff \beta^* = (L_g M L_g + \lambda_n L_g)^{-1} L_g M W\,,
\end{align*}
which exactly matches our desired result.

\end{proof}

\begin{proof}[Proof of \Cref{lem:uniform-jn-approximation}]

For any given $\theta \in \Theta$, let $f_n^*(\theta)$ be the element of $\mathcal F$ maximizing the adversarial game objective $U_n(\theta,f)$, and for each $i \in [m]$ define $F_{i,n} = \sup_{\theta \in \Theta} \norm{f^*_n(\theta)_i}_\mathcal F$. Now, by \cref{lem:rkhs-dual} and the definitions of $\bar h_n(\theta)$ and $B_n$ we clearly have
\begin{align*}
    &\inner{\bar h_n(\theta)}{h'} - \frac{1}{4} \inner{B_n^{-2} h'}{h'} \\
    &\qquad = \e_n[\varphi(f)(Z)\tr  \rho(X;\theta)] - \frac{1}{4} \e_n[(\varphi(h')(Z)\tr  \rho(X;\tilde\theta_n))^2] - \frac{\alpha_n}{4} \norm{\varphi(h')}_{\mathcal F}\,,
\end{align*}
so it easily follows that $f_n^*(\theta) = 2 B_n^2 \bar h_n(\theta)$. Now, by \cref{asm:regularity} the functions $\bar h_n(\theta)$ and $\bar h(\theta)$ satisfy the conditions of \cref{lem:ht-properties}, so by \cref{lem:ht-properties,lem:bn-squared-convergence} we have that $F_{i,n} = O_p(1)$ for each $i \in [m]$.

Next, for any integer $k > d_z/2$, where $d_z$ is the dimensionality of $\mathcal Z$, let $W_k$ denote the Sobolev space of functions of the form $\mathcal Z \mapsto \mathbb R$ using derivatives up to order $k$, and let $\norm{\cdot}_{W_k}$ denote the corresponding Sobolev norm. Then, as argued in the proof of theorem D of \citet{cucker2002mathematical}, there exists some constant $c_k$, which may depend on $k$, such that $\norm{f}_{W_k} \leq c_k \norm{f}_{\mathcal F}$ for any function $f$. That is, $\sup_{\theta \in \Theta} \norm{f_n^*(\theta)_i}_{W_k} \leq c_k F_{i,n}$.

Now, by theorem 1 \citet{yarotsky2017error}, there exists a constant $c'_k$ which depends on $k$, such that for any $\epsilon > 0$ and any $f$ satisfying $\norm{f}_{W_k} \leq 1$, there exists some ReLU network $f'$ with depth at most $c_k(\log(1/\epsilon) + 1)$, and at most $c_k \epsilon^{-d_z/k} (\log(1/\epsilon) + 1)$ computation units in total, such that $\norm{f - f'}_\infty \leq \epsilon$. It immediately follows that this function $f'$ is guaranteed to exist within a fully connected ReLU architecture with depth $c_k(\log(1/\epsilon) + 1)$ and constant width $c_k \epsilon^{-d_z/k} (\log(1/\epsilon) + 1)$. Putting this result together with our previous results above, and noting that an $m$-dimensional function defined by $m$ separate networks with width $W$ and depth $D$ can be defined within a fixed architecture of width $m W$ and depth $D$, it immediately follows that as long as $D_n = \omega(\log(1/\epsilon) + 1)$ and $W_n = \omega(\epsilon^{-d_z/k} (\log(1/\epsilon) + 1))$, we can ensure that $\sup_{\theta \in \Theta} \inf_{f \in \mathcal F_n} \norm{f^*_n(\theta) - f}_\infty \leq c_k \max_{i \in [m]} F_{i,n} \epsilon$ for some sufficiently large $n$ that doesn't depend on $\epsilon$. Therefore, noting that the choice of $k > d_z/2$ was arbitrary, if we choose $D_n$ and $W_n$ as in \cref{asm:network-size}, we have $\sup_{\theta \in \Theta} \inf_{f \in \mathcal F_n} \norm{f^*_n(\theta) - f}_\infty = O_p(r_n)$.

Finally, for each $\theta \in \Theta$, let $f'_n(\theta) = \argmin_{f \in \mathcal F_n} \norm{f_n^*(\theta) - f}_\infty$, and let $U_n$ be the game objective for $\hat\theta_n^{\text{NK-VMM}}$. Then we have
\begin{align*}
    \sup_{\theta \in \Theta} \abs{J_n(\theta) - J_n^*(\theta)} &\leq \sup_{\theta \in \Theta} \abs{U_n(\theta, f_n^*(\theta)) - U_n(\theta, f_n'(\theta))} \\
    &\leq \sup_{\theta \in \Theta} \abs{\e_n[(f_n^*(\theta) - f_n'(\theta))(Z)\tr  \rho(X;\theta)]} \\
    &\qquad + \sup_{\theta \in \Theta} \e_n[(f_n^*(\theta) - f_n'(\theta))(Z)\tr  \rho(X;\tilde\theta_n) \rho(X;\tilde\theta_n)\tr  f_n^*(\theta)] \\
    &\qquad + \sup_{\theta \in \Theta} \e_n[f_n'(\theta)(Z)\tr  \rho(X;\tilde\theta_n) \rho(X;\tilde\theta_n)\tr  (f_n^*(\theta) - f_n'(\theta))(Z)] \\
    &\qquad + \sup_{\theta \in \Theta} \sum_{k=1}^m \abs{((f_n^*)_k - (f_n')_k)\tr  K_k^{-1} (f_n^*)_k} \\
    &\qquad + \sup_{\theta \in \Theta} \sum_{k=1}^m \abs{(f_n')_k\tr  K_k^{-1} ((f_n^*)_k - (f_n')_k)}\,.
\end{align*}
Now, by \cref{asm:network-size} we have $r_n = o(n^{-1-q} k_n^{-1})$, so therefore by our above result we have $\sup_{\theta \in \Theta} \norm{f_n^*(\theta) - f_n'(\theta)}_\infty = o_p(n^{-1-q} k_n^{-1})$, and $\sup_{\theta \in \Theta} \norm{f_n^*(\theta) - f_n'(\theta)}_2 = o_p(n^{-q-1/2} k_n^{-1})$. In addition, $\sup_{\theta \in \Theta} \norm{f_n^*(\theta)}_\infty$ and $\sup_{\theta \in \Theta} \norm{f_n'(\theta)}_\infty$ are stochastically bounded given the above and applying \cref{lem:f-properties}, and therefore $\sup_{\theta \in \Theta} \norm{f_n^*(\theta)}_2$ and $\sup_{\theta \in \Theta} \norm{f_n'(\theta)}_2$ are both in $O_p(\sqrt{n})$, and also by \cref{asm:k-growth} $K_k^{-1}$ has $O_p(k_n)$ operator norm for each $k \in [m]$. Therefore, plugging these results into the above bound, we get $\sup_{\theta \in \Theta} \abs{J_n(\theta) - J_n^*(\theta)} = O_p(n^{-q})$, as required.

\end{proof}

\begin{proof}[Proof of \Cref{lem:asymptotic-variance}]

We first note that
\begin{equation*}
    \beta\tr \Omega_0^{-1} \beta = \sup_{\gamma \in \mathbb R^b} \gamma\tr \beta - \frac{1}{4} \gamma\tr \Omega_0 \gamma
\end{equation*}
follows immediately from \cref{lem:operator-sqrt-inverse}, with $\mathcal H = \mathbb R^b$, $h = \beta$, $C = \Omega_0$, and $\alpha = 0$ since all linear operators in $\mathbb R^b$ are compact. This immediately gives us our first desired result.

For the second part of the lemma, consider the term $\gamma\tr \Omega_0 \gamma$. By the definition of $\Omega_0$ in \cref{thm:efficiency}, we have
\begin{equation*}
    \gamma\tr \Omega_0 \gamma = \e[q_\gamma(Z)\tr V(Z;\theta_0) q_\gamma(Z)] \\
\end{equation*}
where
\begin{equation*}
    q_\gamma(Z) = \e[\nabla \rho(X;\theta_0) \gamma \mid Z]\,.
\end{equation*}

Next, applying \cref{lem:bh-value} to the RHS of the above, we obtain
\begin{equation*}
    \gamma\tr \Omega_0 \gamma = \norm{B h_\gamma}^2 \,,
\end{equation*}
where $h_\gamma \in \mathcal H$ is defined according to
\begin{equation*}
    h_\gamma(f) = \e[f(Z)\tr \nabla \rho(X;\theta_0) \gamma]\,.
\end{equation*}
Recalling that $B = C^{-1/2}$, and that $C$ is compact from \cref{lem:c-compact}, applying \cref{lem:operator-sqrt-inverse} again gives us
\begin{equation*}
    \gamma\tr \Omega_0 \gamma = \sup_{h \in \mathcal H} \inner{h_\gamma}{h} - \frac{1}{4} \inner{C h}{h}\,,
\end{equation*}
and then by \cref{lem:rkhs-dual} and the definitions of $h_\gamma$ and $C$ we get
\begin{equation*}
    \gamma\tr \Omega_0 \gamma = \sup_{f \in \mathcal F} \e[f(Z)\tr \nabla \rho(X;\theta_0) \gamma] - \frac{1}{4} \e[(f(Z)\tr \rho(X;\theta_0))^2]\,.
\end{equation*}
Finally, plugging this into the right hand side of our equation above for $\beta\tr \Omega_0^{-1} \beta$, and taking out a $-1/4$ term, gives us our final result.

\end{proof}

\begin{proof}[Proof of \Cref{thm:kernel-inference}]

We first observe that, in order to establish that $\Omega_n \to \Omega_0$ in probability, it is sufficient to show that $\sup_{\norm{\gamma} \leq 1} \gamma\tr (\Omega_n - \Omega_0) \gamma \to 0$ in probability, since convergence of a matrix in spectral norm implies element-wise convergence. Therefore, we will instead argue this.

Next, we note that as argued in the proof of \cref{lem:asymptotic-variance}, we have
\begin{equation*}
    \gamma\tr \Omega_0 \gamma = \norm{B h(\gamma;\theta_0)}^2\,,
\end{equation*}
where $h(\gamma;\theta) \in \mathcal H$ is defined according to
\begin{equation*}
    h(\gamma;\theta)(f) = \e[f(Z)\tr \nabla \rho(X;\theta) \gamma]\,.
\end{equation*}

Furthermore, given the form of $\Omega_n$ and applying an almost identical derivation as in the proof of \cref{lem:kernel-vmm-closed-form}, we have
\begin{equation*}
    \gamma\tr \Omega_n \gamma = \norm{B_n h_n(\gamma;\hat\theta_n)}^2\,,
\end{equation*}
where $h_n(\gamma;\theta) \in \mathcal H$ is defined according to
\begin{equation*}
    h_n(\gamma;\theta)(f) = \e_n[f(Z)\tr \nabla \rho(X;\theta) \gamma]\,.
\end{equation*}

Now, we can bound
\begin{align*}
    &\sup_{\norm{\gamma} \leq 1} \gamma\tr (\Omega_n - \Omega_0) \gamma \\
    &\leq \sup_{\norm{\gamma} \leq 1} \left( \norm{B_n h_n(\gamma;\tilde\theta_n)} + \norm{B h(\gamma;\theta_0)} \right) \left| \norm{B_n h_n(\gamma;\tilde\theta_n)} - \norm{B h(\gamma;\theta_0)} \right| \\
    &\leq \left( 2 \sup_{\norm{\gamma} \leq 1} \norm{B h(\gamma;\theta_0)} + \epsilon_n \right) \epsilon_n \,,
\end{align*}
where
\begin{align*}
    \epsilon_n &= \sup_{\norm{\gamma} \leq 1} \left| \norm{B_n h_n(\gamma;\hat\theta_n)} - \norm{B h(\gamma;\theta_0)} \right| \\
    &\leq \sup_{\theta \in \Theta, \norm{\gamma} \leq 1} \Big| \norm{B_n h_n(\gamma;\theta)} - \norm{B h(\gamma;\theta)} \Big| + \sup_{\norm{\gamma} \leq 1} \left| \norm{B h(\gamma;\hat\theta_n)} - \norm{B h(\gamma;\theta_0)} \right| \\
    &\leq \sup_{\theta \in \Theta, \norm{\gamma} \leq 1} \norm{B_n h_n(\gamma;\theta) - B h(\gamma;\theta)} + \sup_{\norm{\gamma} \leq 1} \left| \norm{B h(\gamma;\hat\theta_n)} - \norm{B h(\gamma;\theta_0)} \right| \,.
\end{align*}

Given \cref{asm:continuous-deriv}, it trivially follows that the collection of functions $q(x;\gamma,\theta) = \nabla \rho(x;\theta) \gamma$ satisfy the assumptions of \cref{lem:ht-properties} under any product norm on $\gamma$ and $\theta$ given by $\norm{(\gamma,\theta) - (\gamma',\theta')} = (\norm{\gamma - \gamma'}^2 + \norm{\theta - \theta'}^2)^{1/2}$, so given the conditions on $\alpha_n$ assumed by \cref{thm:efficiency} it follows from \cref{lem:bn-convergence} that the first term in above bound on $\epsilon_n$ converges to zero in probability. In addition, by \cref{lem:ht-properties} we have that the second term in this bound on $\epsilon_n$ is bounded by $L' \norm{\hat\theta_n - \theta_0}$, for some $L'$ that doesn't depend on $\hat\theta_n$ or $\theta_0$. Therefore, this second term also converges in probability to zero, since by \cref{thm:consistency} $\hat\theta_n$ is consistent for $\theta_0$. Thus $\epsilon_n \to 0$ in probability.

Finally, again by \cref{lem:ht-properties} we have that $\sup_{\norm{\gamma} \leq 1} \norm{B h(\gamma;\theta_0)}$ is finite. Putting all of the above together, we have $\sup_{\norm{\gamma} \leq 1} \gamma\tr (\Omega_n - \Omega_0) \gamma \to 0$ in probability, which gives us our desired result.

\end{proof}

\section{Additional Experiment Details}

Here we provide some additional experimental details, including some details of our methods and scenarios not included in \cref{sec:experiments} for the sake of brevity, as well as some additional results.

\subsection{Additional Details of Scenarios}

\paragraph{Policy Learning Scenario} In our policy learning scenario, the functions $e$, $\mu_1$, $\mu_{-1}$, $\sigma_1$, and $\sigma_{-1}$ are defined according to:
\begin{align*}
    e(Z) &= \expit \left( -0.5 -0.75 Z_1 - 0.5 Z_2 - 0.25 Z_1^2 + 0.75 Z_2^2 + Z_1 Z_2 \right) \\
    \mu_{-1}(Z) &= Z_1 - Z_2 + 1.5 Z_2^2 + Z_1 Z_2 \\
    \mu_1(Z) &= 0.5 -3Z_1 + 0.5Z_2 - 2.5 Z_1^2 + 0.5 Z_2^2 + 4 Z_1Z_2 \\
    \sigma_{-1}(Z) &= \softplus(1 + Z_1 + Z_2 + Z_1^2 + Z_2^2 + 2 Z_1 Z_2) \\
    \sigma_1(Z) &= 1 \,.
\end{align*}
In addition, the explicit parameterization of our policy class is given by
\begin{equation*}
    g(z;\theta) = \theta_1 + \theta_2 Z_1 + \theta_3 Z_2 + \theta_4 Z_1^2 + \theta_5 Z_2^2 + 2 \theta_6 Z_1 Z_2 \,.
\end{equation*}
Given this and the definitions of $\mu_{-1}$ and $\mu_1$ above, it is easy to verify that the optimal parameters are given $\theta_0 = [0.5, -4, 1.5, -2.5, -1.0, 1.5]$.

\subsection{Additional Details of Estimation Methods}

\paragraph{Kernel Function for KernelVMM and MMR} Given $n$ training examples, let $s$ denote the median of the entries in the $n \times n$ matrix of euclidean distances between the observed $Z$ values, and define $\sigma_1 = 0.1 s$, $\sigma_2 = s$, and $\sigma_3 = 10s$ Then, we define our kernel according to
\begin{equation*}
    K(z_1, z_2) = \frac{1}{3} \sum_{i=1}^3 \exp \left( \frac{(z_1 - z_2)\tr(z_1 - z_2)}{2 \sigma_i^2} \right)\,.
\end{equation*}

\paragraph{Details of NeuralVMM} First, in our pilot experiments we found that using a batch size of $200$ and learning rates of $5 \times 10^{-4}$ and $2.5 \times 10^{-3}$ for optimizing $\theta$ and $f$ respectively led to consistently stable optimization of the game objective, so we used these hyperparameter values in all cases. Second, regarding the architecture of our $f$ network, we used 50 units in the first hidden layer and 20 units in the second hidden layer, with leaky ReLU activations. These choices were motivated by the choices in past work \citep{bennett2019deep,bennett2020efficient}. Note that in our pilot experiments we did not find results to be very sensitive to the choice of $\mathcal F_n$. Finally, we performed early stopping using the MMR objective on a held-out dev set Specifically, we compute the MMR objective on the dev data every $\lceil 2000/n_b \rceil$ epochs, where $n_b$ is the number of minibatches in a single epoch. After 3 burn-in evaluation cycles, we stop once the computed objective fails to improve 5 consecutive times. We chose this heuristic since, although we found MMR to have sub-optimal performance, it performed fairly consistently and this objective is quick to compute.

\paragraph{Details of SMD} First, suppose our sieve basis is $\{f^{(1)},f^{(2)},\ldots,f^{(k_n)}\}$, and let $F : \mathcal Z \mapsto \mathbb R^{k_n \times m}$ denote the matrix-valued function defined according to $F(z)_{i,j} = f^{(i)}_j(z)$. Then, by plugging in the corresponding least-squares sieve estimate $q_n$, it can easily be seen that the SMD objective is equivalent to 
\begin{equation*}
    J_n^{\text{SMD}}(\theta) = \e_n[F(Z) \rho(X;\theta)]\tr  \Delta_n \e_n[F(Z) \rho(X;\theta)] \,,
\end{equation*}
where
\begin{equation*}
    \Delta_n = \e_n[F(Z) F(Z)\tr]^{-} \e_n[F(Z) \Gamma_n(Z)^{-} F(Z)\tr] \e_n[F(Z) F(Z)\tr]^{-} \,.
\end{equation*}
Thus, given $\Gamma_n$ and $F(Z)$, we have a simple objective to minimize that is convex in the $\rho(X;\theta)$ terms, which in practice in our experiments we optimized using L-BFGS. Finally, for our sieve basis, we used B-splines with 5 knots and degree 2, which was chosen on the basis that it achieved relatively competitive results in some pilot experiments. We note that in past work the authors found that performance is typically very similar for other choices of sieve basis as long as the choice of $k_n$ is similar \citep{chen2012estimation}.

\paragraph{Details of OWGMM}

For this method, we used a 2-stage OWGMM estimator, where in the first stage we used a randomly chosen $\tilde\theta_n$, and in the second stage we used $\tilde\theta_n$ from the first stage. Compared with SMD which also hused a B-spline basis of functions, here
we instead used 10 knots and degree 3, as this was found to result in superior performance for large $n$ for this method. 

\subsection{Additional Details of Inference Methods}

\paragraph{Details of Neural Method} We used the same neural architecture for $\mathcal F_n$ as for our NeuralVMM estimation method in our estimation experiments. We also used an almost-identical optimization procedure based on alternating OAdam updates, except that we did not perform early stopping, and instead optimized for a fixed number of epochs. Based on the results of our pilot experiments, we optimized for $\lceil3000/n_b\rceil$ epochs, where $n_b$ is the number of minibatches in a single epoch. Finally, we used a learning rate of $5 \times 10^{-2}$ for each of $\gamma$ and $f$, and a batch size of 200.

\subsection{Additional Estimation Results}

\begin{table}\centering\footnotesize
    \begin{minipage}[b]{1.0\textwidth}
    \centering
        \begin{tabular}{clcccccc}
         \hline
         \multicolumn{2}{c}{\multirow{2}{*}{Method}} &  \multicolumn{6}{c}{$n$} \\
         & & 200 & 500 & 1,000 & 2,000 & 5,000 & 10,000 \\
         \hline
         \multirow{6}{*}{K-VMM} & $\alpha_n=0$ & $>100$ & $.92\pm1.8$ & $2.1\pm13.1$ & $.16\pm.34$ & $.06\pm.05$ & $.03\pm.07$ \\
         & $\alpha_n=10^{-8}$ & $14.0\pm59.8$ & $>100$ & $.36\pm.69$ & $.15\pm.31$ & $.05\pm.04$ & $.02\pm.01$ \\
         & $\alpha_n=10^{-6}$ & $1.4\pm1.3$ & $.44\pm.37$ & $.19\pm.14$ & $.09\pm.07$ & $.05\pm.04$ & $.02\pm.01$ \\
         & $\alpha_n=10^{-4}$ & $1.4\pm1.4$ & $.40\pm.33$ & $.18\pm.13$ & $.09\pm.07$ & $.05\pm.04$ & $.02\pm.01$ \\
         & $\alpha_n=10^{-2}$ & $1.5\pm1.5$ & $.49\pm.47$ & $.21\pm.17$ & $.09\pm.07$ & $.05\pm.03$ & $.02\pm.01$ \\
         & $\alpha_n=1$ & $1.7\pm1.6$ & $.87\pm.79$ & $.52\pm.64$ & $.35\pm.49$ & $.22\pm.18$ & $.19\pm.19$ \\
         \hdashline
         \multirow{4}{*}{N-VMM} & $\lambda_n=0$ & $5.2\pm2.7$ & $1.5\pm.74$ & $.55\pm.29$ & $.32\pm.20$ & $.16\pm.10$ & $.09\pm.07$ \\
         & $\lambda_n=10^{-4}$ & $5.0\pm3.0$ & $1.5\pm.73$ & $.58\pm.32$ & $.30\pm.18$ & $.15\pm.09$ & $.09\pm.08$ \\
         & $\lambda_n=1$ & $3.7\pm1.8$ & $1.4\pm.58$ & $.54\pm.29$ & $.32\pm.18$ & $.15\pm.10$ & $.09\pm.08$ \\
         \hdashline
         \multirow{3}{*}{SMD} & Identity & $4.4\pm2.9$ & $4.4\pm4.0$ & $3.3\pm3.8$ & $2.7\pm3.1$ & $2.5\pm2.9$ & $3.7\pm4.0$ \\
         & Homo & $4.3\pm3.1$ & $3.4\pm5.8$ & $3.3\pm4.9$ & $3.7\pm3.9$ & $3.6\pm3.3$ & $3.2\pm3.1$ \\
         & Hetero & $4.8\pm3.4$ & $3.5\pm4.0$ & $3.4\pm3.7$ & $2.4\pm2.9$ & $3.2\pm3.1$ & $2.7\pm3.3$ \\
         \hdashline
         MMR & & $2.1\pm.81$ & $1.7\pm.44$ & $1.5\pm.29$ & $1.4\pm.31$ & $1.3\pm.24$ & $1.3\pm.17$ \\
         \hdashline
         OWGMM & & $7.8\pm10.0$ & $5.4\pm7.1$ & $4.0\pm4.3$ & $2.0\pm2.2$ & $.71\pm.91$ & $.39\pm.41$ \\
         \hdashline
         NCB & & $5.9\pm1.3$ & $5.7\pm.67$ & $5.5\pm.63$ & $5.6\pm.53$ & $5.6\pm.28$ & $5.5\pm.22$ \\
         \hline
    \end{tabular}
    \centering{\sffamily\small\begin{enumerate}\item[(a)] \simpleiv \end{enumerate}}
    \end{minipage}\hfill\hspace{0.3cm}
    \linebreak
    \begin{minipage}[b]{1.0\textwidth}
    \centering
        \begin{tabular}{clcccccc}
         \hline
         \multicolumn{2}{c}{\multirow{2}{*}{Method}} &  \multicolumn{6}{c}{$n$} \\
         & & 200 & 500 & 1,000 & 2,000 & 5,000 & 10,000 \\
         \hline
         \multirow{6}{*}{K-VMM} & $\alpha_n=0$ & $>100$ & $.92\pm1.8$ & $2.1\pm13.1$ & $.16\pm.34$ & $.06\pm.05$ & $.03\pm.07$ \\
         & $\alpha_n=10^{-8}$ & $14.0\pm59.8$ & $>100$ & $.36\pm.69$ & $.15\pm.31$ & $.05\pm.04$ & $.02\pm.01$ \\
         & $\alpha_n=10^{-6}$ & $1.4\pm1.3$ & $.44\pm.37$ & $.19\pm.14$ & $.09\pm.07$ & $.05\pm.04$ & $.02\pm.01$ \\
         & $\alpha_n=10^{-4}$ & $1.4\pm1.4$ & $.40\pm.33$ & $.18\pm.13$ & $.09\pm.07$ & $.05\pm.04$ & $.02\pm.01$ \\
         & $\alpha_n=10^{-2}$ & $1.5\pm1.5$ & $.49\pm.47$ & $.21\pm.17$ & $.09\pm.07$ & $.05\pm.03$ & $.02\pm.01$ \\
         & $\alpha_n=1$ & $1.7\pm1.6$ & $.87\pm.79$ & $.52\pm.64$ & $.35\pm.49$ & $.22\pm.18$ & $.19\pm.19$ \\
         \hdashline
         \multirow{4}{*}{N-VMM} & $\lambda_n=0$ & $5.2\pm2.7$ & $1.5\pm.74$ & $.55\pm.29$ & $.32\pm.20$ & $.16\pm.10$ & $.09\pm.07$ \\
         & $\lambda_n=10^{-4}$ & $5.0\pm3.0$ & $1.5\pm.73$ & $.58\pm.32$ & $.30\pm.18$ & $.15\pm.09$ & $.09\pm.08$ \\
         & $\lambda_n=1$ & $3.7\pm1.8$ & $1.4\pm.58$ & $.54\pm.29$ & $.32\pm.18$ & $.15\pm.10$ & $.09\pm.08$ \\
         \hdashline
         \multirow{3}{*}{SMD} & Identity & $4.4\pm2.9$ & $4.4\pm4.0$ & $3.3\pm3.8$ & $2.7\pm3.1$ & $2.5\pm2.9$ & $3.7\pm4.0$ \\
         & Homo & $4.3\pm3.1$ & $3.4\pm5.8$ & $3.3\pm4.9$ & $3.7\pm3.9$ & $3.6\pm3.3$ & $3.2\pm3.1$ \\
         & Hetero & $4.8\pm3.4$ & $3.5\pm4.0$ & $3.4\pm3.7$ & $2.4\pm2.9$ & $3.2\pm3.1$ & $2.7\pm3.3$ \\
         \hdashline
         MMR & & $2.1\pm.81$ & $1.7\pm.44$ & $1.5\pm.29$ & $1.4\pm.31$ & $1.3\pm.24$ & $1.3\pm.17$ \\
         \hdashline
         OWGMM & & $4.7\pm3.3$ & $3.7\pm3.1$ & $4.1\pm5.9$ & $3.3\pm3.7$ & $3.4\pm3.1$ & $3.1\pm3.1$ \\
         \hdashline
         NCB & & $5.9\pm1.3$ & $5.7\pm.67$ & $5.5\pm.63$ & $5.6\pm.53$ & $5.6\pm.28$ & $5.5\pm.22$ \\
         \hline
    \end{tabular}
    \centering{\sffamily\small\begin{enumerate}\item[(b)] \heteroiv \end{enumerate}}
    \end{minipage}\hfill\hspace{0.3cm}
    \linebreak
    \begin{minipage}[b]{1.0\textwidth}
    \centering
        \begin{tabular}{clcccccc}
         \hline
         \multicolumn{2}{c}{\multirow{2}{*}{Method}} &  \multicolumn{6}{c}{$n$} \\
         & & 200 & 500 & 1,000 & 2,000 & 5,000 & 10,000 \\
         \hline
         \multirow{6}{*}{K-VMM} & $\alpha_n=0$ & $.25\pm.51$ & $.13\pm.45$ & $.09\pm.18$ & $.03\pm.05$ & $.28\pm.90$ & $.16\pm.60$ \\
         & $\alpha_n=10^{-8}$ & $.16\pm.12$ & $.10\pm.29$ & $.06\pm.06$ & $.03\pm.04$ & $.03\pm.14$ & $.02\pm.02$ \\
         & $\alpha_n=10^{-6}$ & $.13\pm.09$ & $.06\pm.03$ & $.04\pm.02$ & $.02\pm.01$ & $.01\pm.01$ & $.01\pm.00$ \\
         & $\alpha_n=10^{-4}$ & $.13\pm.09$ & $.05\pm.03$ & $.04\pm.02$ & $.02\pm.01$ & $.02\pm.01$ & $.01\pm.00$ \\
         & $\alpha_n=10^{-2}$ & $.13\pm.10$ & $.06\pm.03$ & $.04\pm.02$ & $.03\pm.01$ & $.02\pm.01$ & $.02\pm.01$ \\
         & $\alpha_n=1$ & $.17\pm.17$ & $.07\pm.04$ & $.05\pm.02$ & $.04\pm.02$ & $.03\pm.01$ & $.03\pm.01$ \\
         \hdashline
         \multirow{4}{*}{N-VMM} & $\lambda_n=0$ & $.22\pm.12$ & $.10\pm.04$ & $.06\pm.03$ & $.03\pm.02$ & $.01\pm.01$ & $.01\pm.00$ \\
         & $\lambda_n=10^{-4}$ & $.20\pm.10$ & $.09\pm.05$ & $.06\pm.03$ & $.03\pm.02$ & $.01\pm.01$ & $.01\pm.01$ \\
         & $\lambda_n=1$ & $.18\pm.11$ & $.08\pm.05$ & $.05\pm.03$ & $.03\pm.02$ & $.01\pm.01$ & $.01\pm.00$ \\
         \hdashline
         \multirow{3}{*}{SMD} & Identity & $.86\pm.97$ & $.67\pm.88$ & $.47\pm.39$ & $.54\pm.81$ & $.52\pm.71$ & $.62\pm.87$ \\
         & Homo & $.88\pm.89$ & $.69\pm.78$ & $.50\pm.66$ & $.62\pm.92$ & $.61\pm.78$ & $.56\pm.67$ \\
         & Hetero & $1.3\pm1.2$ & $.79\pm1.00$ & $.70\pm.70$ & $.65\pm.68$ & $.78\pm1.1$ & $.66\pm.86$ \\
         \hdashline
         MMR & & $.22\pm.29$ & $.10\pm.07$ & $.09\pm.09$ & $.06\pm.06$ & $.07\pm.07$ & $.06\pm.05$ \\
         \hdashline
         OWGMM & & $.26\pm.17$ & $.15\pm.17$ & $.15\pm.18$ & $.13\pm.18$ & $.14\pm.24$ & $.11\pm.19$ \\
         \hdashline
         NCB & & $.39\pm.40$ & $.32\pm.34$ & $.30\pm.27$ & $.28\pm.25$ & $.30\pm.29$ & $.33\pm.23$ \\
         \hline
    \end{tabular}
    \centering{\sffamily\small\begin{enumerate}\item[(c)] \policylearning \end{enumerate}}
    \end{minipage}\hfill\hspace{0.3cm}
    \caption{Results for our estimation experiments in terms of cost (MSE of predicted function for the IV scenarios, or policy sub-optimality for the policy learning scenario.) For each combination of scenario, method, and $n$, the mean cost of $\hat\theta_n$ is estimated over 50 replications, along with standard errors.}
\label{tab:est-results-cost}
\normalsize
\end{table}

\begin{table}\centering\footnotesize
    \begin{minipage}[b]{1.0\textwidth}
    \centering
        \begin{tabular}{clcccccc}
         \hline
         \multicolumn{2}{c}{\multirow{2}{*}{Method}} &  \multicolumn{6}{c}{$n$} \\
         & & 200 & 500 & 1,000 & 2,000 & 5,000 & 10,000 \\
         \hline
         \multirow{6}{*}{K-VMM} & $\alpha_n=0$ & $64.1$ & $.48$ & $12.8$ & $.04$ & $.01$ & $.01$ \\
         & $\alpha_n=10^{-8}$ & $.14$ & $.10$ & $.04$ & $.26$ & $.01$ & $.02$ \\
         & $\alpha_n=10^{-6}$ & $.22$ & $.14$ & $.05$ & $.02$ & $.01$ & $.01$ \\
         & $\alpha_n=10^{-4}$ & $.34$ & $.23$ & $.04$ & $.03$ & $.02$ & $.01$ \\
         & $\alpha_n=10^{-2}$ & $.30$ & $.16$ & $.03$ & $.09$ & $.03$ & $.03$ \\
         & $\alpha_n=1$ & $.58$ & $.12$ & $.03$ & $.06$ & $.02$ & $.03$ \\
         \hdashline
          \multirow{4}{*}{N-VMM} & $\lambda_n=0$ & $1.0$ & $.38$ & $.29$ & $.05$ & $.12$ & $.16$ \\
         & $\lambda_n=10^{-4}$ & $1.0$ & $.35$ & $.21$ & $.07$ & $.10$ & $.19$ \\
         & $\lambda_n=1$ & $.94$ & $.41$ & $.24$ & $.05$ & $.14$ & $.20$ \\
         \hdashline
         \multirow{3}{*}{SMD} & Identity & $.21$ & $.07$ & $.09$ & $.01$ & $.03$ & $.04$ \\
         & Homo & $.21$ & $.07$ & $.09$ & $.01$ & $.03$ & $.03$ \\
         & Hetero & $.24$ & $.06$ & $.08$ & $.03$ & $.03$ & $.03$ \\
         \hdashline
         MMR & & $.64$ & $.26$ & $.12$ & $.06$ & $.03$ & $.03$ \\
         \hdashline
         OWGMM & & $.25$ & $.12$ & $.15$ & $.05$ & $.07$ & $.12$ \\
         \hdashline
         NCB & & $2.4$ & $2.4$ & $2.4$ & $2.4$ & $2.4$ & $2.4$ \\
         \hline
    \end{tabular}
    \centering{\sffamily\small\begin{enumerate}\item[(a)] \simpleiv \end{enumerate}}
    \end{minipage}\hfill\hspace{0.3cm}
    \linebreak
    \begin{minipage}[b]{1.0\textwidth}
    \centering
        \begin{tabular}{clcccccc}
         \hline
         \multicolumn{2}{c}{\multirow{2}{*}{Method}} &  \multicolumn{6}{c}{$n$} \\
         & & 200 & 500 & 1,000 & 2,000 & 5,000 & 10,000 \\
         \hline
         \multirow{6}{*}{K-VMM} & $\alpha_n=0$ & $>100$ & $.22$ & $1.1$ & $.10$ & $.12$ & $.07$ \\
         & $\alpha_n=10^{-8}$ & $1.9$ & $>100$ & $.31$ & $.11$ & $.14$ & $.03$ \\
         & $\alpha_n=10^{-6}$ & $.41$ & $.18$ & $.28$ & $.13$ & $.09$ & $.02$ \\
         & $\alpha_n=10^{-4}$ & $.44$ & $.08$ & $.16$ & $.06$ & $.07$ & $.01$ \\
         & $\alpha_n=10^{-2}$ & $.63$ & $.21$ & $.18$ & $.08$ & $.06$ & $.01$ \\
         & $\alpha_n=1$ & $1.4$ & $1.3$ & $1.1$ & $1.0$ & $1.1$ & $1.0$ \\
         \hdashline
         \multirow{4}{*}{N-VMM} &    $\lambda_n=0$ & $2.6$ & $2.1$ & $1.5$ & $1.1$ & $.87$ & $.60$ \\
         & $\lambda_n=10^{-4}$ & $2.5$ & $2.1$ & $1.4$ & $1.0$ & $.83$ & $.62$ \\
         & $\lambda_n=1$ & $2.5$ & $2.0$ & $1.4$ & $1.2$ & $.87$ & $.61$ \\
         \hdashline
         \multirow{3}{*}{SMD} & Identity & $1.7$ & $6.7$ & $4.5$ & $.88$ & $1.1$ & $2.8$ \\
         & Homo & $3.8$ & $2.4$ & $1.0$ & $1.6$ & $2.0$ & $3.0$ \\
         & Hetero & $3.6$ & $.92$ & $3.8$ & $1.4$ & $1.8$ & $3.4$ \\
         \hdashline
         MMR & & $3.1$ & $3.1$ & $3.1$ & $3.1$ & $3.1$ & $3.1$ \\
         \hdashline
         OWGMM & & $2.6$ & $6.2$ & $1.4$ & $3.3$ & $3.3$ & $1.4$ \\
         \hdashline
         NCB & & $2.6$ & $2.8$ & $2.7$ & $2.8$ & $2.8$ & $2.7$ \\
         \hline
    \end{tabular}
    \centering{\sffamily\small\begin{enumerate}\item[(b)] \heteroiv \end{enumerate}}
    \end{minipage}\hfill\hspace{0.3cm}
    \linebreak
    \begin{minipage}[b]{1.0\textwidth}
    \centering
        \begin{tabular}{clcccccc}
         \hline
         \multicolumn{2}{c}{\multirow{2}{*}{Method}} &  \multicolumn{6}{c}{$n$} \\
         & & 200 & 500 & 1,000 & 2,000 & 5,000 & 10,000 \\
         \hline
         \multirow{6}{*}{K-VMM} & $\alpha_n=0$ & $>100$ & $>100$ & $1.5$ & $1.1$ & $>100$ & $>100$ \\
         & $\alpha_n=10^{-8}$ & $5.5$ & $2.4$ & $4.1$ & $2.2$ & $2.7$ & $.81$ \\
         & $\alpha_n=10^{-6}$ & $1.7$ & $.72$ & $1.0$ & $1.2$ & $1.3$ & $1.4$ \\
         & $\alpha_n=10^{-4}$ & $1.0$ & $1.1$ & $1.4$ & $1.5$ & $1.6$ & $1.7$ \\
         & $\alpha_n=10^{-2}$ & $1.1$ & $1.8$ & $2.0$ & $2.0$ & $2.1$ & $2.2$ \\
         & $\alpha_n=1$ & $2.2$ & $2.8$ & $2.9$ & $2.9$ & $2.9$ & $2.9$ \\
         \hdashline
         \multirow{4}{*}{N-VMM} & $\lambda_n=0$ & $10.9$ & $5.7$ & $2.0$ & $.81$ & $.59$ & $.92$ \\
         & $\lambda_n=10^{-4}$ & $11.0$ & $6.8$ & $2.0$ & $.78$ & $.62$ & $.91$ \\
         & $\lambda_n=1$ & $6.3$ & $1.4$ & $.58$ & $.85$ & $1.3$ & $1.4$ \\
         \hdashline
         \multirow{3}{*}{SMD} & Identity & $8.1$ & $2.5$ & $3.1$ & $3.5$ & $53.2$ & $4.9$ \\
         & Homo & $3.3$ & $3.0$ & $2.8$ & $3.0$ & $3.1$ & $3.1$ \\
         & Hetero & $>100$ & $>100$ & $>100$ & $>100$ & $>100$ & $>100$ \\
         \hdashline
         MMR & & $2.7$ & $3.2$ & $3.3$ & $3.5$ & $3.6$ & $3.6$ \\
         \hdashline
         OWGMM & & $2.2$ & $1.3$ & $1.6$ & $1.1$ & $1.8$ & $1.5$ \\
         \hdashline
         NCB & & $16.1$ & $4.9$ & $3.2$ & $4.2$ & $7.8$ & $6.2$ \\
         \hline
    \end{tabular}
    \centering{\sffamily\small\begin{enumerate}\item[(c)] \policylearning \end{enumerate}}
    \end{minipage}\hfill\hspace{0.3cm}
    \caption{Results for our estimation experiments in terms bias. For each combination of scenario, method, and $n$, we estimated the mean bias of $\hat\theta_n$ over 50 replications. We write $>100$ whenever the bias was greater than $100$.}
\label{tab:est-results-bias}
\normalsize
\end{table}

\begin{table}\centering\footnotesize
    \begin{minipage}[b]{1.0\textwidth}
    \centering
        \begin{tabular}{clcccccc}
         \hline
         \multicolumn{2}{c}{\multirow{2}{*}{Method}} &  \multicolumn{6}{c}{$n$} \\
         & & 200 & 500 & 1,000 & 2,000 & 5,000 & 10,000 \\
         \hline
         \multirow{6}{*}{K-VMM} & $\alpha_n=0$ & $>100$ & $2.9$ & $89.7$ & $.82$ & $.48$ & $.37$ \\
         & $\alpha_n=10^{-8}$ & $2.3$ & $1.7$ & $1.6$ & $1.8$ & $.50$ & $.41$ \\
         & $\alpha_n=10^{-6}$ & $2.3$ & $1.6$ & $1.3$ & $.88$ & $.49$ & $.37$ \\
         & $\alpha_n=10^{-4}$ & $2.3$ & $1.6$ & $1.3$ & $.85$ & $.50$ & $.37$ \\
         & $\alpha_n=10^{-2}$ & $2.4$ & $1.6$ & $1.3$ & $.84$ & $.51$ & $.37$ \\
         & $\alpha_n=1$ &  $3.3$ & $2.0$ & $1.5$ & $.86$ & $.58$ & $.40$ \\
         \hdashline
          \multirow{4}{*}{N-VMM} & $\lambda_n=0$ & $1.2$ & $1.2$ & $.92$ & $.64$ & $.38$ & $.27$ \\
         & $\lambda_n=10^{-4}$ & $1.3$ & $1.3$ & $.88$ & $.62$ & $.41$ & $.27$ \\
         & $\lambda_n=1$ & $1.1$ & $1.4$ & $.94$ & $.62$ & $.41$ & $.27$ \\
         \hdashline
         \multirow{3}{*}{SMD} & Identity & $2.0$ & $1.6$ & $1.4$ & $.83$ & $.48$ & $.39$ \\
         & Homo & $2.0$ & $1.6$ & $1.4$ & $.83$ & $.49$ & $.39$ \\
         & Hetero & $2.1$ & $1.5$ & $1.3$ & $.81$ & $.49$ & $.38$ \\
         \hdashline
         MMR & & $4.2$ & $2.3$ & $1.7$ & $.91$ & $.61$ & $.41$ \\
         \hdashline
         OWGMM & & $1.8$ & $1.5$ & $1.3$ & $.92$ & $.57$ & $.42$ \\
         \hdashline
         NCB & & $.44$ & $.27$ & $.17$ & $.17$ & $.10$ & $.08$ \\
         \hline
    \end{tabular}
    \centering{\sffamily\small\begin{enumerate}\item[(a)] \simpleiv \end{enumerate}}
    \end{minipage}\hfill\hspace{0.3cm}
    \linebreak
    \begin{minipage}[b]{1.0\textwidth}
    \centering
        \begin{tabular}{clcccccc}
         \hline
         \multicolumn{2}{c}{\multirow{2}{*}{Method}} &  \multicolumn{6}{c}{$n$} \\
         & & 200 & 500 & 1,000 & 2,000 & 5,000 & 10,000 \\
         \hline
         \multirow{6}{*}{K-VMM} & $\alpha_n=0$ & $>100$ & $1.9$ & $8.3$ & $.79$ & $.48$ & $.29$ \\
         & $\alpha_n=10^{-8}$ & $5.7$ & $>100$ & $1.1$ & $.79$ & $.43$ & $.24$ \\
         & $\alpha_n=10^{-6}$ & $2.9$ & $1.4$ & $.84$ & $.57$ & $.46$ & $.24$ \\
         & $\alpha_n=10^{-4}$ & $3.1$ & $1.4$ & $.87$ & $.59$ & $.45$ & $.23$ \\
         & $\alpha_n=10^{-2}$ & $2.9$ & $1.6$ & $1.0$ & $.63$ & $.45$ & $.25$ \\
         & $\alpha_n=1$ & $2.8$ & $1.8$ & $1.5$ & $1.2$ & $.64$ & $.59$ \\
         \hdashline
         \multirow{4}{*}{N-VMM} & $\lambda_n=0$ & $1.6$ & $1.0$ & $.76$ & $.84$ & $.63$ & $.57$ \\
         & $\lambda_n=10^{-4}$ & $1.4$ & $.97$ & $.93$ & $.81$ & $.62$ & $.57$ \\
         & $\lambda_n=1$ & $1.0$ & $.90$ & $.83$ & $.69$ & $.56$ & $.54$ \\
         \hdashline
         \multirow{3}{*}{SMD} & Identity & $17.6$ & $22.5$ & $16.7$ & $11.4$ & $12.6$ & $15.2$ \\
         & Homo & $14.5$ & $13.8$ & $14.7$ & $19.0$ & $16.9$ & $16.2$ \\
         & Hetero & $17.7$ & $11.6$ & $16.4$ & $12.4$ & $13.7$ & $14.2$ \\
         \hdashline
         MMR & & $.98$ & $.56$ & $.50$ & $.46$ & $.36$ & $.29$ \\
         \hdashline
         OWGMM & & $18.5$ & $14.9$ & $11.9$ & $15.5$ & $12.2$ & $12.1$ \\
         \hdashline
         NCB & & $1.5$ & $.90$ & $.65$ & $.51$ & $.27$ & $.19$ \\
         \hline
    \end{tabular}
    \centering{\sffamily\small\begin{enumerate}\item[(b)] \heteroiv \end{enumerate}}
    \end{minipage}\hfill\hspace{0.3cm}
    \linebreak
    \begin{minipage}[b]{1.0\textwidth}
    \centering
        \begin{tabular}{clcccccc}
         \hline
         \multicolumn{2}{c}{\multirow{2}{*}{Method}} &  \multicolumn{6}{c}{$n$} \\
         & & 200 & 500 & 1,000 & 2,000 & 5,000 & 10,000 \\
         \hline
         \multirow{6}{*}{K-VMM} & $\alpha_n=0$ & $>100$ & $>100$ & $9.5$ & $.64$ & $>100$ & $>100$ \\
         & $\alpha_n=10^{-8}$ & $22.7$ & $14.9$ & $30.7$ & $17.0$ & $21.9$ & $4.2$ \\
         & $\alpha_n=10^{-6}$ & $2.8$ & $1.1$ & $.72$ & $.47$ & $.28$ & $.22$ \\
         & $\alpha_n=10^{-4}$ & $2.4$ & $.98$ & $.68$ & $.44$ & $.27$ & $.20$ \\
         & $\alpha_n=10^{-2}$ & $1.7$ & $.82$ & $.54$ & $.36$ & $.23$ & $.20$ \\
         & $\alpha_n=1$ & $1.5$ & $.71$ & $.47$ & $.39$ & $.24$ & $.20$ \\
         \hdashline
         \multirow{4}{*}{N-VMM} & $\lambda_n=0$ & $10.8$ & $4.6$ & $1.6$ & $.69$ & $.39$ & $.27$ \\
         & $\lambda_n=10^{-4}$ & $10.9$ & $6.7$ & $1.9$ & $.71$ & $.40$ & $.29$ \\
         & $\lambda_n=1$ & $7.7$ & $1.9$ & $.83$ & $.46$ & $.30$ & $.20$ \\
         \hdashline
         \multirow{3}{*}{SMD} & Identity & $40.0$ & $20.8$ & $5.8$ & $3.1$ & $>100$ & $24.1$ \\
         & Homo & $3.7$ & $5.8$ & $5.1$ & $4.2$ & $4.3$ & $4.8$ \\
         & Hetero & $>100$ & $>100$ & $>100$ & $>100$ & $>100$ & $>100$ \\
         \hdashline
         MMR & & $1.3$ & $.80$ & $.62$ & $.55$ & $.42$ & $.34$ \\
         \hdashline
         OWGMM & & $9.9$ & $4.6$ & $5.5$ & $6.4$ & $3.1$ & $9.2$ \\
         \hdashline
         NCB & & $34.5$ & $19.7$ & $9.0$ & $8.7$ & $28.0$ & $20.7$ \\
         \hline
    \end{tabular}
    \centering{\sffamily\small\begin{enumerate}\item[(c)] \policylearning \end{enumerate}}
    \end{minipage}\hfill\hspace{0.3cm}
    \caption{Results for our estimation experiments in terms standard deviation. For each combination of scenario, method, and $n$, we estimated the mean standard deviation of $\hat\theta_n$ over 50 replications. We write $>100$ whenever the standard deviation was greater than $100$.}
\label{tab:est-results-std}
\normalsize
\end{table}

Here we provide additional results for our estimation experiments, in terms of the mean cost of our different estimation methods. As discussed in \cref{sec:experiments}, cost is calculated in terms of the $L_2$ error in the predicted regression function for our \simpleiv and \heteroiv scenarios, and in terms of regret of our learnt policy for our \policylearning scenario. We present additional tables of results here in the same format of those in \cref{tab:est-results}, but in terms of cost rather than squared error. These are provided in \cref{tab:est-results-cost}.
In addition, we present additional tables of results where we break down the squared error in terms of bias and standard deviation, which we provide in \cref{tab:est-results-bias} and \cref{tab:est-results-std}.

\subsection{Additional Inference Results}

\begin{table}\centering\small
    \begin{minipage}[b]{1.0\textwidth}
    \centering
        \begin{tabular}{cclccccc}
         \hline
         $n$ & \multicolumn{2}{c}{Method} & Cov & CovBC & PredSD(.05) & PredSD(.5) & PredSD(.95) \\
         \hline
         \multirow{10}{*}{$200$} & \multirow{6}{*}{Kernel} & $\alpha_n=0$ & 79.0 & 94.5 & .20 & .31 & .56 \\
         & & $\alpha_n=10^{-8}$ & 79.0 & 94.0 & .21 & .31 & .54 \\
         & & $\alpha_n=10^{-6}$ & 79.5 & 95.0 & .21 & .32 & .54 \\
         & & $\alpha_n=10^{-4}$ & 81.0 & 96.5 & .21 & .33 & .57 \\
         & & $\alpha_n=10^{-2}$ & 83.0 & 97.0 & .22 & .35 & .68 \\
         & & $\alpha_n=1$ & 87.0 & 98.5 & .24 & .38 & .78 \\
         \cdashline{2-8}
         & \multirow{4}{*}{Neural} & $\lambda_n=0$ & 78.0 & 94.5 & .21 & .31 & .50 \\
         & & $\lambda_n=10^{-4}$ & 77.5 & 93.5 & .21 & .31 & .51 \\
         & & $\lambda_n=1$ & 78.0 & 94.0 & .21 & .31 & .53 \\
         \hdashline
         \multirow{10}{*}{$2000$} & \multirow{6}{*}{Kernel} & $\alpha_n=0$ & 88.0 & 88.0 & .19 & .21 & .26 \\
         & & $\alpha_n=10^{-8}$ & 88.0 & 88.5 & .19 & .21 & .26 \\
         & & $\alpha_n=10^{-6}$ & 88.5 & 88.5 & .20 & .22 & .25 \\
         & & $\alpha_n=10^{-4}$ & 89.0 & 89.0 & .20 & .22 & .26 \\
         & & $\alpha_n=10^{-2}$ & 91.5 & 91.5 & .21 & .24 & .29 \\
         & & $\alpha_n=1$ & 99.5 & 99.5 & .47 & .58 & 1.0 \\
         \cdashline{2-8}
         & \multirow{4}{*}{Neural} & $\lambda_n=0$ & 87.0 & 87.5 & .19 & .21 & .23 \\
         & & $\lambda_n=10^{-4}$ & 87.0 & 87.0 & .19 & .21 & .23 \\
         & & $\lambda_n=1$ & 87.0 & 87.5 & .19 & .21 & .23 \\
         \hline
    \end{tabular}
    \centering{\sffamily\normalsize\begin{enumerate}\item[(a)] \simpleiv; the true standard deviation over the 200 replications was $1.6$ when $n=200$, and $0.21$ when $n=2000$. \end{enumerate}}
    \end{minipage}\hfill\hspace{0.3cm}
    \linebreak
    \begin{minipage}[b]{1.0\textwidth}
    \centering
        \begin{tabular}{cclccccc}
         \hline
         $n$ & \multicolumn{2}{c}{Method} & Cov & CovBC & PredSD(.05) & PredSD(.5) & PredSD(.95) \\
         \hline
         \multirow{10}{*}{$200$} & \multirow{6}{*}{Kernel} & $\alpha_n=0$ & 76.5 & 76.5 & .34 & .58 & 5.3 \\
         & & $\alpha_n=10^{-8}$ & 77.0 & 75.5 & .36 & .58 & 6.4 \\
         & & $\alpha_n=10^{-6}$ & 81.0 & 78.0 & .43 & .63 & 5.5 \\
         & & $\alpha_n=10^{-4}$ & 85.5 & 84.0 & .49 & .74 & 8.1 \\
         & & $\alpha_n=10^{-2}$ & 92.0 & 89.0 & .60 & .98 & 12.8 \\
         & & $\alpha_n=1$ & 99.0 & 98.5 & 1.4 & 3.3 & 42.1 \\
         \cdashline{2-8}
         & \multirow{4}{*}{Neural} & $\lambda_n=0$ & 64.5 & 58.5 & .25 & .44 & .87 \\
         & & $\lambda_n=10^{-4}$ & 65.5 & 58.5 & .27 & .46 & .86 \\
         & & $\lambda_n=1$ & 65.5 & 59.0 & .25 & .45 & .87 \\
         \hdashline
         \multirow{10}{*}{$2000$} & \multirow{6}{*}{Kernel} & $\alpha_n=0$ & 93.5 & 94.5 & .19 & .21 & .24 \\
         & & $\alpha_n=10^{-8}$ & 93.5 & 95.0 & .19 & .21 & .24 \\
         & & $\alpha_n=10^{-6}$ & 94.0 & 95.0 & .20 & .22 & .25 \\
         & & $\alpha_n=10^{-4}$ & 94.0 & 95.5 & .20 & .22 & .25 \\
         & & $\alpha_n=10^{-2}$ & 97.0 & 96.5 & .22 & .23 & .28 \\
         & & $\alpha_n=1$ & 100.0 & 100.0 & .47 & .61 & .98 \\
         \cdashline{2-8}
         & \multirow{4}{*}{Neural} & $\lambda_n=0$ & 91.5 & 93.5 & .19 & .21 & .22 \\
         & & $\lambda_n=10^{-4}$ & 91.5 & 93.5 & .19 & .21 & .23 \\
         & & $\lambda_n=1$ & 91.5 & 94.0 & .20 & .21 & .23 \\
         \hline
    \end{tabular}
    \centering{\sffamily\normalsize\begin{enumerate}\item[(b)] \heteroiv; the true standard deviation over the 200 replications was $3.6$ when $n=200$, and $0.23$ when $n=2000$. \end{enumerate}}
    \end{minipage}\hfill\hspace{0.3cm}
    \caption{Results of our inference experiments, using kernel VMM estimation with $\alpha_n = 10^{-8}$, and various VMM inference methods. We list results using various. The columns have the same interpretation as in \cref{tab:inf-results}.}
    \label{tab:inf-results-ksmall}
\end{table}

\begin{table}\centering\small
    \begin{minipage}[b]{1.0\textwidth}
    \centering
        \begin{tabular}{cclccccc}
         \hline
         $n$ & \multicolumn{2}{c}{Method} & Cov & CovBC & PredSD(.05) & PredSD(.5) & PredSD(.95) \\
         \hline
         \multirow{10}{*}{$200$} & \multirow{6}{*}{Kernel} & $\alpha_n=0$ & 90.0 & 91.0 & .21 & .34 & .59 \\
         & & $\alpha_n=10^{-8}$ & 91.0 & 91.5 & .22 & .34 & .60 \\
         & & $\alpha_n=10^{-6}$ & 91.0 & 92.0 & .22 & .34 & .59 \\
         & & $\alpha_n=10^{-4}$ & 92.0 & 92.5 & .23 & .35 & .60 \\
         & & $\alpha_n=10^{-2}$ & 92.5 & 93.0 & .23 & .37 & .68 \\
         & & $\alpha_n=1$ & 93.0 & 94.0 & .25 & .42 & .82 \\
         \cdashline{2-8}
         & \multirow{4}{*}{Neural} & $\lambda_n=0$ & 90.0 & 90.5 & .22 & .33 & .53 \\
         & & $\lambda_n=10^{-4}$ & 91.5 & 91.5 & .22 & .33 & .56 \\
         & & $\lambda_n=1$ & 90.5 & 91.5 & .21 & .33 & .55 \\
         \hdashline
         \multirow{10}{*}{$2000$} & \multirow{6}{*}{Kernel} & $\alpha_n=0$ & 87.0 & 84.5 & .20 & .23 & .34 \\
         & & $\alpha_n=10^{-8}$ & 86.5 & 84.5 & .20 & .23 & .34 \\
         & & $\alpha_n=10^{-6}$ & 87.5 & 85.5 & .20 & .23 & .34 \\
         & & $\alpha_n=10^{-4}$ & 89.0 & 87.0 & .21 & .24 & .35 \\
         & & $\alpha_n=10^{-2}$ & 91.0 & 89.0 & .22 & .26 & .42 \\
         & & $\alpha_n=1$ & 100.0 & 100.0 & .48 & .66 & 1.5 \\
         \cdashline{2-8}
         & \multirow{4}{*}{Neural} & $\lambda_n=0$ & 79.0 & 78.0 & .19 & .21 & .24 \\
         & & $\lambda_n=10^{-4}$ & 79.0 & 78.0 & .19 & .21 & .23 \\
         & & $\lambda_n=1$ & 79.5 & 78.5 & .19 & .21 & .23 \\
         \hline
    \end{tabular}
    \centering{\sffamily\normalsize\begin{enumerate}\item[(a)] \simpleiv; the true standard deviation over the 200 replications was $0.39$ when $n=200$, and $0.40$ when $n=2000$. \end{enumerate}}
    \end{minipage}\hfill\hspace{0.3cm}
    \linebreak
    \begin{minipage}[b]{1.0\textwidth}
    \centering
        \begin{tabular}{cclccccc}
         \hline
         $n$ & \multicolumn{2}{c}{Method} & Cov & CovBC & PredSD(.05) & PredSD(.5) & PredSD(.95) \\
         \hline
         \multirow{10}{*}{$200$} & \multirow{6}{*}{Kernel} & $\alpha_n=0$ & 80.5 & 79.0 & .41 & .73 & 4.0 \\
         & & $\alpha_n=10^{-8}$ & 79.5 & 79.5 & .41 & .71 & 3.8 \\
         & & $\alpha_n=10^{-6}$ & 84.5 & 82.5 & .45 & .75 & 4.1 \\
         & & $\alpha_n=10^{-4}$ & 88.5 & 86.0 & .51 & .88 & 4.5 \\
         & & $\alpha_n=10^{-2}$ & 94.0 & 91.5 & .64 & 1.1 & 7.4 \\
         & & $\alpha_n=1$ & 100.0 & 100.0 & 1.5 & 3.6 & 24.1 \\
         \cdashline{2-8}
         & \multirow{4}{*}{Neural} & $\lambda_n=0$ & 58.0 & 62.5 & .30 & .52 & .86 \\
         & & $\lambda_n=10^{-4}$ & 58.5 & 62.5 & .27 & .53 & .86 \\
         & & $\lambda_n=1$ & 58.5 & 60.5 & .28 & .53 & .86 \\
         \hdashline
         \multirow{10}{*}{$2000$} & \multirow{6}{*}{Kernel} & $\alpha_n=0$ & 87.5 & 89.0 & .20 & .23 & .35 \\
         & & $\alpha_n=10^{-8}$ & 87.5 & 90.0 & .20 & .23 & .36 \\
         & & $\alpha_n=10^{-6}$ & 88.5 & 90.5 & .20 & .23 & .36 \\
         & & $\alpha_n=10^{-4}$ & 91.0 & 90.5 & .21 & .24 & .37 \\
         & & $\alpha_n=10^{-2}$ & 93.0 & 93.5 & .22 & .26 & .46 \\
         & & $\alpha_n=1$ & 100.0 & 100.0 & .48 & .66 & 1.8 \\
         \cdashline{2-8}
         & \multirow{4}{*}{Neural} & $\lambda_n=0$ & 80.0 & 82.0 & .19 & .21 & .23 \\
         & & $\lambda_n=10^{-4}$ & 80.5 & 81.5 & .19 & .21 & .23 \\
         & & $\lambda_n=1$ & 81.0 & 82.0 & .19 & .21 & .23 \\
         \hline
    \end{tabular}
    \centering{\sffamily\normalsize\begin{enumerate}\item[(b)] \heteroiv; the true standard deviation over the 200 replications was $1.7$ when $n=200$, and $0.34$ when $n=2000$. \end{enumerate}}
    \end{minipage}\hfill\hspace{0.3cm}
    \caption{Results of our inference experiments, using kernel VMM estimation with $\alpha_n = 1$, and various VMM inference methods. We list results using various. The columns have the same interpretation as in \cref{tab:inf-results}.}
    \label{tab:inf-results-klarge}
\end{table}

\begin{table}\centering\small
    \begin{minipage}[b]{1.0\textwidth}
    \centering
        \begin{tabular}{cclccccc}
         \hline
         $n$ & \multicolumn{2}{c}{Method} & Cov & CovBC & PredSD(.05) & PredSD(.5) & PredSD(.95) \\
         \hline
         \multirow{10}{*}{$200$} & \multirow{6}{*}{Kernel} & $\alpha_n=0$ & 18.5 & 98.0 & .17 & .25 & .39 \\
         & & $\alpha_n=10^{-8}$ & 19.0 & 96.5 & .17 & .26 & .40 \\
         & & $\alpha_n=10^{-6}$ & 19.5 & 97.0 & .18 & .26 & .40 \\
         & & $\alpha_n=10^{-4}$ & 22.0 & 98.5 & .18 & .27 & .42 \\
         & & $\alpha_n=10^{-2}$ & 24.0 & 98.5 & .19 & .28 & .46 \\
         & & $\alpha_n=1$ & 33.0 & 100.0 & .21 & .31 & .55 \\
         \cdashline{2-8}
         & \multirow{4}{*}{Neural} & $\lambda_n=0$ & 19.0 & 96.5 & .18 & .25 & .37 \\
         & & $\lambda_n=10^{-4}$ & 17.5 & 96.5 & .18 & .25 & .37 \\
         & & $\lambda_n=1$ & 18.0 & 97.0 & .17 & .25 & .37 \\
         \hdashline
         \multirow{10}{*}{$2000$} & \multirow{6}{*}{Kernel} & $\alpha_n=0$ & 85.5 & 89.0 & .20 & .24 & .33 \\
         & & $\alpha_n=10^{-8}$ & 86.0 & 89.5 & .20 & .24 & .33 \\
         & & $\alpha_n=10^{-6}$ & 86.5 & 89.5 & .20 & .24 & .34 \\
         & & $\alpha_n=10^{-4}$ & 87.0 & 90.0 & .21 & .25 & .35 \\
         & & $\alpha_n=10^{-2}$ & 91.5 & 91.0 & .22 & .27 & .42 \\
         & & $\alpha_n=1$ & 100.0 & 100.0 & .48 & .76 & 1.5 \\
         \cdashline{2-8}
         & \multirow{4}{*}{Neural} & $\lambda_n=0$ & 78.0 & 82.5 & .19 & .20 & .22 \\
         & & $\lambda_n=10^{-4}$ & 77.5 & 82.5 & .19 & .21 & .22 \\
         & & $\lambda_n=1$ & 78.0 & 83.0 & .19 & .20 & .22 \\
         \hline
    \end{tabular}
    \centering{\sffamily\normalsize\begin{enumerate}\item[(a)] \simpleiv; the true standard deviation over the 200 replications was $0.23$ when $n=200$, and $0.29$ when $n=2000$. \end{enumerate}}
    \end{minipage}\hfill\hspace{0.3cm}
    \linebreak
    \begin{minipage}[b]{1.0\textwidth}
    \centering
        \begin{tabular}{cclccccc}
         \hline
         $n$ & \multicolumn{2}{c}{Method} & Cov & CovBC & PredSD(.05) & PredSD(.5) & PredSD(.95) \\
         \hline
         \multirow{10}{*}{$200$} & \multirow{6}{*}{Kernel} & $\alpha_n=0$ & 77.5 & 80.5 & .39 & .65 & 2.3 \\
         & & $\alpha_n=10^{-8}$ & 76.0 & 79.5 & .42 & .67 & 2.0 \\
         & & $\alpha_n=10^{-6}$ & 79.5 & 84.0 & .46 & .73 & 2.7 \\
         & & $\alpha_n=10^{-4}$ & 84.5 & 88.5 & .53 & .91 & 3.5 \\
         & & $\alpha_n=10^{-2}$ & 90.0 & 92.5 & .61 & 1.2 & 5.8 \\
         & & $\alpha_n=1$ & 100.0 & 100.0 & 1.4 & 3.5 & 21.9 \\
         \cdashline{2-8}
         & \multirow{4}{*}{Neural} & $\lambda_n=0$ & 49.0 & 51.5 & .20 & .38 & .72 \\
         & & $\lambda_n=10^{-4}$ & 53.0 & 53.0 & .23 & .40 & .70 \\
         & & $\lambda_n=1$ & 52.0 & 53.0 & .20 & .38 & .67 \\
         \hdashline
         \multirow{10}{*}{$2000$} & \multirow{6}{*}{Kernel} & $\alpha_n=0$ & 88.0 & 94.0 & .20 & .24 & .31 \\
         & & $\alpha_n=10^{-8}$ & 89.5 & 95.0 & .20 & .24 & .31 \\
         & & $\alpha_n=10^{-6}$ & 89.5 & 95.0 & .21 & .24 & .32 \\
         & & $\alpha_n=10^{-4}$ & 90.5 & 95.0 & .21 & .24 & .33 \\
         & & $\alpha_n=10^{-2}$ & 93.0 & 96.5 & .22 & .27 & .39 \\
         & & $\alpha_n=1$ & 100.0 & 100.0 & .48 & .74 & 1.3 \\
         \cdashline{2-8}
         & \multirow{4}{*}{Neural} & $\lambda_n=0$ & 82.5 & 93.0 & .19 & .20 & .22 \\
         & & $\lambda_n=10^{-4}$ & 82.5 & 93.0 & .19 & .20 & .22 \\
         & & $\lambda_n=1$ & 82.5 & 92.5 & .19 & .20 & .22 \\
         \hline
    \end{tabular}
    \centering{\sffamily\normalsize\begin{enumerate}\item[(b)] \heteroiv; the true standard deviation over the 200 replications was $1.1$ when $n=200$, and $0.24$ when $n=2000$. \end{enumerate}}
    \end{minipage}\hfill\hspace{0.3cm}
    \caption{Results of our inference experiments, using kernel VMM estimation with $\lambda_n = 0$, and various VMM inference methods. We list results using various. The columns have the same interpretation as in \cref{tab:inf-results}.}
    \label{tab:inf-results-n}
\end{table}

In this subsection we provide additional inference results, using different algorithms for the estimation procedure. In particular, in \cref{tab:inf-results-ksmall} we present results using kernel VMM with $\alpha_n=10^{-8}$, in \cref{tab:inf-results-klarge} we present results using kernel VMM with $\alpha_n=1$, and \cref{tab:inf-results-n} we present results using neural VMM with $\lambda_n=0$.

\end{document}